\documentclass[twoside,11pt]{article}
\usepackage{jmlr2e}
\usepackage{amsmath,amssymb,mathrsfs}
\usepackage{algorithmic}
\usepackage{algorithm}
\usepackage{hyperref}
\usepackage{makecell}
\usepackage{tablefootnote}
\usepackage{comment}
\usepackage{subcaption,cleveref}
\usepackage{xcolor}
\def\mfd{{\MM}_{\boldsymbol{r}}^{\textsf{tt}}}

\def\r{\boldsymbol{r}}
\def\ambspace{\RR^{d_1\times \cdots \times d_m}}

\def\wt{\widetilde}

\def\hat{\widehat}
\def\reshape{\textsf{reshape}}
\def\ttsvd{\textsf{SVD}^{\textsf{tt}}}
\def\incoh{\textsf{Incoh}}
\def\spiki{\textsf{Spiki}}
\def\trim{\textsf{Trim}}
\def\no{\notag}
\def\tilde{\widetilde}

\def\sigmamin{\underline{\sigma}}
\def\sigmamax{\overline{\sigma}}
\def\rmax{\overline{r}}
\def\rmin{\underline{r}}
\def\dmax{\overline{d}}
\def\ranktt{\textsf{rank}_{\textsf{tt}}}
\def\rankcp{\textsf{rank}_{\textsf{cp}}}
\def\ranktk{\textsf{rank}_{\textsf{tucker}}}
\def\rank{\textsf{rank}}
\def\dof{\textsf{dof}}

\newcommand\fro[1]{\| #1 \|_{\rm{F}}}
\newcommand\nuc[1]{\| #1 \|_{*}}
\newcommand\op[1]{\|#1\|}
\newcommand\twoinf[1]{\| #1 \|_{2,\infty}}

\newcommand\inp[2]{\langle #1, #2 \rangle}
\newcommand{\mat}[1]{\begin{bmatrix}#1 \\ \end{bmatrix}}
\newcommand{\lrangle}[1]{\langle #1 \rangle}

\newcommand{\RN}[1]{%
	\textup{\uppercase\expandafter{\romannumeral#1}}%
}
\newcommand\lr[1]{^{\lrangle{#1}}}

\def\calA{{\mathcal A}}
\def\calB{{\mathcal B}}

\def\calD{{\mathcal D}}
\def\calE{{\mathcal E}}

\def\calG{{\mathcal G}}

\def\calI{{\mathcal I}}
\def\calJ{{\mathcal J}}

\def\calP{{\mathcal P}}
\def\calQ{{\mathcal Q}}

\def\calS{{\mathcal S}}
\def\calT{{\mathcal T}}
\def\calU{{\mathcal U}}

\def\calW{{\mathcal W}}
\def\calX{{\mathcal X}}
\def\calY{{\mathcal Y}}
\def\calZ{{\mathcal Z}}

\def\bcalE{{\boldsymbol{\mathcal E}}}

\def\EE{{\mathbb E}}

\def\KK{{\mathbb K}}

\def\MM{{\mathbb M}}
\def\NN{{\mathbb N}}
\def\OO{{\mathbb O}}
\def\PP{{\mathbb P}}

\def\RR{{\mathbb R}}
\def\SS{{\mathbb S}}
\def\TT{{\mathbb T}}

\def\a{{\boldsymbol a}}
\def\b{{\boldsymbol b}}

\def\r{{\boldsymbol r}}

\def\x{{\boldsymbol x}}
\def\y{{\boldsymbol y}}

\jmlrheading{}{}{}{}{}{}{}

\ShortHeadings{Provable TT-format Tensor Completion by Riemannian Optimization}{Cai, Li and Xia}
\firstpageno{1}

\begin{document}

\title{Provable Tensor-Train Format Tensor Completion by Riemannian Optimization}

\author{\name Jian-Feng Cai \email jfcai@ust.hk \\
\name Jingyang Li \email jlieb@connect.ust.hk \\
\name Dong Xia \email madxia@ust.hk\\
\addr Department of Mathematics\\
       Hong Kong University of Science and Technology\\
       Hong Kong SAR, China}
\editor{}

\maketitle

\begin{abstract}%   <- trailing '%' for backward compatibility of .sty file
The tensor train (TT) format enjoys appealing advantages in handling structural high-order tensors. The recent decade has witnessed the wide applications of TT-format tensors from diverse disciplines, among which tensor completion has drawn considerable attention. Numerous fast algorithms, including the Riemannian gradient descent (RGrad),  have been proposed for the TT-format tensor completion. However, the theoretical guarantees of these algorithms are largely missing or sub-optimal, partly due to the complicated and recursive algebraic operations in TT-format decomposition. Moreover, existing results established for the tensors of other formats, for example, Tucker and CP, are inapplicable because the algorithms treating TT-format tensors are substantially different and more involved. In this paper, we provide, to our best knowledge, the first theoretical guarantees of the convergence of RGrad algorithm for TT-format tensor completion, under a nearly optimal sample size condition. The RGrad algorithm converges linearly with a constant contraction rate that is free of tensor condition number without the necessity of re-conditioning. We also propose a novel approach, referred to as the  sequential second-order moment method, to attain a warm initialization under a similar sample size requirement. As a byproduct, our result even significantly refines the prior investigation of RGrad algorithm for matrix completion.  Lastly,  statistically (near) optimal rate is derived for RGrad algorithm if the observed entries consist of random sub-Gaussian noise.  Numerical experiments confirm our theoretical discovery and showcase the computational speedup gained by the TT-format decomposition. 
\end{abstract}

\section{Introduction}

An $m$-th order tensor is an $m$-dimensional array, that is, a matrix is a $2$nd-order tensor. Tensor completion refers to the task of recovering the whole tensor by observing only a {\it small} subset of its entries. Of course, this is possible only when the underlying tensor possesses certain structural properties such that the tensor of interest actually lies in a low-dimensional space. Throughout this paper, we assume that the tensor of interest is {\it low-rank}. 
There are diverse applications that drive the research of tensor completion: visual data in-painting \citep{liu2012tensor, li2017low}, medical imaging \citep{gandy2011tensor, semerci2014tensor, cheng2017comprehensive, wang2020learning}, seismic data analysis \citep{kreimer2013tensor, ely20135d}, personalized medicine \citep{soroushmehr2016transforming, pawlowski2019machine}, to name a few. It is a natural generalization of the well-explored matrix completion \citep{candes2009exact, candes2010power, cai2010singular, recht2011simpler,davenport2016overview, xia2021statistical,chen2019inference,cai2016structured,xia2016estimation,sun2016guaranteed,jain2013low,keshavan2010matrix}. While seeming, at least conceptually, a straightforward extension of matrix completion, the multi-linear nature of tensors poses unprecedented challenges to tensor completion from multiple fronts. For instance, convex relaxation by matrix nuclear norm is a prevailing approach for low-rank matrix completion, whereas tensor nuclear norm, also a convex function, is generally NP-hard \citep{hillar2013most}  to compute. This computation hardness exists in many tensor-related convex functions such as tensor operator norm. As a result, the trick of convex relaxation for tensor completion can be computationally infeasible in some cases. Moreover, another phenomenon making tensor completion fundamentally different from matrix completion is the gap between the information-theoretical sample complexity (the number of observed entries for example) and that required by polynomial-time algorithms. Indeed, it is known that matrix completion is solvable by fast algorithms with a nearly optimal sample complexity \citep{gross2011recovering, candes2009exact}. However, evidence \citep{barak2016noisy} has been found showing that the sample size required by polynomial-time algorithms for tensor completion is significantly larger than the information-theoretical sample complexity. This phenomenon is also observed in other tensor-related problems such as tensor PCA \citep{zhang2018tensor,brennan2018reducibility} and tensor clustering \citep{luo2020tensor}.

Since a multi-linear array can always be re-arranged into a matrix,  tensor completion can also be re-formulated as matrix completion.  Along this direction are some representable works \citep{liu2012tensor, song2019tensor, yuan2017incoherent, gandy2011tensor,mu2014square}. These methods,  while easy to implement (borrowing state-of-art conclusions from matrix completion),  suffer from unnecessarily high computational cost because the intrinsic tensor structure is abandoned and the resultant matrix lies in a much higher dimensional space.  
To overcome these issues,  methods for tensor completion better act on the low-rank tensor structure {\it directly}.  
Unlike matrices whose ranks are universally defined,  there exist multiple widely-accepted definitions of the rank(s) for tensors and,  hence,  different formats of tensor decomposition.  Existing literature of tensor completion, especially those with theoretical investigations, mainly focus on the CP format and Tucker format.  The CP decomposition of a tensor seeks a representation by the sum of a minimal number, called the CP rank,  of rank-one tensors.  Under the assumption that the underlying tensor is CP decomposable with $r$ {\it orthogonal} components,  \citet{jain2014provable} proposed a fast alternating minimization algorithm for {\it exactly} completing a $d\times d\times d$ tensor using merely $O(\kappa_0^4r^5d^{3/2}\cdot\textrm{\small Polylog(d)})$ randomly observed entries,  where $\kappa_0$ denotes the tensor condition number (see formal definition in Section \ref{sec:pre}).  Later,  \citet{barak2016noisy} introduced a semi-definite programming, referred to as the sum-of-squares (SOS) hierarchy,  and demonstrated that this polynomial-time method can {\it approximately} recover the tensor using a similar number of observed entries,  even if the CP components are not orthogonal.  Then,  in \citet{potechin2017exact},  the authors further proved that, with orthogonal CP components,  SOS method can {\it exactly} recover the tensor by observing $O(rd^{3/2}\cdot\textrm{\small Polylog(d)})$ randomly sampled entries.  
\cite{kivva2020exact} focused on the regime $r>d$ and proved that exact completion from $O(rd^{3/2}\cdot\textrm{\small Polylog(d)})$ samples for random tensors with CP rank less than $r$ is possible with high probability using SOS.
Even though SOS is a {\it provably} polynomial-time method,  it usually runs very slowly making it impractical for real-world applications.  Meanwhile,  the orthogonal decomposability is a rather restrictive assumption.  Still restricted to the CP format,  a spectral algorithm was proposed by \citet{montanari2018spectral} which {\it approximately} recovers the tensor under a sample size requirement comparable to \citet{potechin2017exact}.  The spectral algorithm runs very fast,  so it is more scalable to large tensors.  More recently,  based on many sample splittings,  \citet{liu2020tensor} proposed an algorithm,  combining both non-convex and convex ideas,  to {\it exactly} recover a tensor with robust linearly independent components. 
Other representable works include \cite{bi2018multilayer,ibriga2021covariate,sun2017provable,cai2020uncertainty} and etc.

The Tucker rank of a tensor refers to the rank of the matrices obtained by  tensor {\it unfolding}.  An $m$-th order tensor of size $d\times\cdots\times d$ admits $m$ ways of unfolding,  so the Tucker rank
consists of a collection of $m$ matrix ranks.  Tucker decomposition factorizes the tensor into the multi-linear product of a {\it core tensor} and $m$ orthonormal matrices,  usually referred to as the singular vectors.  The core tensor is usually small-sized but still of order $m$.  Tucker decomposition is {\it always} attainable via higher-order singular value decomposition (HOSVD).  Readers are suggested to \citet{kolda2009tensor} for more details.  Presuming low Tucker ranks,  \citet{huang2015provable} designed a polynomial-time algorithm,  by tensor unfolding and (matrix) nuclear norm minimization,  for tensor completion.  Due to the unbalanced unfolding of odd-order tensors,  their algorithm requires a random sample of $O(rd^{\lceil \frac{m}{2}\rceil}\cdot\textrm{\small Polylog(d)})$ entries for completing an $m$-th order tensor. \citet{zhang2019cross} proposed a special sampling scheme showing that $O(rd+r^m)$ entries suffice to exactly recover the unknown tensor. Convex approach by tensor nuclear norm minimization, albeit being computationally infeasible, was studied in \citet{yuan2016tensor} which completes the tensor by using $O(r^{1/2}d^{m/2})$ randomly sampled entries. Later,  \citet{xia2017polynomial} investigated a polynomial time algorithm for exact tensor completion with a sample complexity $O(\kappa_0^2r^{m}d^{m/2}\cdot\textrm{\small Polylog(d)}))$ via gradient descent on the Grassmannian manifold, but the iteration complexity was not provided. More recently,  \citet{tong2021scaling} introduced a scaled gradient descent algorithm and derived an iteration complexity that is free of the condition number. In \citet{xia2017statistically}, a fast higher-order orthogonal iteration algorithm was proposed for noisy tensor completion achieving a statistically optimal rate with a sample complexity $O(r^{m/2}d^{m/2}\cdot\textrm{\small Polylog(d)})$.  The rate was derived under sub-Gaussian noise and is optimal w.r.t.  the model complexity and sample size,  which is,  however,  not proportional to the noise variance.  

Despite the popularity of CP format and Tucker format in applications and theories, both these two formats have their pros and cons. The CP decomposition is more friendly for interpreting the principal components of tensors, and the required degree of freedom $O(mrd)$ grows linearly with respect to the order $m$ of a tensor of size $d\times\cdots\times d$ and CP rank $r$. Unfortunately, the set of tensors of a fixed CP rank is not even closed \citep{kolda2009tensor}, implying that the existence of a best rank $r$ approximation is not even guaranteed. Moreover, it is generally computationally NP-hard to determine the CP rank \citep{hillar2013most} of a given tensor. In contrast, the Tucker rank and decomposition of a tensor can be {\it always} and {\it easily} determined by HOSVD, and the tensors with Tucker ranks bounded by a constant constitute a manifold. However, when representing an $m$-way tensor of Tucker rank $\leq(r,\ldots,r)$, the required number of parameters is $O(r^m + mdr)$ and 
grows exponentially fast with respect to the order $m$. As a result, Tucker decomposition consumes a great deal of memory and computing resources for the tensors of very high orders. 

The recent decade has witnessed an increasing attraction in a new tensor format \citep{oseledets2009compact, oseledets2011tensor}, referred to as the {\it tensor train} (TT, see formal definition in Section~\ref{sec:pre}), which enjoys the advantages of both CP and Tucker formats. The TT format was inspired by the matrix product state (MPS, \citealt{perez2006matrix}), an extremely powerful method to represent the state of a large quantum system \citep{gross2011recovering, koltchinskii2015optimal}. The model complexity of TT formats grows {\it linearly} with respect to the tensor order. For instance,  the parameters needed to store an $m$-th order tensor of  TT rank $\leq(r,\ldots,r)$ is $O(mdr^2)$, saving significant space than the Tucker format. More importantly, TT rank is always and easily attainable like the Tucker rank. Indeed, similarly as the Tucker decomposition, an algorithm based on the sequential SVDs, named as tensor train SVD (TTSVD, \citealt{oseledets2011tensor}), is applicable to decide the TT rank, and hence the TT decomposition,  of a tensor. TTSVD can also be viewed as a quasi-optimal approximation of a given tensor \citep{oseledets2011tensor}. Since the tensors of fixed TT ranks construct a manifold \citep{holtz2012manifolds}, numerous manifold-based algorithms are readily adaptable to the TT formats. Due to the aforementioned advantages of TT formats, there has emerged a vast literature in tensor computation, application and theory exploring the TT-format decomposition. The earliest appearance of TT format or MPS can be traced back to the seminal works in physics, specifically in the simulations of quantum dynamics for very large systems \citep{vidal2003efficient, vidal2004efficient,perez2006matrix}. The formal definition of TT ranks is later proposed by \citet{oseledets2011tensor}, which has inspired a great many works for the computation and applications of low TT-rank tensors. \citet{bengua2017efficient} introduced the nuclear norm for TT-format tensors based on the unfolded matrices, and proposed simultaneous matrix factorization for tensor completion, showcasing its superior performances in the recovery of color images and videos. Inspired by their ideas and motivated by the local smoothness of image data, \citet{ding2019low} further proposed a total variation regularization for the image and video inpainting problems. Based on tensor factorization, \citet{wang2016tensor} investigated low TT-rank tensor completion by alternating minimization which updates the estimated components sequentially. This is a non-convex approach where the good initialization plays a critical role, and they adopted the spectral initialization by TTSVD. A gradient descent algorithm was proposed in the work of \citet{yuan2019high} for TT-format tensor completion with a random initialization, which often performs poorly when the sample size is small. More recently, \citet{ko2020fast} provided a novel but heuristic initialization method for the applications in recovering images and videos that is efficient in numerical experiments. These prior works mostly focus on the methodology and algorithm designs without, or with rather limited, theoretical justification. Towards that end, \citet{imaizumi2017tensor} introduced a new convex relaxation by the Schatten norm of matrices obtained by the separations (see definition in Section~\ref{sec:pre}) of a TT-format tensor. They investigated a convex method for tensor completion showing that the TT-format tensor is {\it approximately} recovered by observing $O(rd^{\lceil m/2\rceil}\cdot\textrm{\small Polylog(d)}))$ randomly sampled entries. More recently, \citet{zhou2020optimal} established the statistically optimal convergence rates of tensor SVD by the fast higher-order orthogonal iteration algorithm in the TT-format. 

A particularly important class of efficient algorithms for learning low-rank tensor decomposition is based on the Riemannian optimization. Different from projected gradient descent \citep{chen2019non}, in each step, the gradient is projected to the tangent space.
The gist of these algorithms is to view the tensor of interest as a point on the Riemannian manifold \citep{holtz2012manifolds}, for example, the collection of tensors with a bounded Tucker-rank or TT-rank.,  and then to adapt the Riemannian gradient descent algorithm (RGrad) for minimizing the associated objective function. An incomplete list of representable works of RGrad for matrix and tensor applications includes \citet{steinlechner2016riemannian, wei2016guarantees, wei2016guarantee,kressner2014low, cai2021generalized}. Similarly, the TT-format tensor completion can be recast to an unconstrained problem over the Riemannian manifold and is numerically solvable via RGrad \citep{wang2019tensor}. This algorithm is similar to most non-convex algorithms in the sense that it starts from a good initial point on the manifold, iteratively updates the new estimate by descending along the {\it Riemannian gradient} and retracts it back to the target manifold by TTSVD. Here Riemannian gradient is simply the projection of vanilla gradient onto the tangent space of the manifold. The Riemannian gradient is low-rank and hence can significantly speedup the downstream task of TTSVD. Fortunately, Riemannian gradient, as shown in the works of \citet{lubich2015time, steinlechner2016riemannian}, can be fast computed using QR decomposition and recursive SVD. All these foregoing properties of computational efficiency make RGrad a perfect choice for TT-format tensor completion, especially for tensors of a very high order. Interestingly, we  also observe considerable time saving by TT-format tensor completion compared with the Tucker-format tensor completion, even when the Riemannian gradient descent algorithm is applied for both scenario. See the numerical experiments in Section~\ref{sec:numerical}. 
 
\subsection{Our Contributions}
Despite the rich literature in algorithm designs for TT-format tensor completion and their empirical efficiency, the theoretical understanding, for example, sample size requirement, initialization condition, convergence behaviour and recovery guarantee, of those algorithms is relatively scarce. While abundant results are available for the CP-format and Tucker-format tensor completion, they cannot be directly translated into the case of TT-format for, at least, four reasons. First, the gap of model complexity between TT-format and other formats suggests that the sample size requirement can be different. Secondly, another fundamental condition making tensor completion possible is the so-called {\it incoherence condition}. It can be straightforwardly defined by the components of tensor decomposition in the CP-format and Tucker-format. Since the decomposition of a TT-format tensor is recursive, a suitable adaptation of the incoherence condition is necessary which causes additional complications in the theoretical understanding. Thirdly, the algorithm design (for instance, RGrad) for TT-format tensor completion is quite special, also due to the recursive nature of TT decomposition, involving the recursive reshapes by tensor separation. It poses extra challenges in analyzing the convergence behaviour of any iterative algorithms for TT-format tensor completion. Finally, obtaining a warm initialization is crucial. 
The naive spectral initialization suggested by \citet{wang2016tensor} requires a large sample size observed by empirical simulations. In addition, the second-order moment method \citep{xia2017statistically} of spectral initialization for Tucker-format is not directly applicable for the tensors of TT-format.

In this manuscript,  we investigate the Riemannian gradient descent algorithm for TT-format tensor completion and provide,  to our best knowledge,  the first theoretical guarantees of this algorithm under a nearly optimal sample size requirement.  More specifically,  for an $m$-th order tensor $\calT^{\ast}$ of size $d\times\cdots\times d$ in the TT-format with TT rank bounded by $r$,  the RGrad algorithm can exactly recover $\calT^{\ast}$ by observing $O(\kappa_0^{4m-4}r^{(5m-9)/2}d^{m/2}\cdot \textrm{\small Polylog(d)}+\kappa_0^{4m+8}r^{3m-3}d\cdot \textrm{\small Polylog(d)})$ randomly sampled entries where $\kappa_0$ denotes the condition number of $\calT^{\ast}$.  
Our contributions can be summarized into three folds.  First of all,  by a more sophisticated approach of analysis,  we show that the RGrad algorithm for tensor completion converges linearly as long as the initialization is just reasonably good,  namely $\|\calT_0-\calT^{\ast}\|_{\rm F}=o(1)\cdot \|\calT^{\ast}\|_{\rm F}$.  This significantly improves existing results \citep{wei2016guarantees, wei2016guarantee}  on the RGrad algorithm for matrix completion ($m=2$) which require a stringent initialization condition $\|\calT_0-\calT^{\ast}\|_{\rm F}=o(n^{1/2}/d)\cdot \|\calT^{\ast}\|_{\rm F}$.  Secondly,  we prove that the error of the iterates produced by RGrad algorithm contracts at a constant rate that is strictly smaller than $1$,  under a nearly optimal sample size condition.  The contract rate is free of the tensor condition number,  improving over prior works \citep{jain2014provable, cai2021nonconvex, xia2017polynomial} the iteration complexity for completing an  ill-conditioned tensor.  The attained iteration complexity matches the best one achieved by a recently proposed scaled gradient descent algorithm \citep{tong2021scaling}.  Finally,  inspired by the idea in \citet{xia2017polynomial},  we propose a novel initialization,  based on a {\it sequential} second order moment method,  as the input of RGrad algorithm for tensor completion.    This approach, unlike the one-pass spectral estimate in the work of \citet{xia2017polynomial},  involves a recursive estimate of the components of TT decomposition, posing additional challenges in the proof.  Our method guarantees a warm initialization with a nearly optimal sample size requirement.  Our result fills the void of guaranteed initialization for TT-format tensor completion.  Existing initialization approaches \citep{ko2020fast,wang2016tensor,yuan2019high} are either heuristic or miss theoretical justifications.  Finally,  we investigate the statistical property of RGrad and sequential second-order moment initialization assuming the observed entries contain sub-Gaussian noise.  Statistically (near) optimal rate is established,  which is proportional of the noise variance.

For readers' convenience,  we showcase our theoretical results in Table \ref{table:comparison} and compare with representable literature of tensor completion with respect to the tensor formats,  sample complexity and iteration complexity.  For ease of exposition, Table \ref{table:comparison} only focus on third order tensors with $m=3$ and of size $d\times d\times d$.

\begin{table}[h!]
	\centering
	\begin{tabular}{c|c|c|c} 
		\hline
		Data type & Algorithms & \thead{Sample\\ complexity} & \thead{Iteration\\ complexity} \\ [0.5ex] 
		\hline\hline
		\thead{CP format with \\{\it orthogonal}\\ {\it components}} &\thead{Alternating minimization\\ \citep{jain2014provable} }& $d^{3/2}r^5\kappa_0^4\log^4(d)$
		& $\log(\frac{r\kappa_0}{\epsilon})$ \\ 
		\hline
		CP format & \thead{Vanilla gradient descent\\ \citep{cai2021nonconvex}} & $C_{\kappa_0}d^{3/2}r^4\log^4d$ & $\kappa_0^{8/3}\log(\frac{1}{\epsilon})$\\
		\hline
		Tucker & \thead{Grassmannian gradient descent\\ \citep{xia2017polynomial}}& 
		$d^{3/2}r^{7/2}\kappa_0^4\log^{7/2}(d)$
		& N/A \\
		\hline
		Tucker & \thead{Scaled gradient descent\\ \citep{tong2021scaling}} & $d^{3/2}r^2\kappa_0(\sqrt{r}\vee\kappa_0^2)\log^3(d)$ 
		& $\log(\frac{1}{\epsilon})$ \\
		\hline
		{\bf Tensor Train} & \thead{Riemannian gradient descent \\ ({\bf this paper})} & $d^{3/2}r^{3}\kappa_0^8\log^5(d)$ 
		& $\log(\frac{1}{\epsilon})$ \\ [1ex] 
		\hline
	\end{tabular}
	\caption{Comparisons between TT-format RGrad algorithm and prior algorithms for tensor completion with respect to the tensor formats,  sample complexity and iteration complexity. For ease of exposition, we only present the results for $3$rd-order tensor $\calT^{\ast}$ of size $d\times d\times d$ with CP rank $r$, Tucker rank $(r,r,r)$ and TT rank $(r,r)$, respectively. For sample complexity, we state the number required in terms of $d,r$ only for clarity.
		For the iteration complexity, we consider the number of iterations needed to output $\hat \calT$ such that $\fro{\hat \calT - \calT^*}\leq \epsilon\sigmamin$. Here $\kappa_0$ and $\sigmamin$ are the condition number and minimum singular value of $\calT^*$, and may be defined in different ways for different formats. See more details for the case of TT-format in Section~\ref{sec:pre}. We note that, in the work of \citet{cai2021nonconvex},  the condition number $\kappa_0$ is assumed $\asymp 1$. In fact, their sample complexity and iteration complexity, after checking the proof, indeed depend on $\kappa_0$. Moreover, Theorem~1.1 in the work of \citet{jain2014provable} states the iteration complexity as $\log(\sqrt{r}\fro{\calT^*}/(\epsilon\sigmamin))$, where $\fro{\calT^*}/\sigmamin$ has an order $\sqrt{r}\kappa_0$.}
	\label{table:comparison}
\end{table}

\subsection{Organization of the Paper}
The rest of this manuscript is organized as follows. Section~\ref{sec:pre} reviews the basics of TT-format tensors, its decomposition and TTSVD. The formulation of tensor completion and the incoherence of TT-format tensors are presented in Section~\ref{sec:tc_formulation}. 
In Section \ref{sec:alg}, we explain in details the RGrad algorithm for TT-format tensor completion, and the sequential second order moment method for  initialization. Section \ref{sec:thoery} presents the main theorems regarding the convergence of RGrad algorithm and the performance bound of spectral  initialization, and establishes the specific sample size requirement. We demonstrate the statistical optimality of RGrad for noisy tensor completion in Section~\ref{sec:noisy}.   
Comprehensive numerical experiments are displayed in Section \ref{sec:numerical}.  All proofs and technical lemmas are relegated to the appendix.  

\subsection{Notations}
Throughout this manuscript, we shall use the calligraphic letters ($\calT,\calX$) to denote tensors of size $d_1\times \cdots \times d_m$, the capital letters ($T,M$) to denote the components  (see formal definition in Section~\ref{sec:pre})  of TT tensors or matrices,  blackboard bold-face letters ($\RR,\MM$) for sets,  and the lower case bold-face letters ($\x,\y$) to denote vectors. The $j$-th canonical basis vector is denoted by $e_j$, and we omit the ambient space it lies in whenever the context is clear. For a positive integer $d$, denote $[d]:=\{1,\cdots,d\}$. The standard basis in the tensor space $\ambspace$ is denoted by $\{\calE_{\omega}: \omega\in[d_1]\times \cdots\times[d_m]\}$, where $\calE_{\omega}$ is a binary tensor of size $d_1\times\cdots\times d_m$ with only the $\omega$-th entry being $1$. We use $\calT(\boldsymbol{x}), \calT(\omega)$ for $\omega,\boldsymbol{x}\in [d_1]\times\cdots\times [d_m]$ as the entry of $\calT$.

Let $\fro{\cdot}$ denote the Frobenius norm of tensors or matrices. We use $\|\cdot\|_{\ell_p}$ to denote the $\ell_p$ norm of vectors for $0< p\leq \infty$ and $\|\cdot\|_{\ell_0}$ to represent the number of nonzero entries. The notations $C,C_1,\ldots$ are reserved for some positive and absolute constants which do not depend on the related parameters of the problem, but their actual values may change from line to line. Sometimes, these constants may depend on the tensor order $m$ and we shall write them as $C_{m}, C_{m,1},\cdots$. For positive integers $r_1,\cdots, r_{m-1}$,  we denote $\dmax := \max_{i=1}^m d_i, \rmax := \max_{i=1}^{m-1} r_i$ and $\rmin := \min_{i=1}^{m-1} r_i$. Moreover, define $d^* := d_1\cdots d_m$ and $r^* := r_1\cdots r_{m-1}$.

\section{Preliminaries of TT-format Tensors}\label{sec:pre}
We now briefly review the basic ideas of tensor trains and the various formats frequently used in the TT-format tensor representation. They play a critical role in the motivation of the TT-format RGrad algorithm. Interested readers can refer to \citet{oseledets2011tensor} and \citet{holtz2012manifolds} for more details and examples.

\hspace{0.2cm}

\noindent{\it TT-format.} 
For an $m$-th order tensor $\calT\in\RR^{d_1\times \cdots \times d_m}$, the TT-format rewrites it as a product of $m$ 3-way tensors, called the TT-format {\it components} and denoted by $T_1,\ldots, T_m$, where the $i$-th component $T_i\in\RR^{r_{i-1}\times d_i\times r_i}$ with the convention $r_0 = r_m = 1$, such that for all $\boldsymbol{x} = (x_1,\ldots,x_m)\in[d_1]\times \ldots\times [d_m]$, 
$$
\calT(\boldsymbol{x}) = \sum_{k_1,\ldots,k_{m-1}}T_1(x_1,k_1)T_2(k_1,x_2,k_2)\cdots T_m(k_{m-1},x_m),
$$
where the auxiliary index variables $k_i$ runs from $1$ to $r_i$. The representation can be simplified.  If we view each $T_i$ as a matrix valued function, usually called a \textit{component function}, defined by 
$$
T_i: [d_i]\rightarrow \RR^{r_{i-1}\times r_i}, \quad T_i(x_i)=T_i(:,x_i,:).
$$
Here we follow the Matlab syntax to denote $T_i(:,x_i,:)$ the sub-matrix of $T_i $ with the second index being fixed at $x_i$. 
Now for $\boldsymbol{x} = (x_1,\ldots,x_m)$, the value $\calT(\boldsymbol{x})$ can be compactly written in the matrix product form:
\begin{align}\label{eq:tt-format}
\calT(\boldsymbol{x}) = T_1(x_1)\cdots T_m(x_m),
\end{align}
where $T_1(x_1)$ is a row vector and $T_m(x_m)$ is a column vector. To simplify the notations, we often write the TT-format in short as  $\calT = [T_1,\ldots,T_m]$.

\hspace{0.2cm}

\noindent{\it Separation and TT rank.} 
The dimension $r_i$'s of the component $T_i$ are called the TT ranks of $\calT$. They are defined by the ranks of matrices obtained from the so-called {\it separation} of $\calT$. More specifically, the $i$-th separation of $\calT$, denoted by $\calT\lr{i}$, is a matrix of size $(d_1\cdots d_i)\times (d_{i+1}\cdots d_m)$ and defined by 
$$
\calT\lr{i}(x_1\cdots x_i, x_{i+1}\cdots x_m) = \calT(\boldsymbol{x}).
$$
Then, $r_i$ is defined by the rank of $\calT\lr{i}$. The collection $\r=(r_1,\cdots,r_{m-1})$ is called the TT-rank of $\calT$. For ease of exposition, we denote 
$\ranktt(\calT) = \boldsymbol{r}$.  As proved by Theorem~1 in \citet{holtz2012manifolds}, the TT rank is well defined for any tensor. Meanwhile,  a TT decomposition $\calT = [T_1,\ldots,T_m]$ is always attainable with $T_i\in\RR^{r_{i-1}\times d_i \times r_i}$ by the fast TTSVD algorithm, to be introduced in subsequent paragraphs. See Algorithm~\ref{alg:svd}. 

\hspace{0.2cm}

\noindent{\it Left and right unfoldings.} 
Note that the TT decomposition (\ref{eq:tt-format}) for a given tensor is not unique. Indeed, one can always multiply $T_i(x_i)$ from right with any invertible matrix $A$ and meanwhile multiply $T_{i+1}(x_{i+1})$ from left with the inverse matrix $A^{-1}$, rendering a distinct TT decomposition. For identifiability, we impose additional conditions on the TT-format components of $\calT$. To that end, we first define the {\it left} and {\it right unfolding} of a $3$-rd order tensor, reformatting the tensor into matrices.   
 For any $\calU\in\RR^{p_1\times p_2\times p_3}$ be a $3$-way tensor, the left and right unfolding linear operators $L: \RR^{p_1\times p_2\times p_3}\rightarrow \RR^{(p_1p_2) \times p_3}$, $R: \RR^{p_1\times p_2\times p_3}\rightarrow \RR^{p_1 \times (p_2p_3)}$ are defined by
\begin{align*}
	L(\calU)(jx,k) = \calU(j,x,k), \quad {\rm and}\quad R(\calU)(j,xk) = \calU(j,x,k).
\end{align*}
We say the component  $T_i$ is \textit{left-orthogonal} if $L(T_i)^{\top}L(T_i)$ is an identity matrix.  Similarly, $T_i$ is said {\it right-orthogonal} if $R(T_i)^{\top}R(T_i)$ is identity.

For identifiability of the TT-format components, we assume that $T_1,\cdots, T_{m-1}$ in eq. (\ref{eq:tt-format}) are {\it all} left-orthogonal. The resultant TT decomposition is called the \textit{left orthogonal decomposition} of $\calT$. Note that no condition is required for the last component $T_m$. 
The left orthogonal decomposition of $\calT$ can be easily obtained using Algorithm \ref{alg:svd} (TTSVD) without the truncation step. By Theorem 1 in  \citet{holtz2012manifolds}, such a decomposition of a TT-format tensor is unique up to the insertions of orthogonal matrices: for any two left-orthogonal decompositions of $\calT$ satisfying $\calT = [T_1,\ldots, T_m] = [\wt T_1,\ldots, \wt T_m]$, there exist orthogonal matrices $Q_1,\ldots,Q_{m-1}$ with $Q_i\in\RR^{r_i\times r_i}$ such that 
\begin{align*}
	T_1(x_1)Q_1 = \wt T_1(x_1), Q_{m-1}^{\top}T_m(x_m) = \wt T_m(x_m) \textrm{ and }	Q_{i-1}^{\top}T_i(x_i)Q_i = \wt T_i(x_i) \textrm{ for } 2\leq i\leq m-1.
\end{align*}
For any TT rank $\r=(r_1,\cdots,r_{m-1})$, define  $\mfd = \{\calT\in\RR^{d_1\times \cdots\times d_m}: \ranktt(\calT)\leq \boldsymbol{r}\}$ the set of tensors with TT rank $\leq \r$.  \citet{holtz2012manifolds} proves that $\mfd$ is a manifold of dimension $\sum_{i=1}^mr_{i-1}d_ir_i -\sum_{i=1}^{m-1}r_i^2$. 

\hspace{0.2cm}

\noindent{\it Relations between TT-rank, CP rank and Tucker rank.} 
The TT-rank, CP rank and Tucker rank are closely related.  For an arbitrary tensor $\calT\in\ambspace$, we denote its TT-rank by $r^{\textsf{tt}} = \ranktt(\calT)\in\NN^{m-1}$, CP rank by $r^{\textsf{cp}} = \rankcp(\calT)\in\NN$ and Tucker rank by $r^{\textsf{tk}} = \ranktk(\calT)\in\NN^m$. Then 
\begin{align*}
	&r^{\textsf{cp}}\leq r^{\textsf{tt}}_1\cdots r^{\textsf{tt}}_{m-1},\quad r^{\textsf{cp}}\leq r^{\textsf{tk}}_1\cdots r^{\textsf{tk}}_{m-1} &\\
	&r^{\textsf{tt}}_i \leq r^{\textsf{cp}}, &\forall~i\in[m-1], \\
	&r^{\textsf{tk}}_i\leq r^{\textsf{tt}}_{i-1}r^{\textsf{tt}}_i, \quad r^{\textsf{tk}}_i\leq r^{\textsf{cp}} &\forall~i\in[m], \\
	&r^{\textsf{tt}}_i\leq\min\{r^{\textsf{tk}}_1\cdots r^{\textsf{tk}}_i,r^{\textsf{tk}}_{i+1}\cdots r^{\textsf{tk}}_m\}, &\forall~i\in[m-1],
\end{align*}
where we use the convention that $r^{\textsf{tt}}_0 = r^{\textsf{tt}}_m = 1$.

\hspace{0.2cm}

\noindent{\it Left and right parts of a TT-format tensor.} 
The low-rank factorization of the separation $\calT\lr{i}$ is frequently used throughout the algorithm design and technical proof. It is attainable from 
the TT decomposition (\ref{eq:tt-format}) of $\calT$. To be exact, define the matrix $T^{\leq i}$ of size $(d_1\cdots d_i)\times r_i$, known as the $i$-th {\it left part} of $\calT$, row-wisely by  
\begin{align*}
	T^{\leq i}(x_1\cdots x_i, :) =T_1(x_1)T_2(x_2)\cdots T_i(x_i)
\end{align*}
Similarly, we define the matrix $T^{\geq i}$ of size $r_{i-1}\times (d_i\cdots d_m)$, know as the $i$-th {\it right part} of $\calT$,  column-wisely by
\begin{align*}
	T^{\geq i}(:, x_i\cdots x_m) = T_i(x_i)T_{i+1}(x_{i+1})\cdots T_m(x_m)
\end{align*}
By default, we set $T^{\leq 0} = T^{\geq m+1} = [1]$. With these notations, the $i$-th separation of $\calT$ can be factorized as 
\begin{align*}
	\calT^{\lrangle{i}} = T^{\leq i} T^{\geq i+1}.
\end{align*}
By definition, there exists a recursive relations between the left parts of $\calT$ given by $T^{\leq i} = (T^{\leq i-1}\otimes I_{d_i})L(T_i)$ and, similarly, between the right parts of $\calT$ given by $T^{\geq i} = R(T_i)(I_{d_i}\otimes T^{\geq i+1})$. Here $\otimes$ denotes the Kronecker product such that $T^{\leq i-1}\otimes I_{d_i}$ is a matrix of size $(d_1\cdots d_i)\times (r_id_i)$. 
Another useful fact is when the TT decomposition (\ref{eq:tt-format}) is a left orthogonal decomposition, its left parts are also orthogonal. See Lemma~\ref{lem:tt-orthogonal}. This can be proved by induction and the recursive equation $T^{\leq i} = (T^{\leq i-1}\otimes I)L(T_i)$.
\begin{lemma}\label{lem:tt-orthogonal}
	Let $\calT = [T_1,\ldots,T_m]\in\mfd$ be a left orthogonal decomposition. Then we have $T^{\leq i \top}T^{\leq i} = I_{r_i}$ for all $i=1,\cdots,m-1$. 
\end{lemma}

\hspace{0.2cm}

\noindent{\it TTSVD.} 
Given an arbitrary $m$-th order tensor $\calT$,  we are often interested in obtaining an approximation of $\calT$ by a TT-format tensor with TT rank $\leq \r=(r_1,\cdots,r_{m-1})$ for some pre-determined positive integers $r_i$'s. Meanwhile, the desired low TT-rank approximation better be readily representable with a left orthogonal decomposition. 
This can be obtained by the TTSVD algorithm proposed by \citet{oseledets2009compact}. For completeness, we here restate the algorithm as in Algorithm~\ref{alg:svd}. 	At first glance Algorithm~\ref{alg:svd} may seem slightly different from the original one proposed by \citet{oseledets2009compact}. But they are indeed equivalent due to the following fact:
$$
%\hat W_{i-1} = \hat T^{\leq i-1}\calT\lr{i-1} \text{~and~}
	\reshape((\hat T^{\leq i-1})^{\top}\calT\lr{i-1},[r_{i-1}d_i, d_{i+1}\cdots d_m]) = (\hat T^{\leq i-1}\otimes I_{d_i})^{\top} \calT^{\lrangle{i}}
$$ 
An immediate question is  whether the output $\hat \calT$ by Algorithm \ref{alg:svd} is the best low TT rank-$\r$ approximation of $\calT$. Unfortunately, this is generally untrue. In fact, based on existing evidence \citep{hillar2013most}, computing the best low TT rank-$\r$ approximation of an arbitrary tensor is generally NP-hard. 

Another technical issue, frequently met in the proof, is that when $\calT=\calT^{\ast}+\calE$ where $\calT^{\ast}\in \mfd$ and $\calE$ is a {\it small} but  arbitrary perturbation, then how accurately does the output $\hat\calT$ from Algorithm \ref{alg:svd} approximate the true low-rank $\calT^{\ast}$? Interestingly, using the spectral representation formula \citep{xia2021normal}, we prove that $\|\hat\calT-\calT^{\ast}\|\leq (1+o(1))\cdot \|\calE\|_{\rm F}$ in Lemma \ref{lemma:tt-perturbation} with a sharp leading constant $1$ and being free of the condtion number of $\calT^{\ast}$, which might be of independent interest. To our best knowledge, ours is the first result of this kind. 

\begin{algorithm}
	\caption{TT-SVD ( $\ttsvd_{\r}$)}\label{alg:svd}
	\begin{algorithmic}
		\STATE{\textbf{Input: } an arbitrary $\calT\in\RR^{d_1\times\cdots\times d_m}$and desired TT rank $\r = (r_1,\ldots,r_{m-1})$.}
		\STATE{Set $\hat T^{\leq 0}=[1]$}
		\FOR{$i=1,\ldots,m-1$}
		\STATE{Let $L(\hat T_i)$ be the top $r_i$ left singular vectors of the matrix $(\hat T^{\leq i-1}\otimes I_{d_i})^{\top} \calT^{\lrangle{i}}$} %, i.e., $(\hat T^{\leq i-1}\otimes I_{d_i})^T\calT^{\lrangle{i}} = L(\hat T_i)\hat W_i + E_i$ with $L(\hat T_i)\hat W_i$ the best rank $r_i$ approximation of  $(\hat T^{\leq i-1}\otimes I_{d_i})^T\calT^{\lrangle{i}}$.}
		\STATE{Set $\hat T^{\leq i}=(\hat T^{\leq i-1}\otimes I_{d_i})L(\hat T_i)$}
		%		\STATE{$M_{i+1} = (M_iL(\hat U_i))\otimes I_{d_{i+1}}$.}
		\ENDFOR
		\STATE{$\hat T_m = (\hat T^{\leq m-1})^{\top} \calT\lr{m-1}$.}
		\STATE{\textbf{Output: }$\hat \calT = [\hat T_1,\ldots,\hat T_m]\in \mfd$.}
	\end{algorithmic}
\end{algorithm}

\hspace{0.2cm}

\noindent{\it Tensor condition number.}  
The signal strength of a TT-format tensor $\calT$ with TT-rank $\r=(r_1,\cdots, r_{m-1})$ is defined by the smallest singular value among all the matrices obtained from separation, that is, 
$$
\underline{\sigma}(\calT):=\min\nolimits_{i=1}^{m-1} \sigma_{r_i}(\calT\lr{i}),
$$
where $\sigma_k(\cdot)$ returns the $k$-th singular value of a matrix.  Similarly, the largest singular value of $\calT$ is defined by $\overline{\sigma}(\calT):=\max\nolimits_{i=1}^{m-1} \sigma_1(\calT\lr{i})$.  Then the condition number of $\calT$ is defined by $\kappa(\calT):=\underline{\sigma}(\calT)^{-1}\overline{\sigma}(\calT)$.  The condition number plays a critical role in the convergence behaviour of many iterative algorithms for tensor completion.  See, for instance, Table \ref{table:comparison}.

\section{TT-format Tensor Completion and Incoherence Condition}\label{sec:tc_formulation}
The goal of tensor completion is to (exactly) recover an underlying tensor by only observing a small subset of its entries. Denote by $\calT^{\ast}$ the true underlying tensor of size $d_1\times\cdots\times d_m$ with TT rank $\r=(r_1,\cdots,r_{m-1})$ which, for simplicity, is assumed known beforehand and satisfy $r_i\ll d_i$. The observed entries of $\calT^{\ast}$ are assumed to be uniformly sampled with replacement, a prevailing sampling scheme in the literature \citep{koltchinskii2011nuclear, xia2017polynomial, xia2017statistically} for its convenience in modelling randomness. More exactly, let $\Omega=\{\omega_i: i=1,\cdots,n\}$ where $\omega_i$ is independently and uniformly sampled from the set of collections $[d_1]\times\cdots\times [d_m]$. By observing only the entries $\{\calT^{\ast}(\omega_i)\}_{i=1}^n$, we aim to design computationally efficient methods to recover the whole tensor $\calT^{\ast}$. Intuitively, tensor completion becomes easier when more entries are observed. Oftentimes, the number of observed entries $n$, known as sample size, is significantly smaller than $d^{\ast}:=d_1\cdots d_m$. 

Note that the problem can be ill-posed if $\calT^{\ast}$ has, for instance, only one entry that is non-zero, then it is impossible to recover $\calT^{\ast}$ unless this non-zero entry is indeed observed. To resolve this issue, it is usually assumed that the information $\calT^{\ast}$ carries spreads fairly among almost all its entries. One concept characterizing this information spread is by the spikiness of $\calT^{\ast}$ \citep{yuan2016tensor,cai2021generalized} defined by 
$$
\spiki(\calT^{\ast}) := (d^*)^{(1/2)}\|\calT^*\|_{\ell_{\infty}}/ \fro{\calT^*}.
$$
If the spikiness of $\calT^{\ast}$ is upper bounded by a constant, it means that most of its entries have comparable magnitudes. Oftentimes, another related concept characterizing the information spread, called incoherence condition, is more frequently used. The exact definition of incoherence condition usually relies on the tensor formats. See the definitions for CP-format tensors in the works of \citet{jain2014provable,cai2021nonconvex} and for Tucker-format tensors in the works of  \citet{yuan2016tensor,xia2017statistically}. To formalize the definition of incoherence for TT-format tensors, we begin with reviewing the incoherence condition of a matrix. 

For a matrix $M$ of size $p_1\times p_2$ and rank $r$ whose compact SVD is given by $M=U\Sigma V^{\top}$, the incoherence of $M$ is defined by 
$$
\incoh(M) :=\max\left\{\sqrt{p_1/r}\cdot \max\nolimits_{i\in[p_1]}\|U^{\top}e_i\|_{\ell_2}, \sqrt{p_2/r}\cdot \max\nolimits_{j\in[p_2]}\|V^{\top}e_j\|_{\ell_2}\right\}.
$$
We say $M$ is incoherent with a constant $\mu$ if $\incoh(M)\leq \mu^{1/2}$. Note that incoherence is defined through the singular subspace of $M$. Therefore, the incoherence of $M$ can be equivalently obtained from any low-rank decomposition $M=U_1\Sigma_1 V_1^{\top}$ satisfying $U_1U_1^{\top}=UU^{\top}$ and $V_1V_1^{\top}=VV^{\top}$. This simple property will ease our understanding of incoherence for TT-format tensors.  Now the incoherence  of the TT-format tensor $\calT^{\ast}\in \mfd$ is defined by 
$$
\incoh(\calT^{\ast}):=\max \big\{\incoh(\calT^{\ast \langle i\rangle}): i=1,2,\cdots, m-1 \big\}.
$$
Similarly, we say $\calT^{\ast}$ is incoherent with a constant $\mu$ if $\incoh(\calT^{\ast})\leq \mu^{1/2}$. We write the left orthogonal decomposition of $\calT^{\ast}$ by $\calT^{\ast}=[T_1^{\ast},\cdots, T_m^{\ast}]$. Then for $1\leq i\leq m-1$, the $i$-th separation $\calT^{\ast\langle i\rangle}$ can be written as $\calT^{\ast\langle i\rangle}=T^{\ast\leq i}\Lambda_{i+1}^{\ast}V_{i+1}^{\ast \top}$ where $\Lambda_{i+1}^{\ast}$ is invertible and of size $r_i\times r_i$,  and $T^{\ast\leq i\top}T^{\ast\leq i} = V_{i+1}^{\ast \top}V_{i+1}^{\ast} = I_{r_i}$.  Then the condition $\incoh(\calT^{\ast})\leq \mu^{1/2}$ implies that 
\begin{align*}
	\max_{i=1}^{m-1}\left\{\max_{k\in[d_1\ldots d_i]}\|T^{\ast \leq i \top}e_k\|_{\ell_2}(d_1\ldots d_i/r_i)^{1/2}, \max_{k\in[d_{i+1}\ldots d_m]}\|V^{\ast\top}_{i+1}e_k\|_{\ell_2}(d_{i+1}\ldots d_m/r_i)^{1/2}\right\}\leq \sqrt{\mu}.
\end{align*}

The spikiness and incoherence of TT-format tensors are closely related and summarized in the following lemma. 
\begin{lemma}[Spikiness implies incoherence]\label{lemma:spiki implies incoh}
	Let $\calT^*\in\mfd$ satisfy $\spiki(\calT^*) \leq \nu$. Then we have
	$$\incoh(\calT^*) \leq \nu\kappa_0,$$
	where $\incoh(\calT^*)$ is the incoherence parameter of $\calT^*$ and $\kappa_0$ is the condition number of $\calT^*$ defined by $\kappa_0 = \sigmamax(\calT^{\ast})/\sigmamin(\calT^{\ast})$.
\end{lemma}

Now that assuming the incoherence property of the $\calT^{\ast}$,  the problem of tensor completion is well-posed.  To recover $\calT^{\ast}$, the natural idea is to search for a low TT rank tensor that is consistent with the observed entries $\{\calT^{\ast}(\omega_i)\}_{i=1}^n$.  This can be formalized as a non-convex optimization program of a least squares estimator
\begin{align}\label{prob:tensor completion}
	&\min_{\calT\in\RR^{d_1\times \cdots \times d_m}}f_{\Omega}(\calT) := \frac{1}{2}\inp{\calT - \calT^*}{\calP_{\Omega}(\calT-\calT^*)}=\frac{1}{2}\sum_{\omega\in\Omega}\big(\calT(\omega)-\calT^{\ast}(\omega)\big)^2\\
	&\hspace{4cm}\text{such that } \ranktt(\calT)\leq \boldsymbol{r}, \notag
\end{align}
where the operator $\calP_{\Omega}$ is defined by $\calP_{\Omega}(\calT) := \sum_{i=1}^n\calT(\omega_i)\cdot\calE_{\omega_i}$ and the inner product of any two tensors $\calT_1,\calT_2$ of the same size $d_1\times\cdots\times d_m$ is defined by
$
\inp{\calT_1}{\calT_2} := \sum\nolimits_{\boldsymbol{x}\in[d_1]\times\cdots\times[d_m]}\calT_1(\boldsymbol{x})\calT_2(\boldsymbol{x}).
$
We remark that,  due to the sampling with replacement,  $f_{\Omega}(\calT)$ may not be equal to $\|\calP_{\Omega}(\calT-\calT^{\ast})\|_{\rm F}^2/2$. This is slightly different from the setting of sampling without replacement. To simplify the notation, we shall just write $f(\cdot)$ in short for $f_{\Omega}(\cdot)$. 

The optimization program (\ref{prob:tensor completion}) is highly non-convex due to the constraint of TT rank, which is usually solvable only locally. Since the objective function in (\ref{prob:tensor completion}) is smooth, the major concern in algorithm design is usually placed on the effective enforcement of the rank constraint. 
A particularly popular class of algorithms is based on the projected gradient descent including the singular value projection (SVP, \citealt{meka2009guaranteed}) and iterative hard thresholding (IHT, \citealt{goldfarb2011convergence}). These algorithms consist of mainly two steps: (1) update the estimate along the vanilla gradient descent direction and (2) retract the new estimate to the target tensor/matrix manifold by low-rank approximation, such as HOSVD or TTSVD.  Oftentimes, these method suffer from high computational cost because the vanilla gradient is often full-rank and so is the resultant new estimate. Consequently, the retraction in the second step relies on the SVD of a very large and full-rank matrix, at each iteration.

\section{Riemannian Gradient Descent and Spectral Initialization}\label{sec:alg}
As explained, the vanilla gradient descent algorithm often suffers from high computational burdens.  
To reduce the computational costs, a modified algorithm, named as the Riemannian gradient descent, was proposed by \citet{edelman1998geometry, kressner2014low,vandereycken2013low}, which explores the local geometry structure of low-rank tensors/matrices. The essential modification is to take advantage of the manifold structure and project the vanilla gradient onto the tangent space leading to the so-called Riemannian gradient (RGrad).  Compared with the vanilla gradient, the Riemannian gradient is also low-rank, rendering considerable speedup in the downstream task of retraction. Inspired by this idea, the Riemannian gradient descent algorithm has been widely applied for matrix/tensor completion \citep{wei2016guarantees, kressner2014low},  generalized low-rank tensor estimation \citep{cai2021generalized} and etc. 
We are in position to explain how RGrad can be adapted to the TT-format tensor completion.  

\subsection{TT-format Riemannian Gradient Descent}
%The problem \eqref{prob:tensor completion} is highly non convex and can be optimized locally. If we can obtain a good initialization, then we can design an iterative algorithm to find a local minimum.
RGrad is an iterative algorithm and, given the current estimate $\calT_l\in\mfd$ at an iteration, the algorithm updates the estimate by three steps: (1) compute the Riemannian gradient; (2) descent along the Riemannian gradient and (3) retract back to the TT-format tensor manifold. See the pseudocodes in Algorithm~\ref{alg:rgradtt}. 

In the first step, the Riemannian gradient \citep{absil2009optimization} is obtained via projecting the vanilla gradient $\nabla f(\calT_l)$ onto the tangent space, denoted by $\TT_l$, of $\mfd$ at the point $\calT_l$. The tangent space $\TT_l$ has an explicit form so that the projection $\calP_{\TT_l}(\nabla f(\calT_l))$ onto $\TT_l$ is attainable by fast computations. For cleaner presentation here, we sink the detailed explanation of $\calP_{\TT_l}(\cdot)$ to Section~\ref{sec:computation of pta}. The follow-up gradient descent step is easy, and we demonstrate that a fixed stepsize $\alpha=0.12n^{-1}d^{\ast}$ suffices for convergence where $d^{\ast}=d_1\cdots d_m$ and the constant $0.12$ is slightly adjustable in practice. The last step, {\it retraction}, is of crucial importance. The updated estimate $\calW_l=\calT_l-\alpha_l\cdot \calP_{\TT_l}(\nabla f(\calT_l))$ after the first two steps is generally not an element in $\mfd$, and in fact, the TT-rank of $\calW_l$ is larger than the desired TT-rank $\r$, violating the rank constraints. The retraction procedure enforces the TT-rank constraint by projecting $\calW_l$ back to $\mfd$. This can be done by TTSVD, denoted by $\ttsvd_{\r}$ in Algorithm~\ref{alg:svd}. 

Due to technical reasons, Algorithm \ref{alg:rgradtt} involves an additional procedure, called {\it trimming}. The trimming operator $\trim_{\zeta}$ which acts entry-wisely on a tensor is defined as follows:
\begin{align*}
	\trim_{\zeta}(\calT) = \left\{
	\begin{array}{rlc}
		&\zeta\cdot \text{sign}(\calT(\boldsymbol x)), \quad\text{~if~} |\calT({\boldsymbol{x}})|\geq \zeta,\\
		&\calT(\boldsymbol{x}), \quad\text{otherwise}
	\end{array}\right.
\end{align*}
Basically, it trims those large entries of $\calT$ and thus maintains an acceptable spikiness or incoherence.  The trimming step is necessary to guarantee the incoherence property of the new estimate $\calT_l$ at each iteration, playing a critical role in proving the local convergence of Algorithm~\ref{alg:rgradtt}. It might be possible to directly prove through a more sophisticated analysis \citep{cai2021nonconvex} that the incoherence property holds automatically without trimming. We indeed observe in numerical experiments that the RGrad algorithm without trimming runs nearly the same as the trimmed version, implying the automatic validity of incoherence. 
However, to directly prove the incoherence without trimming is very challenging. Instead, we resort to trimming for simplicity. This procedure is completely for convenience of the technical proof and alters almost nothing in our numerical experiments.

\begin{algorithm}
	\caption{TT-format Riemannian Gradient Descent (RGrad)}\label{alg:rgradtt}
	\begin{algorithmic}
		\STATE{\textbf{Initialization: } $\calT_0\in\mfd$, spikiness parameter $\nu$ and maximum iterations $l_{\textsf{max}}$}
		\FOR{$l=0,1,\cdots, l_{\textsf{max}}$}
		\STATE{$\calG_l = \calP_{\Omega}(\calT_l-\calT^*)$}
		\STATE{$\alpha_l = 0.12\frac{d^*}{n}$}
		\STATE{$\calW_l = \calT_l - \alpha_l\cdot\calP_{\TT_l}\calG_l$}
		\STATE{Set $\wt\calW_l = \trim_{\zeta_l}(\calW_l)$ with $\zeta_l = \frac{10\fro{\calW_l}}{9\sqrt{d^*}}\nu$}
		\STATE{$\calT_{l+1} = \ttsvd_{\r}(\wt\calW_l)$}
		\ENDFOR
	\end{algorithmic}
\end{algorithm}

We further remark on the choice of spikiness parameter $\nu$ in the algorithm. Note that here $\nu$ is not necessary the true spikiness parameter of $\calT^{\ast}$. It suffices to take $\nu$ as a tuning parameter that is slightly larger than $\spiki(\calT^{\ast})$, assuming $\spiki(\calT^{\ast})$ is relatively small. For instance, when one is confident that the true spikiness is bounded by $O(1)$, then this tuning parameter can be set as $\nu\asymp \log \bar d$ in numerical implementation. 

Theorem~\ref{thm:localconvergence} in Section \ref{sec:thoery} demonstrates that, under mild conditions on the initialization and a suitable sample size requirement, Algorithm~\ref{alg:rgradtt} guarantees that $\|\calT_{l+1}-\calT^{\ast}\|_{\rm F}\leq \gamma\cdot \|\calT_{l}-\calT^{\ast}\|_{\rm F}$ for an absolute constant $\gamma\in(0,1)$, implying a linear convergence of Algorithm \ref{alg:rgradtt}. However, it requires the initialization $\calT_0$ to be close enough to the underlying tensor $\calT^{\ast}$. Existing literature \citep{wang2016tensor} suggests a naive spectral initialization by the observed tensor $n^{-1}d^{\ast}\calP_{\Omega}(\calT^{\ast})$. It turns out that a naive spectral initialization performs poorly and only works when the sample size $n$ is exceedingly large. We now propose a novel approach, inspired by \citet{xia2017statistically} and called sequential second-order moment method, to produce a good initialization requiring only an almost optimal sample size.

\subsection{Sequential Spectral Initialization}\label{sec:initialization}
Recall that the left orthogonal decomposition of $\calT^{\ast}$ is written as $\calT^{\ast}=[T_1^{\ast},\cdots, T_m^{\ast}]$. Our initialization method yields good estimates for $T_1^*,\ldots,T_m^*$ up to orthogonal rotations. For ease of illustration, we denote by $\wt\calT = n^{-1}d^*\calP_{\Omega}(\calT^*)$ the scaled observed tensor. Due to the uniform sampling scheme, it is clear that $\EE\wt\calT = \calT^*$. Since $T_1^*$ contains the left singular vectors of $\calT^{\ast \langle i\rangle}$, the naive spectral initialization, suggested by \citet{wang2016tensor}, takes the top $r_1$ left singular vectors of $\wt\calT\lr{1}$ as an estimation for $T_1^*$. Unfortunately, this method usually performs poorly because the matrix $\calT^{\ast\langle 1\rangle}$ is of size $d_1\times (d_2\cdots d_m)$ whose number of columns can be significantly larger than $d_1$. 
 While we are only interested in a parameter from a $d_1$-dimensional space, the quality of spectral estimate from $\tilde\calT\lr{1}$ is affected by the larger dimension between $d_1$ and $d_2\cdots d_m$.
 
The second-order spectral moment method is inspired \cite{xia2017polynomial} by the fact that $T_1^*$ also contains the eigenvectors of $\calT^{\ast\langle 1\rangle}\calT^{\ast\langle 1\rangle \top}$ which is a square matrix of size $d_1\times d_1$. Therefore, \citet{xia2017polynomial} proposed a U-statistic to estimate the square matrix $\calT^{\ast\langle 1\rangle}\calT^{\ast\langle 1\rangle \top}$ directly and then applied the spectral initialization. Unlike the Tucker-format where the components can be computed independent of each other (Section 4 of \citealt{xia2017polynomial}), the computation of components in TT-format decomposition depends recursively on each other, that is, $\hat T_{i+1}$ relies on the availability of $\hat T_i$, see the TTSVD procedure in Algorithm~\ref{alg:init}. Indeed, $L(T_{i+1}^*)$ is the top $r_{i+1}$ left singular vectors of the matrix $(T^{*\leq i}\otimes I)^{\top}\calT^{\ast \langle i+1\rangle}$, implying that we shall aim to estimate the eigenvectors of $(T^{*\leq i}\otimes I)^{\top}\calT^{\ast\langle i+1\rangle}\calT^{\ast\langle i+1\rangle \top}(T^{*\leq i}\otimes I)$. Conceptually, there is no difficulty to generalize the second-order moment method for this purpose. However, on the technical front, the dependence of $\hat T^{\leq i}$ on the original data creates substantial challenges in establishing a sharp spectral perturbation bound. For simplicity, we resort to the trick of sample splitting.

For ease of illustration, some additional notations are necessary.  Without loss of generality,  assume $n=(2m-1)n_0$ for some integer $n_0$.  We split the observed sample $[n]$ into equally-sized $2m-1$ sub-samples,  each of which is of size $n_0$\footnote{We note that the {\it minimal} sample size requirement for ensuring the effectiveness of $N_i$ is distinct for different $i$'s.  This is reasonable because the dimensions of $N_i$'s change with respect to $i$.  For ease of exposition,  we set all the $n_i$'s to be equal.}.  For each $1\leq j\leq 2m-1$,  define $\calP_{\Omega_j}(\calT^{\ast}):=n_0^{-1}\sum_{k=(j-1)n_0+1}^{jn_0} \calT^{\ast}(\omega_i)\cdot \calE_{\omega_i}$.  Due to i.i.d. sampling with replacement,  we have $\EE \calP_{\Omega_j}(\calT^{\ast})=\calT^{\ast}/d^{\ast}$.  
Basically, for each $i=1,\cdots,m-1$, we use the sub-samples $\Omega_{2i-1}$ and $\Omega_{2i}$, together with the estimates $\hat T_1, \cdots, \hat T_{i-1}$,  to estimate $T_i^{\ast}$. Finally, the last sub-sample $\Omega_{2m-1}$ is used to estimate $T_m^{\ast}$.  Now, for each $i=1,2,\cdots,m-1$, define a $(d_1\cdots d_i)\times (d_1\cdots d_i)$ matrix
$$
N_i
= \frac{(d^*)^2}{2n_0^2}\left(\calP_{\Omega_{2i-1}}(\calT^*)\lr{i}\big(\calP_{\Omega_{2i}}(\calT^*)\lr{i}\big)^{\top} + \calP_{\Omega_{2i}}(\calT^*)\lr{i}\big(\calP_{\Omega_{2i-1}}(\calT^*)\lr{i}\big)^{\top}\right)
$$
Due to the independence between $\Omega_{2i-1}$ and $\Omega_{2i}$, one can easily verify $\EE N_i=\calT^{\ast \langle i\rangle} \calT^{\ast \langle i\rangle\top}$ so  that $N_i$ is an unbiased estimator.  Note that the dimension of $N_i$ becomes larger when $i$ increases.  It turns out that a direct spectral initialization by $N_i$ still performs poorly unless the sample size is greater than $d_1\cdots d_i$.  Instead,  we multiply the left and right hand side of $N_i$ by the estimated left part $\hat T^{\leq i-1}\otimes I_{d_i}$ and its transpose,  respectively.  The resultant symmetric matrix is then used for estimating the $i$-th component $T_i^{\ast}$.  See the details in Algorithm \ref{alg:init}.  So the initialization is proceeded in an iterative fashion.  Note that an additional truncation procedure is applied to guarantee the incoherence property of the estimates $\hat T_i$'s,  where $\hat T_i^1$ denotes the $1$-st row of matrix $\hat T_i$.  

After obtaining the estimates $\hat T_1, \ldots, \hat T_{m-1}$ and the respectively constructed left part $\hat T^{\leq m-1}$ , the last component  $\hat T_m$ is estimated by the minimizer of 
$$
\min\nolimits_{T_m}\big\|\hat T^{\leq m-1} T_m - n_0^{-1}d^*\calP_{\Omega_{2m-1}}(\calT^*)\lr{m-1}\big\|_{\rm F},
$$
whose solution is explicitly given by $\hat T_m: = n_0^{-1}d^*\hat T^{\leq m-1 \top}(\calP_{\Omega_{2m-1}}(\calT^*))
\lr{m-1}$.  Finally,  a trimming treatment is implemented on the reconstructed low TT-rank tensor $\hat\calT$ to ensure the desired spikiness condition.  We again remark that the numbers $\mu$ and $\nu$ are not necessarily the true spikiness and incoherence parameter of $\calT^{\ast}$,  and they are treated as tuning parameters of Algorithm \ref{alg:init}.

\begin{algorithm}
	\caption{Initialization by Sequential Spectral Initialization}\label{alg:init}
	\begin{algorithmic}
		\STATE{\textbf{Input}: Spikiness parameter $\nu$ and incoherence parameter $\mu$}
		%\STATE{\textbf{Step 1}:}
		\STATE{Set $\wt T_1$ be the top $r_1$ left singular vectors of $N_1$}
		\STATE{Truncation: $\overline T_1^{i} = \frac{\wt T_1^{i}}{\|\wt T_1^i\|_{\ell_2}}\cdot\min\{\|\wt T_1^i\|_{\ell_2},  (\mu r_1/d_1)^{1/2}\}$}  %\COMMENT{Truncation for incoherence}
		\STATE{Re-normalization: $\hat T_1 = \overline T_1(\overline T_1^{\top}\overline T_1)^{-1/2}$} 
		\FOR{$i = 2,\ldots,m-1$}
		\STATE{Set $L(\wt T_{i})$ to be the top $r_i$ left singular vectors of $(\hat T^{\leq i-1}\otimes I)^{\top}N_i(\hat T^{\leq i-1}\otimes I)$}
		\STATE{Truncation: $L(\overline T_i)^{j} = \frac{L(\wt T_i)^j}{\|L(\wt T_i)^j\|_{\ell_2}}\cdot\min\{\|L(\wt T_i)^j\|_{\ell_2}, \sqrt{\mu r_i/d_i}\}$}
		\STATE{Re-normalization: $L(\hat T_i) = L(\overline T_i)\big(L(\overline T_i)^{\top}L(\overline T_i)\big)^{-1/2}$}
		\ENDFOR
		\STATE{The last component: $\hat T_m = (\hat T^{\leq m-1})^{\top}\Big(\frac{d^*}{n_{0}}\calP_{\Omega_{2m-1}}(\calT^*)\Big)\lr{m-1}$}
		\STATE{Reconstruction: $\hat\calT = [\hat T_1,\ldots,\hat T_m]$}
		\STATE{\textbf{Output}: $\calT_0 = \ttsvd_{\r}(\trim_{\zeta}(\hat\calT))$ with $\zeta = \frac{10\fro{\hat\calT}}{9\sqrt{d^*}}\nu$}
	\end{algorithmic}
\end{algorithm}

\subsection{Computation of Riemannian Gradient}\label{sec:computation of pta}
In this section,  we  provide more details regarding the computation of Riemannian gradient, that is, $\calP_{\TT}(\calA)$,  where $\calA$ is a given tensor of size $d_1\times \cdot\times d_m$ and  $\TT$ is the tangent space of the TT-format tensor manifold $\mfd$ at the point $\calT$ with a left orthogonal decomposition $\calT=[T_1,\cdots, T_m]$.  By definition,  $\calP_{\TT}(\calA)$ is the projection of $\calA$ onto the tangent space $\TT$.  

Let us begin with parametrizing the tangent space $\TT$.  By Theorem~2 of \citet{holtz2012manifolds},  the parametrization of $\TT$ depends on a gauge sequence.  For simplicity and ease of exposition,  we take the gauge sequence as a sequence of identity matrices.  By \citet{holtz2012manifolds},  for any element $\calX\in\TT$,  there exist a sequence of tensors $X_1,\cdots, X_m$ with $X_i$ being a tensor of size $r_{i-1}\times d_i\times r_i$ such that $\calX$ can be explicitly written in the form
\begin{align}\label{eq:TT-formula}
\calX=\sum_{i=1}^m \delta\calX_i\quad \textrm{where the TT-format tensor   } \delta\calX_i=[T_1,\cdots, T_{i-1}, X_i, T_{i+1},\cdots, T_m].
\end{align}
The tensor $X_i$ should satisfy that $L(T_i)^{\top}L(X_i)$ is an all-zero matrix of size $r_i\times r_i$ for all $i=1,\cdots, m-1$.  There is no constraint on the last component $X_m$.  

Based on the parametrization form (\ref{eq:TT-formula}),  for an arbitrary tensor $\calA$ of size $d_1\times\cdots\times d_m$,  the Riemannian gradient $\calP_{\TT}(\calA) $ must be of the following form
$$
\calP_{\TT}(\calA) = \delta\calA_1 + \ldots + \delta\calA_m \quad \textrm{where } \delta\calA_i=[T_1,\cdots, T_{i-1}, A_i, T_{i+1},\cdots, T_m]
$$
for some tensor $A_i$ of size $r_{i-1}\times d_i\times r_i$ satisfying $L(T_i)^{\top}L(A_i)$ is an all-zero matrix for all $i=1,\cdots, m-1$.   This suggests that,  for all $i\neq j$,  the tensor $\delta\calA_i$ is orthogonal to $\delta\calA_j$,  proved in the following lemma.    
\begin{lemma}\label{lem:TT}
	Let $\calT = [T_1,\ldots,T_m]\in\mfd$ be a left orthogonal decomposition of $\calT$.  For an arbitrary $\calA$ of size $d_1\times\cdots\times d_m$,  the components $\delta\calA_{i}$'s of $\calP_{\TT}(\calA)$ satisfy 
$
\inp{\delta\calA_i}{\delta\calA_j} = 0
$
 for all $1\leq i\neq j\leq m$.  
\end{lemma}
\begin{proof}
	Without loss of generality,   assume $i<j$. Then we have
	\begin{align*}
		\inp{\delta\calA_i}{\delta\calA_j} = \inp{(\delta\calA_i)\lr{i}}{(\delta\calA_j)\lr{i}}= \inp{(T^{\leq i-1}\otimes I)L(A_i)T^{\geq i+1}}{(T^{\leq i-1}\otimes I)L(T_i)\tilde{A}_j^{\geq i+1}} = 0,
	\end{align*}
	where the last equality is due to the facts that $(T^{\leq i-1}\otimes I)$ has orthonormal columns and $L(A_i)^{\top}L(T_i)$ is an all-zero matrix.   Here,  for simplicity,  we denote $(\delta\calA_{j})\lr{i} = T^{\leq i}\tilde{A}_j^{\geq i+1}$ for some matrix $\tilde{A}_j^{\geq i+1}$.
\end{proof}

Due to this orthogonality property of Lemma~\ref{lem:TT},   for all $i\in[m-1]$,  determining $\delta\calA_i$ is equivalent to solving the following individual optimization problem
\begin{align}\label{tangent:leqm-1}
	\min_{A_i}\fro{\calA - \delta\calA_{i}}, \text{~~~s.t.~} \delta\calA_i = [T_1,\ldots,A_i,\ldots,T_m] \textrm{~~and~~} L(A_i)^{\top}L(T_i) = 0
\end{align}
For the last component,  it suffices to solve
\begin{align}\label{tangent:m}
	\min_{A_m}\fro{\calA - \delta\calA_{m}}, \text{~~~s.t.~} \delta\calA_m = [T_1,\ldots,T_{m-1},A_m]
\end{align}
Finally, based on (\ref{tangent:leqm-1}) and (\ref{tangent:m}),  we can obtain a closed-form solution of $A_i, i\in[m]$,  represented in the form of $L(A_i)$,  by
\begin{align}\label{tangent:W}
	L(A_i) = \left\{
	\begin{array}{rlc}
		&(I - L(T_i)L(T_i)^{\top})(T^{\leq i-1}\otimes I)^{\top}\calA\lr{i}(T^{\geq i+1})^{\top}(T^{\geq i+1}(T^{\geq i+1})^{\top})^{-1}, &i\in[m-1]\\
		&(T^{\leq m-1}\otimes I)^{\top}\calA\lr{m}, & \textrm{if~~} i = m.
	\end{array}
	\right.
\end{align}
This yields the way of computing the Riemannian gradient $\calP_{\TT}(\calA)$.

\section{Exact Recovery and Convergence Analysis}\label{sec:thoery}
In this section, we prove the validity of the sequential second-order spectral initialization and linear convergence behaviour of the RGrad algorithm  so that the underlying tensor can be exactly recovered with an almost optimal sample size of observed entries.  
For ease of exposition, we assume $r_1,\ldots, r_{m-1}$ are of the same order and denote $\rmax$ an upper bound for them. The sample size requirement in more general cases of $r_i$'s can be found from Theorem~\ref{thm:localconvergence:detail} and Theorem~\ref{thm:init:detail} in the Appendix. Recall the notations $\dmax=\max_j d_j$ and $d^{\ast}=d_1\cdots d_m$. The smallest singular value and condition number of $\calT^{\ast}$ are denoted by $\sigmamin$ and $\kappa_0$, respectively. Our main theorem of exact recovery is described as follows. % whose proof relies on Lemma~\ref{thm:init} and Lemma~\ref{thm:localconvergence}. 

\begin{theorem}\label{thm:main}
Suppose that $\calT^{\ast}$ is of size $d_1\times\cdots\times d_m$ with a TT-rank $\boldsymbol{r} = (r_1,\ldots,r_{m-1})$ whose spikiness is bounded by $\spiki(\calT^*) \leq \nu$. Let $\calT_0$ be initialized by the sequential second-order method as Algorithm \ref{alg:init} and $\{\calT_l\}_{l=1}^{l_{\textsf{max}}}$ be the iterates produced by Algorithm \ref{alg:rgradtt} where $l_{\textsf{max}}$ is the maximum number of iterations, and the stepsize $\alpha=0.12n^{-1}d^{\ast}$. There exist absolute constants $C_{m,1}, C_{m, 2}>0$ depending only on $m$ such that if the sample size $n$ satisfies 
$$
n \geq C_{m,1}\cdot \big(\kappa_0^{4m-4}\nu^{m+3} (d^{\ast})^{1/2}\rmax^{(5m-9)/2}\log^{m+2}\dmax + \kappa_0^{4m+8}\nu^{2m+2}\dmax\rmax^{3m-3}\log^{2m+4}\dmax\big),
$$
then with probability at least $1- (2m+4)\dmax^{-m}$, for any $\varepsilon\in(0,1)$, after $l_{\textsf{max}}=\lceil C_{m,2}\cdot\log(\sigmamin/\varepsilon) \rceil$ iterations, the final output achieves error $\|\calT_{l_{\textsf{max}}}-\calT^{\ast}\|_{\rm F}\leq \varepsilon$. 
\end{theorem}

If the tensor dimension is balanced $d_j\asymp d$ for all $j$, rank $\rmax=O(1)$ and $\calT^{\ast}$ is well conditioned in that $\kappa_0, \nu=O(1)$, the sample size requirement of Theorem~\ref{thm:main} simplifies to $O_m(d^{m/2}\cdot \textsf{\small Polylog(d)})$. This improves the existing result \citep{imaizumi2017tensor} based on matricization and matrix nuclear norm penalization, which requires sample size $O_{m}(d^{\frac{m+1}{2}}\cdot \textsf{\small Polylog(d)})$ when $m$ is odd.  Moreover, for the case $m=3$, \citet{barak2016noisy} conjectures that, based on the reduction to Boolean satisfiability problem, $O(d^{3/2})$ is a lower bound for the sample size such that polynomial-time algorithm exists for exact tensor completion. Therefore, suppose $\rmax = O(1)$, the sample size requirement of Theorem \ref{thm:main} is likely optimal, at least for $m=3$, up the logarithmic factors if only polynomial-time algorithms are sought. 

We note that, in the sample size requirement, the exponent on rank $r$ still depends on the order of tensor, due to technical difficulties. The recursive nature of computing a TT decomposition substantially complicates the theoretical analysis, involving repeated appearances of the incoherence parameters and TT rank. Interestingly, the required number of iterations $l_{\textsf{max}}$ is free of the condition number, which often appears in decomposition-based algorithms \citep{cai2021nonconvex, han2020optimal}, except the recently proposed scaled gradient descent algorithm \citep{tong2021scaling}. 

The proof of Theorem \ref{thm:main} relies on two essential parts: the local convergence of warm-initialized Algorithm \ref{alg:rgradtt} and  the validity of initialization by Algorithm~\ref{alg:init}, which are separately dealt with in the subsequent sections. 

\subsection{Local Convergence of Riemannian Gradient Descent}
We study the local convergence of Algorithm~\ref{alg:rgradtt} in a small neighbourhood of the global minimizer, namely $\calT^{\ast}$ for our case. 
Lemma \ref{thm:localconvergence} dictates that Algorithm \ref{alg:rgradtt} converges linearly to $\calT^{\ast}$ as long as the initialization, be it obtained from our proposed Algorithm~\ref{alg:init} or other approaches,  $\calT_0$ is sufficiently close to $\calT^{\ast}$ and a sample size condition holds.  The proof of Lemma~\ref{thm:localconvergence} is postponed to the Appendix. 
\begin{lemma}\label{thm:localconvergence}
Suppose the conditions on $\calT^{\ast}$ from Theorem~\ref{thm:main} hold, and the initialization $\calT_0$ satisfies 
$$
	\fro{\calT_0 - \calT^*}\leq \frac{\sigmamin}{Cm\kappa_0\rmax^{1/2}}\quad {\rm and}\quad  \incoh(\calT_0)\leq 2\kappa_0^2\nu
$$
for a sufficiently large but absolute constant $C>0$. There exists an absolute constant $C_m>0$ depending only on $m$ such that if the sample size $n$ satisfies
\begin{align*}
		n \geq C_m\cdot \Big(\kappa_0^{2m+4}\nu^{m+1}(d^*)^{1/2}\rmax^{\frac{m+3}{2}\vee (m-1/2)}\log^{m+2}\dmax
		+\kappa_0^{4m+8}\nu^{2m+2}\dmax\rmax^{(m+3)\vee (2m-1)}\log^{2m+4}\dmax\Big),
\end{align*}
then with probability at least $1-(m+4)\dmax^{-m}$, the sequence $\{\calT_l\}_{l=1}^{\infty}$ generated by Algorithm \ref{alg:rgradtt} with a constant step size $\alpha = 0.12n^{-1}d^*$ satisfy
	$$
	\fro{\calT_{l} - \calT^*}^2 \leq 0.975\cdot \fro{\calT_{l-1} - \calT^*}^2.
	%	\fro{\calT_l - \calT^*}^2 \leq 0.975^l\fro{\calT_0 - \calT^*}^2.
	$$
for all $l=1,2,\cdots$. 
\end{lemma}

Lemma~\ref{thm:localconvergence} dictates that the error contracts at a constant rate which is strictly smaller than $1$.  One interesting fact of our results is that the contraction rate is independent of the condition number,   which is the reason that $l_{\textsf{max}}$ in Theorem \ref{thm:main} is free of $\kappa_0$. It suggests that the Riemannian gradient descent algorithm converges fast even for very ill-conditioned tensors, improving the existing results \citep{jain2014provable} and \citep{cai2021nonconvex}. Recently, \citet{tong2021scaling} introduced a scaled gradient descent algorithm to remove the dependence on the condition number, where the rescaling procedure plays a role of re-conditioning. Interestingly, Riemannian gradient descent algorithm automatically achieves this performance without the need to re-scaling. This is perhaps an intrinsic advantage of the manifold-type algorithms (see also \citealt{cai2021generalized}). We note that the contraction rate $0.975$ is improvable but no further efforts are made for that purpose. 

In the case $\rmax, \kappa_0, \nu=O(1)$, the sample size required by Lemma \ref{thm:localconvergence} is $O_m(d^{m/2}\cdot \textsf{\small Polylog(d)})$, which matches that required by Theorem \ref{thm:main}. However, if $\rmax$ grows with $\dmax$, the exponent on $\rmax$ is slightly better than that of Theorem \ref{thm:main}. 
	
If $m=2$, a TT-format tensor is merely a matrix and the left orthogonal decomposition reduces a decomposition with an orthogonal matrix on left hand side.  The convergence analysis of Riemannian gradient descent algorithm for matrix completion was investigated by \citet{wei2016guarantees}, showing that the algorithm converges linearly if the initialization is so good that $\|\calT_0-\calT^{\ast}\|_{\rm F}=o\big((n/d^{\ast})^{1/2}\big)\cdot \sigmamin$. This is very restrictive since, oftentimes, the desired sample size $n$ is only of order $d^{\ast 1/2}$. It is speculated that this gap is due to technical reasons, deriving from the special form of Riemannian gradient where the existing strategy \citep{candes2009exact} simply fails, but the issue has never been really resolved. Taking a more sophisticated approach, by integrating prior tools from \citet{xia2017polynomial, cai2021generalized} and \citet{tong2021scaling}, we finally provide an affirmative answer that the initialization condition can indeed be relaxed to the typical ones required by other rivalry algorithms. Lemma \ref{thm:localconvergence} suggests that, if $m, \rmax, \kappa_0$ are bounded, we only require the initialization satisfies $\|\calT_0-\calT^{\ast}\|\leq c\cdot \sigmamin$ for a small enough but absolute constant $c>0$. This holds for higher-order TT-format tensors and is not restricted to matrices.

\subsection{Initialization by Sequential Second-Order Spectral Method}
As stated in Lemma~\ref{thm:localconvergence}, the warm initialization is of crucial importance to ensure the convergence of Algorithm \ref{alg:rgradtt}. 
Now we show that our proposed sequential second-order spectral initialization, stated in Algorithm \ref{alg:init}, can indeed, with high probability, deliver an estimate close to $\calT^*$ under a nearly optimal sample size condition. While our algorithm is inspired by \citet{xia2017polynomial}, the theoretical investigation turns out to be more challenging due to the recursive nature of Algorithm \ref{alg:init}. Compared with \citet{xia2017polynomial}, a major challenge in our proof is to establish the concentration of $(\hat T^{\leq i-1}\otimes I)^{\top}N_i (\hat T^{\leq i-1}\otimes I)$ rather than the concentration of $N_i$ itself. Actually, the concentration of $N_i$ is poor because its dimension can be quite large. By multiplying both sides with the incoherent matrix $\hat T^{\leq i-1}$, the resultant smaller matrix enjoys much better concentration.  The proof of Lemma~\ref{thm:init} is relegated to the Appendix. 

\begin{lemma}\label{thm:init}
Suppose the conditions of $\calT^{\ast}$ from Theorem \ref{thm:main} hold. For any absolute constant $C>0$, there exists an absolute constant $C_m>0$ depending only on $m$ such that if 
$$
n\geq C_m\nu^{m+3}\kappa_0^{4m-4}((d^\ast)^{1/2}\rmax^{(5m-9)/2} + \dmax\rmax^{3m-4})\log^2 \dmax,
$$
then with probability at least $1 - m\dmax^{-m}$, the output of Algorithm \ref{alg:init} satisfies 
$$
\fro{\calT_0-\calT^*} \leq \frac{\sigmamin}{Cm\kappa_0^2\rmax^{1/2}}\quad \text{~and~}\quad \incoh(\calT_0)\leq 2\kappa_0^2\nu.
$$
\end{lemma}

Therefore,  the output of Algorithm \ref{alg:init} indeed satisfies the initialization condition required by Lemma \ref{thm:localconvergence}. If the tensor dimension is balanced $d_j\asymp d$ and $\nu, \kappa_0, \rmax =O(1)$, the sample size requirement for warm initialization matches, with a slightly better dependence on $\log \dmax$, that required by the algorithmic convergence of Lemma \ref{thm:localconvergence}. Thus our initialization method is valid requiring a nearly optimal sample size. 

Theorem \ref{thm:main} can be readily proved by combining Lemma \ref{thm:localconvergence} and Lemma \ref{thm:init}. 
\begin{proof}
Under the sample size condition of Theorem \ref{thm:main}, both Lemma \ref{thm:localconvergence} and Lemma \ref{thm:init} hold. Therefore, from the contraction property in Lemma \ref{thm:localconvergence}, we get 
$$
\|\calT_{l_{\textsf{max}}}-\calT^{\ast}\|_{\rm F}\leq (0.975)^{l_{\textsf{max}}}\cdot \|\calT_0-\calT^{\ast}\|_{\rm F}\leq (0.975)^{l_{\textsf{max}}}\cdot \sigmamin,
$$
where the last inequality is guaranteed by Lemma \ref{thm:init} and the above bound holds with probability at least $1-(2m
+4)\dmax^{-m}$. Then, for any $\varepsilon>0$, there exists a constant $C_2>0$ such that if $l_{\textsf{max}}=\lceil C_2\log(\sigmamin/\varepsilon)\rceil$, the final output achieves the error $\|\calT_{l_{\textsf{max}}}-\calT^{\ast}\|\leq \varepsilon$, which concludes the proof. 
\end{proof}

\section{Statistically Optimal Noisy TT-format Tensor Completion}\label{sec:noisy}
Oftentimes,  the observed data are contaminated by noise.  
Suppose we observe i.i.d.  random pairs $\{ (\omega_i, \calY_i)\}_{i=1}^n$, where $\calY_i=\calT^{\ast}(\omega_i)+\xi_i$  with the noise $\{\xi_i\}_{i=1}^n$ being  i.i.d.  $\sigma_s$-sub-Gaussian.  Clearly,  the case $\sigma_s=0$  reduces to the problem of  exact tensor completion studied above.  

The loss function is re-written by 
\begin{align}\label{prob:tensor completion:noise}
	&\min_{\calT\in\RR^{d_1\times \cdots \times d_m}}g_{\Omega}(\calT) := \frac{1}{2}\sum_{i=1}^n\big(\calT(\omega_i)-\calY_i\big)^2
	\hspace{0.5cm}\text{such that } \ranktt(\calT)\leq \boldsymbol{r}. 
\end{align}
The initialization procedure is essentially identical to the noiseless case but shall be re-written using new notations.  Here the sample splitting is applied to $[n]$ as in Section~\ref{sec:initialization} and recall $n=(2m-1)n_0$.  With a slight abuse of notation,  we re-define the matrices $\{N_i\}_{i = 1}^{m-1}$ by
$$
N_i
= \frac{(d^*)^2}{2n_0^2}\sum_{j=(2i-2)n_0+1}^{(2i-1)n_0}\sum_{k=(2i-1)n_0+1}^{2in_0}\calY_j\calY_k(\calE_{\omega_j}\lr{i}(\calE_{\omega_k}\lr{i})^{\top} + \calE_{\omega_k}\lr{i}(\calE_{\omega_j}\lr{i})^{\top})
$$
Similarly,  for the component $\hat T_m$,  it is attained by
$$
\hat T_m = n_0^{-1}d^*\hat T^{\leq m-1 \top}\sum_{i=(2m-2)n_0+1}^{n}\calY_i\calE_{\omega_i}\lr{m-1}
$$
The initialization $\calT_0$ is then obtained by applying Algorithm \ref{alg:init} with the newly defined $\{N_i\}_{i=1}^{m-1}$ and $\hat T_m$.  Lemma~\ref{lemma:init:noise} affirms that,  under a suitable signal-to-noise ratio condition,  Algorithm \ref{alg:init} outputs an initialization that is sufficiently close to $\calT^*$.  For ease of exposition, we assume $d_1\asymp\cdots\asymp d_{m}\asymp \bar d$ and $r_1\asymp \cdots \asymp r_{m-1}\asymp \bar r$.  We also provide a version of Lemma~\ref{lemma:init:noise} with general $d_j$'s and $r_j$'s in Section \ref{sec:pf:init:noise}.  Recall that $d_j\leq \bar d,  r_j\leq \bar r,  r^{\ast}=r_1\cdots r_m$ and $d^{\ast}=d_1\cdots d_m$.  

\begin{lemma}\label{lemma:init:noise}
Suppose $\spiki(\calT^{\ast})\leq \nu,  \kappa(\calT^{\ast})\leq \kappa_0$ and $\{\xi_i\}_{i=1}^n$ are i.i.d. $\sigma_s$ sub-Gaussian with variance $\text{Var~}\xi_1^2\leq C_1\sigma_s^2$ for some absolute constant $C_1>0$.  Suppose the sample size requirement in Lemma \ref{thm:init} holds,  also assume the signal-to-noise ratio (SNR) condition: 
\begin{align*}
	&\sigmamin/\sigma_s\geq
	C_2\cdot\frac{(d^*)^{3/4}(r^*\rmax)^{1/4}\log^{3/2}(\dmax)}{n^{1/2}}
%	 C_m\max\{\kappa_0^{2m-3}(r^*\rmax\rmin)^{1/2}(\frac{d_1d^*\log(d)}{n})^{1/2},\kappa_0^{m-2}(r^*\rmax)^{1/4}\frac{(d^*)^{3/4}\log^{3/2}(d)}{n^{1/2}}(1+\frac{d_1^{1/2}}{(d^*)^{1/4}})\}\\
%	&\hspace{0.5cm}+\sum_{i=2}^{m-1}\Bigg[C_m\kappa_0^{2(m-i)+5}\nu^2(r_i\cdots r_{m-1}\rmax\rmin)^{1/2}\Big(\frac{(d^*)^{1/2}d_i}{n}r_{i-1}\log(\dmax) + \frac{(d^*)^{1/2}d_i\cdots d_m}{n^2}r_{i-1}\log^2(\dmax) \\
%	&\hspace{7cm}+ \frac{(d^*d_i)^{1/2}}{\sqrt{n}}\sqrt{r_{i-1}}\log^{1/2}(\dmax)\Big)\\
%	&\hspace{0.5cm} + C_m\kappa_0^{m-i+2}\nu(r_i\cdots r_{m-1}\rmax\rmin)^{1/4}\Big(\frac{(d^*d_i)^{1/2}}{n^{1/2}}r_{i-1}^{1/2}\log^{5/4}(\dmax) + \frac{d_i^{1/2}(d_{i+1}\cdots d_m)^{1/4}}{n^{1/2}}r_{i-1}^{3/4}\log^{5/4}(\dmax)\Big)\Bigg]\\
%	&\hspace{0.5cm}+C_m\kappa_0\sqrt{\rmax}\Big(\sqrt{\frac{d^*d_mr_{m-1}}{n}\log^{1/2}(\dmax)} + \frac{d^*}{n}\sqrt{\frac{r_{m-1}}{d_1\cdots d_{m-1}}}\kappa_0^2\nu\log(\dmax)\Big).
\end{align*}
for some $C_2>0$ depending only on $m,\kappa_0,\nu$. 
Then with probability at least $1 - 10m\dmax^{-m}$, the output of Algorithm \ref{alg:init} satisfies 
$$
\fro{\calT_0-\calT^*} \leq \frac{\sigmamin}{Cm\kappa_0^2\rmax^{1/2}}\quad \text{~and~}\quad \incoh(\calT_0)\leq 2\kappa_0^2\nu,
$$
where $C>0$ is the same as in Lemma~\ref{thm:init}.  
\end{lemma}

In particular, to fix ideas,  consider the case when $d_1 = \cdots = d_m = d$ and $\kappa_0,\nu, \{r_i\}_{i=1}^{m-1}$ are all $O(1)$,  then there exist $C_1,C_2>0$ independent of $d$,  under the sample size condition $n\geq C_1d^{m/2}\cdot\textsf{\small Polylog(d)}$,  the SNR requirement of Lemma~\ref{lemma:init:noise}  can be simplified as 
$$
\sigmamin/\sigma_s\geq C_2\frac{d^{3m/4}}{n^{1/2}}\cdot\textsf{\small Polylog(d)}.
$$
This SNR requirement matches the previous work in \cite{xia2017statistically} where the noisy Tucker-format tensor completion was investigated.  While it remains mysterious in the minimal SNR requirement for noisy tensor completion,  convincing evidences have appeared showing that the above SNR requirement is likely (near) minimal if only polynomial time algorithms are sought,   at least for the special case $n\asymp d^{\ast}$.  See,  e.g.,  \cite{zhang2018tensor,lyu2022optimal,kunisky2019notes}. 

Once a faithful initialization is available,   the convergence of RGrad Algorithm~\ref{alg:rgradtt} and its statistical performance are guaranteed by the following lemma.  

\begin{lemma}\label{lemma:localconvergence:noise}
Suppose $\spiki(\calT^{\ast})\leq \nu,  \kappa(\calT^{\ast})\leq \kappa_0$ and $\{\xi_i\}_{i=1}^n$ are i.i.d. $\sigma_s$ sub-Gaussian with variance $\text{Var~}\xi_1^2\leq C_1\sigma_s^2$ for some absolute constant $C_1>0$.  
Assume the same condition on $\calT_0$ and sample size $n$ as in Lemma \ref{thm:localconvergence}.  Further suppose the signal-to-noise ratio (SNR) satisfies 
%	$$\sigmamin/\sigma_s \geq C_m \kappa_0^4\nu^2\sqrt{\frac{d^*\rmax\cdot\dof}{n}}$$
$$\sigmamin/\sigma_s \geq C_m\max\Bigg\{\left(\sqrt{\frac{d^*\dmax}{n}} + \frac{d^*}{n}\right)r^*\log^{m+2}(\dmax), \kappa_0^4\nu^2\sqrt{\frac{d^*\rmax\cdot\dof}{n}}\Bigg\}
$$
for some  constant $C_m>0$ depending only on $m$ and $\dof = \sum_{i=1}^mr_{i-1}d_ir_i$ is the degree of freedom of a TT-rank $\boldsymbol{r}$ tensor. 
Then  with probability exceeding $1-(2m+5)\dmax^{-m}$, after $O_m\big(\log(n\sigmamin^2\dof^{-1}\sigma_s^{-2})\big)$ iterations,  Algorithm \ref{alg:rgradtt} outputs an estimator $\hat\calT$ such that 
$$
\frac{1}{d^*}\fro{\hat\calT - \calT^*}^2 \leq C\frac{\rmax\cdot\dof}{n}\sigma_s^2\log(\dmax),
$$
 where the constant $C>0$ depends only on $m,\kappa_0,\nu$.
\end{lemma}

Consider the case $d_j\asymp d$,  $r_j\asymp r$,  and $\kappa_0, \nu=O(1)$,  by Lemma~\ref{lemma:init:noise} and Lemma~\ref{lemma:localconvergence:noise},  if the sample size and SNR satisfy 
$$
n\gg d^{m/2}r^{5m/2}\log^{m+2}d\quad {\rm and}\quad \frac{\sigmamin}{\sigma_s}\gg \frac{d^{3m/4}r^{m/2}\log^{m+2}d}{n^{1/2}}
$$
,  then with sequential second-order moment initialization and after $O_m\big(\log(n\sigmamin^2d^{-1}\sigma_s^{-2})\big)$ Riemannian gradient descent iterations,  we end up with 
$$
\frac{1}{d^{\ast}}\|\hat\calT-\calT^{\ast}\|_{\rm F}^2=O_m\left(\frac{r^3d\log d}{n}\sigma_s^2\right),
$$ 
which holds with high probability.  Note that the model complexity is $mr^2d$.  Therefore,  when $m=O(1)$,  $\hat \calT$ is statistically optimal up to an additional factor of $r$ and the logarithmic factor.  To our best knowledge,  this is the first and sharp statistical bound for noisy TT-format tensor completion.

\section{Numerical Experiments}\label{sec:numerical}
We perform several numerical experiments for both synthetic data and real data. We also compare the RGrad-TT algorithm with RGrad-Tucker format to show the computation efficiency of the tensor train format. Throughout this section, we shall use relative error frequently that is defined by
$$\textsf{relative error } = \frac{\fro{\hat{\calT} - \calT^*}}{\fro{\calT^*}}$$
where $\hat{\calT}$ is the output of the algorithm and $\calT^*$ is the original tensor. And all for the numerical experiments, the stopping criteria is chosen when $\fro{\calT_{l+1} - \calT_l}/\fro{\calT_l} \leq 0.001$.

\subsection{Synthetic Data}
In this section,  we present some synthetic experiments.
\subsubsection{Phase Transition}
The number of measurements required for an algorithm to reliably rebuild a low TT rank tensor is an essential question in tensor completion. We explore the recovery abilities of the proposed algorithm in the framework of phase transition, which compares the number of measurements, $n$, the size of a cubic $d\times d\times d$ tensor of TT rank $(r,r)$. 

In the first test, we fix the rank $(r,r) = (2,2)$ and change the size of the tensor $d$ and the number of measurements $n$. For each $(d,n)$ tuple, we test 10 random instances. The true low TT-rank $(r,r)$ tensor $\calT^*$ is generated from truncating a random Gaussian tensor using TT-SVD. The measurements tensor $\calP_{\Omega}(\calT^*)$ is obtained by sampling $n$ entries of $\calT^*$ uniformly at random. A test is considered to be successful if the returned tensor $\calT$ satisfies $\fro{\calT-\calT^*}/\fro{\calT^*} \leq 0.01$. The dimension of the tensor are ranging from 110 to 190, and the measurements are from 5000 to 55000. The probabilities of successful recovery for the RGrad-TT is displayed in Figure \ref{fig:phase1}. In this figure, white color means that the algorithm can recover $\calT^{\ast}$ in all 10 repeated simulations, whereas the black color means that the algorithm fails in all 10 simulations.  A clear phase transition can be observed from the figure.

In the second test, we fix the dimension $d = 100$. For each $(n,r)$ tuple, we conduct 10 random instances and a test is considered to be successful if the returned tensor $\hat\calT$ satisfies $\fro{\hat\calT-\calT^*}/\fro{\calT^*} \leq 0.01$.
We plot the curves of successful recovery rate against $nd^{-3/2}$ of tensors with TT-ranks from $(4,4)$ to $(8,8)$ in Figure \ref{fig:phase2}.

\begin{figure}
	\centering
	\begin{subfigure}[b]{0.48\textwidth}
		\centering
		\includegraphics[width=\textwidth]{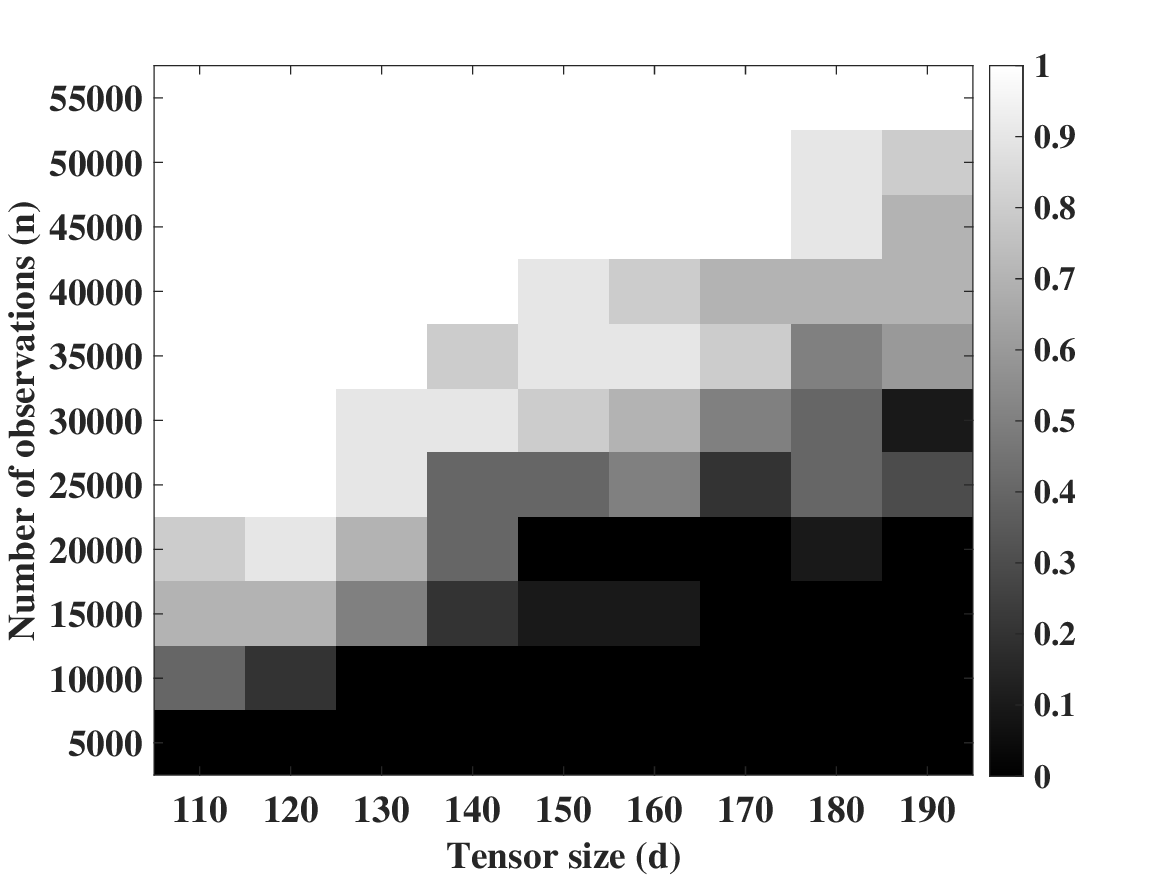}
		\caption{Fix $r = 2$}
		\label{fig:phase1}
	\end{subfigure}
	\hfill
	\begin{subfigure}[b]{0.48\textwidth}
		\centering
		\includegraphics[width=\textwidth]{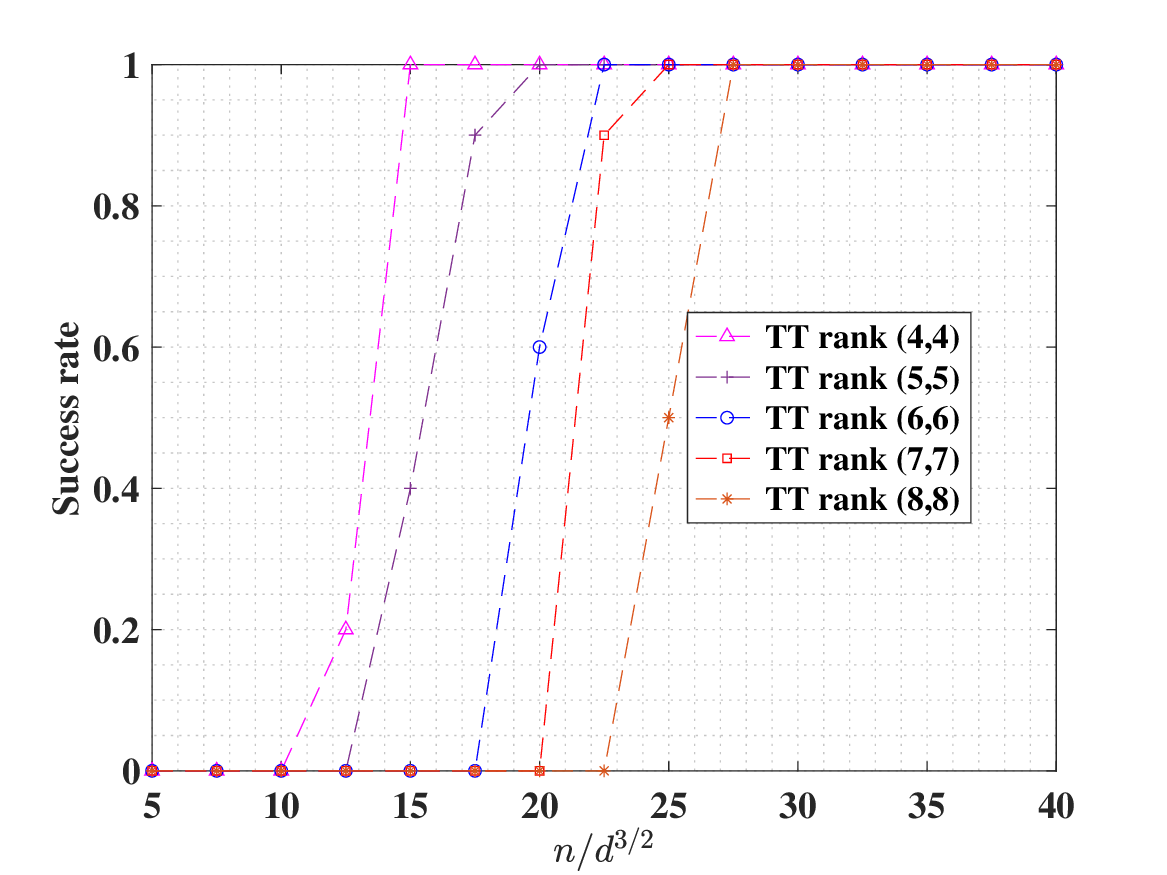}
		\caption{Fix $d = 100$}
		\label{fig:phase2}
	\end{subfigure}
	\caption{Left: Empirical phase transition using RGrad-TT. Successful recovery rate of tensors with a fixed rank $(r,r) = (2,2)$ from noiseless data is shown. White denotes successful recovery in all ten random experiments,  and black denotes failure in all experiments.  Right: Successful recovery rate of tensors with fix dimension $d = 100$ of different TT rank.}
	\label{fig:phase}
\end{figure}

\subsubsection{Efficiency of RGrad-TT}
We also conduct experiments to illustrate the efficiency of RGrad-TT against RGrad-Tucker.  We fix the dimension $d = 300$ and consider cubic tensor $\calT^*$ of size  $d\times d\times d$. We modify the parameter $r$ and generate a tensor with CP rank $r$. Then the TT rank of this tensor is bounded by $(r,r)$ and the Tucker rank of the tensor is bounded by $(r,r,r)$. And we use Algorithm \ref{alg:init} for initialization for TT format and use the second order moment method proposed by \citet{xia2017polynomial} as initialization for Tucker format.
To eliminate the impact of the randomness, for each fixed $r$, we conduct experiments on 10 instances and count the total runtime and runtime for per iteration. The stopping criteria is satisfied when $\fro{\calT_{l+1} - \calT_l}/\fro{\calT_l} \leq 0.001$.
The results are shown in Figure \ref{fig:runtime}.

\begin{figure}
	\centering
	\begin{subfigure}[b]{0.48\textwidth}
		\centering
		\includegraphics[width=\textwidth]{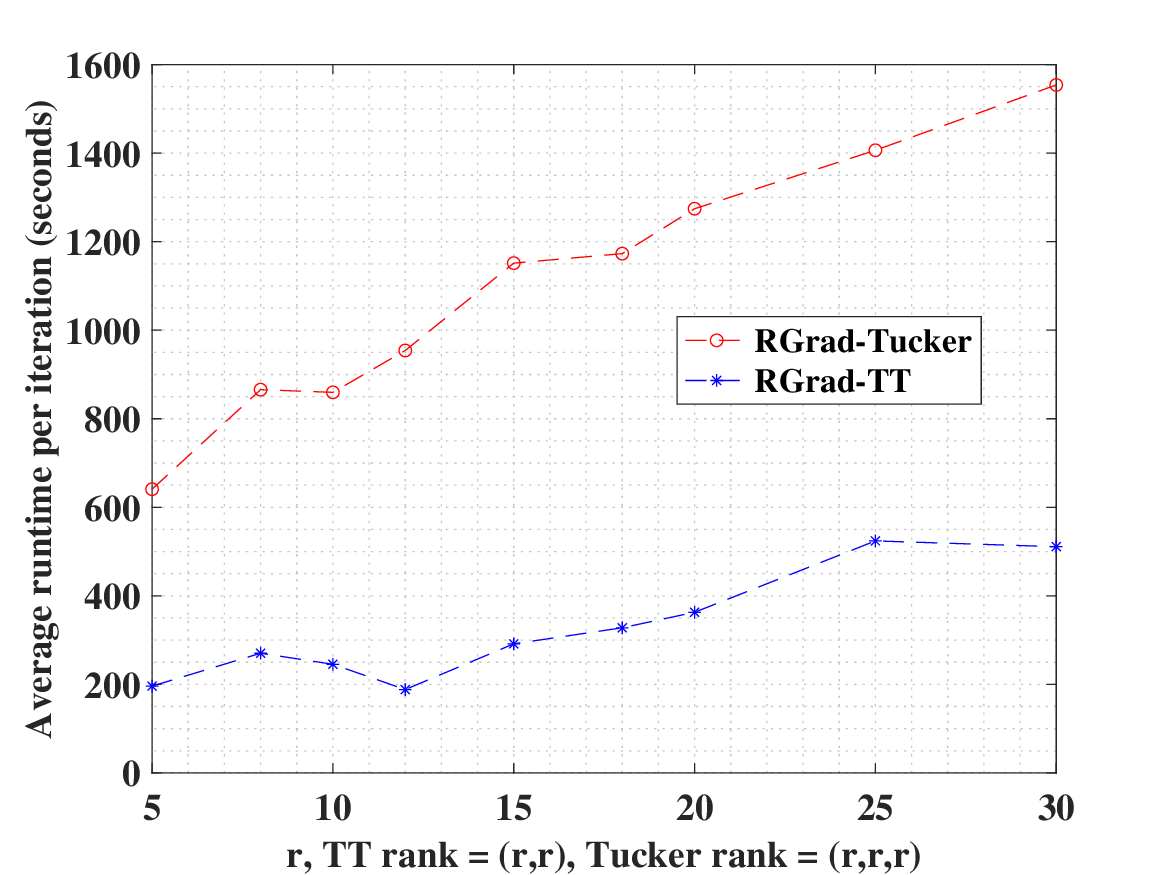}
		\caption{Total runtime of 10 instances}
		\label{fig:runtime1}
	\end{subfigure}
	\hfill
	\begin{subfigure}[b]{0.48\textwidth}
		\centering
		\includegraphics[width=\textwidth]{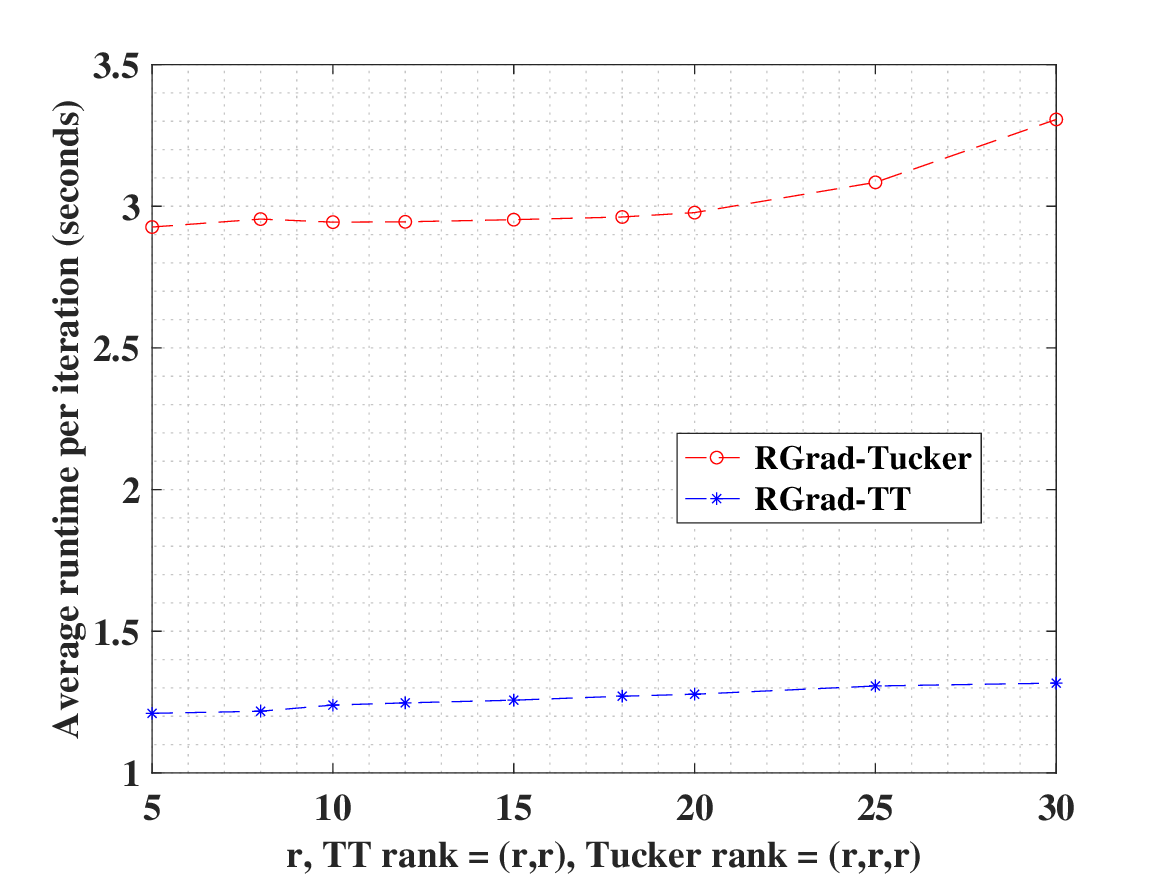}
		\caption{Average runtime per iteration}
		\label{fig:runtime2}
	\end{subfigure}
	\caption{Left: Total runtime of 10 instances; Right: average runtime for per iteration. In both cases we fix the tensor of size 300$\times$300$\times$300.}
	\label{fig:runtime}
\end{figure}

From the figures, we can see that when the TT rank is $(r,r)$ and Tucker rank is $(r,r,r)$, RGrad-TT is much faster in terms of both total runtime and per iteration runtime. 

\subsubsection{Low TT-rank Completion with Noise}
In this section, we present the performance of proposed algorithm when there exists noise. We randomly generate the tensor $\calT^*\in\RR^{100\times 100\times 100}$ with TT-rank $(r,r) = (2,2)$ and $\sigmamin(\calT^*)\approx 5$. We consider varying sampling numbers from 25000 to 75000. For each fixed sampling number, i.i.d. Gaussian noise is generated with variance from 0.001 to 0.009. For each setting, we repeat 10 i.i.d. times and the results are displayed in Figure \ref{fig:noise}. 
\begin{figure}
	\centering
	\includegraphics[width = 0.75\textwidth]{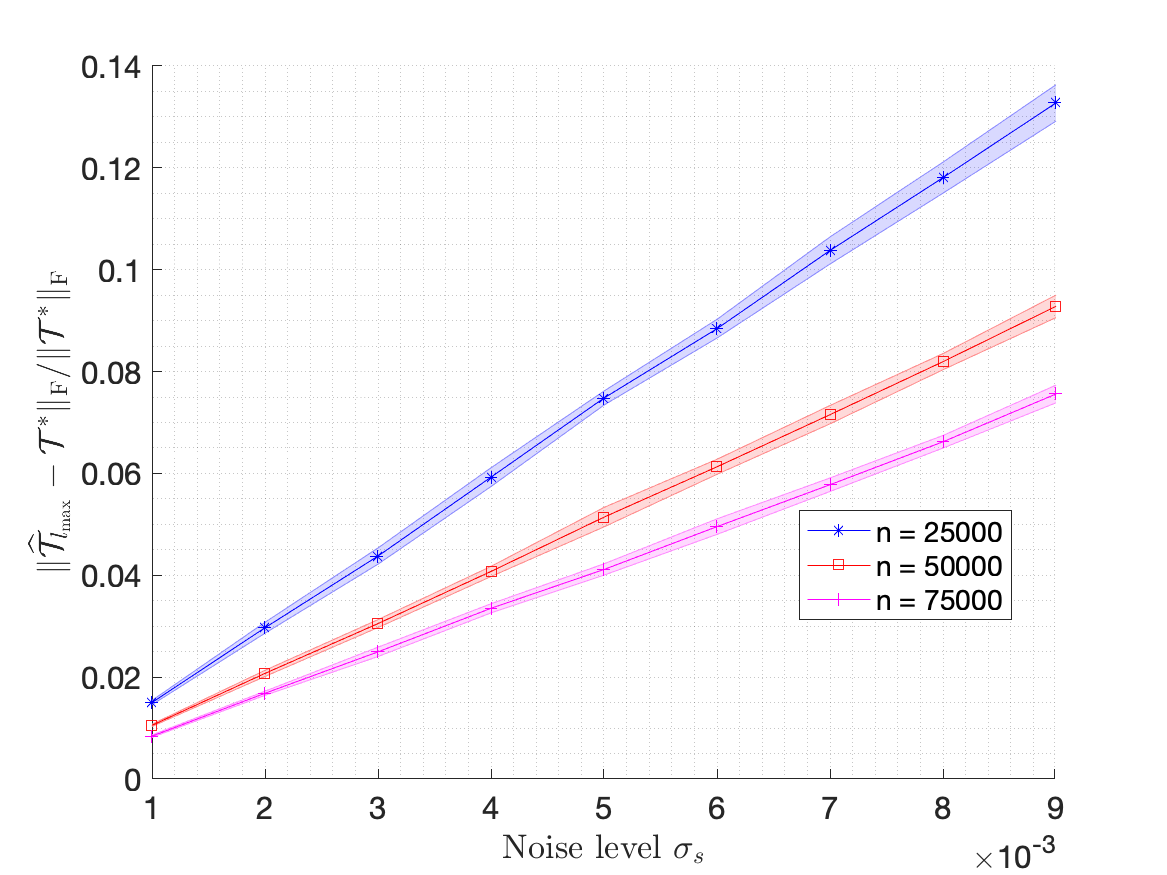}
	\caption{Low TT-rank completion with Gaussian noise. $\calT^*\in\RR^{100\times 100\times 100}$ with TT-rank $(r,r) = (2,2)$ and $\sigmamin(\calT^*)\approx 5$. Each setting is repeated 10 times.}
	\label{fig:noise}
\end{figure}

From Figure \ref{fig:noise}, we can see the estimation error depends linearly in the noise level $\sigma_s$ and is inversely proportional to the $\sqrt{n}$, which consists with our theoretical analysis.

\subsection{Real Data: Video Completion}
We consider a video of a tomato of size (242,320,167), where the third dimension is the number of frames. Since the video is an RGB one, we concatenate one channel after another along the third direction and the size of the true tensor $\calT^*$ is (242,320,501). 

To demonstrate the represent-ability of TT format against Tucker format, we apply TTSVD and HOSVD to the original video with TT rank $(r,r)$ and Tucker rank $(r,r,r)$. The approximation error is measured in the relative error and is plotted in the red and pink curves. From these curves, we can see that TTSVD has a better performance in approximating this video. 

Then we conduct video completion for this data in both TT format and Tucker format. Suppose $90\%$ of the pixels are missing and we would like to recover the original video. We use RGrad-TT (Algorithm \ref{alg:rgradtt}) with Algorithm \ref{alg:init} as initialization. To make the initialization comparable, we use second order method introduced in \citet{xia2017polynomial} for Tucker format. We change the rank parameter $r$, and the corresponding TT rank is $(r,r)$ and the Tucker rank is $(r,r,r)$. The recovered accuracy is measured in terms of relative error as shown in Figure \ref{fig:realdata}\footnote{The results using RGrad-Tucker is better than using HOSVD for small $r$ since HOSVD is only a quasi-optimal approximation. We show the results for HOSVD only for reference.}. From this result, we can see when we fix TT rank to be $(r,r)$ and Tucker rank to be $(r,r,r)$, the accuracy using RGrad-TT is better than RGrad-Tucker. Also, we can see that the recovered accuracy is almost the same as the approximation using TTSVD, which is a quasi-optimal approximation as shown in \citet{oseledets2011tensor}.

\begin{figure}
	\centering
	\includegraphics[width = 0.75\textwidth]{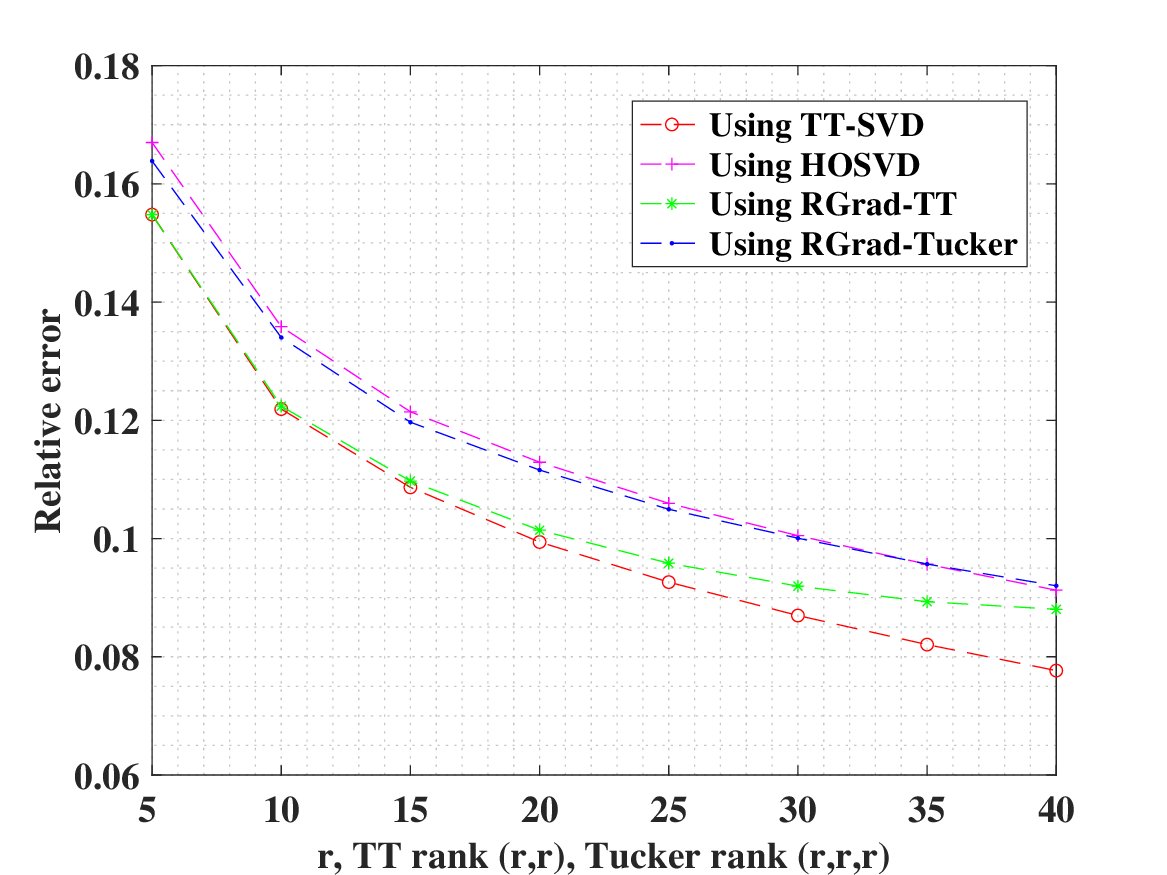}
	\caption{Results using TTSVD, HOSVD to approximate the original video with the given rank and using RGrad-TT, RGrad-Tucker to solve the tensor completion problem in the corresponding format with given rank and 10 percent of observed entries.}
	\label{fig:realdata}
\end{figure}

\acks{JFC was supported by Hong Kong RGC Grant GRF 16310620 and GRF 16309219. DX was supported by Hong Kong RGC Grant ECS 26302019 and GRF 16303320.}

\newpage

\bibliography{sample}

\newpage
\appendix
\section{Proofs of Main Lemmas}
In this section, we will present the proofs for Lemma \ref{thm:localconvergence} and Lemma \ref{thm:init}. Before we start the proof, we shall introduce some notations and definitions that will be used throughout in the proof.

Following \citet{yuan2016tensor}, we define the spectral norm of $\calT\in\ambspace$ as 
$$
\op{\calT} :=\sup_{u_i\in\SS^{d_i-1}} \inp{\calT}{u_1\otimes \cdots\otimes u_{m}},
$$
where $\SS^{d-1} = \{x\in\RR^d: \|x\|_{\ell_2} = 1\}$. The nuclear norm is the dual of spectral norm:
$$
\nuc{\calT} := \max_{\op{\calY}\leq 1}\inp{\calT}{\calY}.
$$
The relation between nuclear norm and Frobenius norm for tensors of low TT rank is summarized in Lemma \ref{lemma:nuclear norm and frobenius norm}.  

We shall use the semicolon $;$ to separate the row and column indices of a matrix, for example, we shall write $\calT\lr{i}(x_1,\ldots,x_i;x_{i+1},\ldots,x_m)$. The operator $\reshape(M,[p_1,\ldots,p_m])$ reshapes the data $M$ of size $p_1\cdots p_m$ to a tensor of size $p_1\times \cdots\times p_m$ in the Matlab way.

We also define some norms for matrices. For a matrix $A\in\RR^{p_1\times p_2}$, the $\|\cdot\|_{2,\infty}$ of $A$ is defined as $\|A\|_{2,\infty} = \max_{i = 1}^{p_1}\|A(i,:)\|_{\ell_2}$ and the $\|\cdot\|_{\ell_{\infty}}$ norm of $A$ is defined to be $\|A\|_{\ell_{\infty}} = \max_{i,j}|A(i,j)|$.

We also define the projection distance and the chordal distance between two orthogonal matrices $U,V\in\RR^{p\times r}$ as 
\begin{align*}
	d_p(U,V) = \fro{UU^T-VV^T},\quad
	d_c(U,V) = \min_{Q\in\OO_r}\fro{UQ - V}.
\end{align*}
Then we have $\frac{1}{\sqrt{2}}d_c(U,V) \leq d_p(U,V) \leq d_c(U,V)$.
\subsection{Proof of Lemma \ref{thm:localconvergence}}
We first restate a more detailed version of the lemma.
\begin{lemma}[Restate of Lemma \ref{thm:localconvergence}]\label{thm:localconvergence:detail}
	Suppose the conditions on $\calT^{\ast}$ from Theorem~\ref{thm:main} hold, and the initialization $\calT_0$ satisfies 
	$$
	\fro{\calT_0 - \calT^*}\leq \frac{\sigmamin}{Cm\kappa_0\rmax^{1/2}}\quad {\rm and}\quad  \incoh(\calT_0)\leq 2\kappa_0^2\nu
	$$
	for a sufficiently large but absolute constant $C>0$. There exists an absolute constant $C_m>0$ depending only on $m$ such that if the sample size $n$ satisfies
	\begin{align*}
	n \geq C_m\bigg(\kappa_0^{2m+4}\nu^{m+1}\log^{m+2}(\dmax)\cdot(d^*)^{1/2}((r^*)^{1/2}\rmax^2\vee r^*\rmax^{1/2})\\
	+\kappa_0^{4m+8}\nu^{2m+2}\log^{2m+4}(\dmax)\cdot\dmax(r^*\rmax^4\vee (r^*)^2\rmax)\bigg),
	\end{align*}
	then with probability at least $1-(m+4)\dmax^{-m}$, the sequence $\{\calT_l\}_{l=1}^{\infty}$ generated by Algorithm \ref{alg:rgradtt} with a constant step size $\alpha = 0.12n^{-1}d^*$ satisfy
	$$
	\fro{\calT_{l} - \calT^*}^2 \leq 0.975\cdot \fro{\calT_{l-1} - \calT^*}^2.
	$$
	for all $l=1,2,\cdots$. 
\end{lemma}
Now we present the proof of this lemma. First we introduce several events that will be useful in the proof. And the randomness of these events are from the sampling set $\Omega$.
\begin{align*}
	\bcalE_1 &= \{\op{\frac{d^*}{n}\calP_{\TT^*}\calP_{\Omega}\calP_{\TT^*} - \calP_{\TT^*}}\leq \frac{1}{2}\},\\
	\bcalE_2 &= \{\max_{x\in[d^*]}\sum_{i=1}^nI(\omega_i = x) \leq 2m\log(\dmax)\},\\
	\bcalE_3 &= \{\op{(\calP_{\Omega} - \frac{n}{d^*}\calI)(\calJ)} \leq C_m\left(\sqrt{\frac{n\dmax}{d^*}} + 1\right)\log^{m+2}(\dmax)\},
\end{align*}
where $\calJ\in\ambspace$ is the tensor with all its entries one and $\calI$ is the identity operator from $\ambspace$ to $\ambspace$.
From Lemma \ref{lemma:concentration}, $\bcalE_1$ holds with probability exceeding $1-\dmax^{-m}$. From Lemma \ref{lemma:repetition} $\bcalE_2$ holds with probability exceeding $1-\dmax^{-m}$. From Lemma~\ref{lemma:operatornorm}, $\bcalE_3$ holds with probability exceeding $1-\dmax^{-m}$.

Also we consider the following empirical process:
\begin{align}\label{emp}
	\beta_n(\gamma_1,\gamma_2) := \sup_{\calA\in\KK_{\gamma_1,\gamma_2}}\left|\inp{\calP_{\Omega}\calA}{\calA} - \frac{n}{d^*}\fro{\calA}^2\right|,
\end{align}
where 
$$\KK_{\gamma_1,\gamma_2} = \{\calA\in\RR^{d_1\times\ldots\times d_m}: \fro{\calA}\leq 1, \|\calA\|_{\ell_{\infty}}\leq \gamma_1, \nuc{\calA}\leq \gamma_2\}.$$
The following lemma states gives the upper bound for $\beta_n(\gamma_1,\gamma_2)$ with high probability.
\begin{lemma}\label{lemma:beta_n}
	Given $0<\delta_1^-<\delta_1^+, 0<\delta_2^-<\delta_2^+$ and $t\geq 1$, let
	$$t = s + \log 2 + \log(\log_2(\frac{\delta_1^+}{\delta_1^-})) + \log(\log_2(\frac{\delta_2^+}{\delta_2^-})).$$
	Then there exists a universal constant $C_m>0$ such that with probability at least $1-e^{-s}$, the following bound holds for all $\gamma_1\in[\delta_1^-, \delta_1^+]$ and all $\gamma_2\in[\delta_2^-, \delta_2^+]$,
	$$\beta_n(\gamma_1,\gamma_2) \leq C_m\gamma_1\gamma_2\left(\sqrt{\frac{n\dmax}{d^*}} + 1\right)\log^{m+2}(\dmax) + 4\gamma_1\sqrt{\frac{nt}{d^*}} + 8\gamma_1^2t.$$
\end{lemma}

Now we consider for any $\calA\in\RR^{d_1\times\ldots\times d_m}$, we have $\frac{\|\calA\|_{\ell_{\infty}}}{\fro{\calA}}\in[1/d^*,1]$. Also from (\citealt{hu2015relations}, Lemma 5.1), $\frac{\nuc{\calA}}{\fro{\calA}}\in[1,\dmax^{(m-1)/2}]$. So we use Lemma~\ref{lemma:beta_n} with $\delta_1^- = 1/d^*, \delta_1^+ = 1, \delta_2^- =1, \delta_2^+ = \dmax^{(m-1)/2}$ and $s = \alpha \log(\dmax)$, then $t = m\log(\dmax) + \log 2+ \log(\log_2(d^*)) + \log(\log_2(\dmax^{(m-1)/2})) \leq 4m\log(\dmax)$. So with probability exceeding $1-\dmax^{-m}$, for all $\gamma_1\in[1/d^*,1]$ and $\gamma_2\in[1,\dmax^{(m-1)/2}]$, 
$$\beta_n(\gamma_1,\gamma_2) \lesssim_{m} \gamma_1\gamma_2\left(\sqrt{\frac{n\dmax}{d^*}} + 1\right)\log^{m+2}(\dmax) + \gamma_1\sqrt{\frac{n\log(\dmax)}{d^*}} + \gamma_1^2\log(\dmax).$$
Denote this event by $\bcalE_4$.

Now we denote the event 
$\bcalE_5^{i} = \left\{\op{\calP^{(i)}(\calP_{\Omega}-\frac{n}{d^*}\calI)\calP^{(i)}} \leq C_m\sqrt{\frac{\mu^2\rmax^2\dmax n\log(\dmax)}{(d^*)^2}}\right\}$ for all $i\in[m]$, where $\calP^{(i)}:\RR^{d_1\times\ldots\times d_m}\rightarrow\RR^{d_1\times\ldots\times d_m}$ in terms of its $i$-th separation:
$$(\calP^{(i)}\calX)\lr{i} = (T^{*\leq i-1}T^{*\leq i-1T}\otimes I)\calX\lr{i}V_{i+1}^*V_{i+1}^{*T}.$$
It is easy to see that $\calP^{(i)}$ is a projection and this operator is independent of the choice of left orthogonal representation of $\calT^*$.
Set $\bcalE_5 = \cap_{i=1}^m\bcalE_5^i$, then $\bcalE_5$ holds with probability exceeding $1-m\dmax^{-m}$ from the following lemma.
\begin{lemma}\label{lemma:inp:nodependence}
	Suppose that $\calT^* = [T_1^*,\ldots,T_m^*]\in\mfd$ satisfies $\incoh(\calT^*)\leq \sqrt{\mu}$. And $\Omega$ is sampled uniformly with replacement such that $|\Omega| = n$. Then we have with probability exceeding $1-\dmax^{-m}$,
	$$
	\op{\calP^{(i)}(\calP_{\Omega}-\frac{n}{d^*}\calI)\calP^{(i)}} \leq C_m\sqrt{\frac{\mu^2\rmax^2\dmax n\log(\dmax)}{(d^*)^2}}
	$$
	holds as long as $n\geq C\mu^2\rmax^2\dmax\log(\dmax)$.
\end{lemma}
When $\calA = [T_1^*,\ldots,T_{i-1}^*,A,T_{i+1}^*,\ldots,T_m^*], \calB = [T_1^*,\ldots,T_{i-1}^*,B,T_{i+1}^*,\ldots,T_m^*]$, then we have $\calP^{(i)}\calA = \calA$ and $\calP^{(i)}\calB = \calB$. Applying Cauchy-Schwartz leads to the following corollary of Lemma \ref{lemma:inp:nodependence}.
\begin{corollary}\label{coro:inp:nodependence}
	For any $i\in[m]$, set
	$$
	\calA = [T_1^*,\ldots,T_{i-1}^*,A,T_{i+1}^*,\ldots,T_m^*], \quad\calB = [T_1^*,\ldots,T_{i-1}^*,B,T_{i+1}^*,\ldots,T_m^*]
	$$
	for arbitrary $A,B\in\RR^{r_{i-1}\times d_i \times r_i}$. Then under the event $\bcalE_5^i$, we have 
	$$
	\inp{\calA}{(\calP_{\Omega}-\frac{n}{d^*}\calI)\calB} \leq C_m\sqrt{\frac{\mu^2\rmax^2\dmax n\log(\dmax)}{(d^*)^2}}\fro{\calA}\fro{\calB}.
	$$
\end{corollary}

And we denote $\bcalE = \bcalE_1\cap\bcalE_2\cap\bcalE_3\cap\bcalE_4\cap\bcalE_5$. Then $\bcalE$ holds with probability exceeding $1-(m+4)\dmax^{-m}$. Now we proceed assuming $\bcalE$ holds.

Using the idea of induction, we start the proof assuming $\fro{\calT_l - \calT^*} \leq \frac{\sigmamin}{600000m\kappa_0\sqrt{\rmax}}$ and $\incoh(\calT_l) \leq 2\kappa_0^2\nu$. For simplicity we drop the subscript and denote $\calT = \calT_l$ and $\TT = \TT_l$ in the following.

Now suppose we fix a left orthogonal decomposition of $\calT = [T_1,\ldots, T_m]$, we choose a left orthogonal decomposition for $\calT^*$ accordingly. First let $\calT^* = [T_1',\ldots,T_m']$ be a left orthogonal decomposition. Define $R_1 = \arg\min_{R\in\OO_{r_1}}\fro{T_1 - T_1'R}$. Now suppose we obtain $R_1,\ldots, R_{i-1}$, define $$R_i = \arg\min_{R\in\OO_{r_i}}\fro{T^{\leq i} - T^{'\leq i}R}.$$
In this way we obtain $R_1,\ldots,R_{m-1}$. And we define $L(T_i^*) = (R_{i-1}\otimes I)^TL(T_i')R_i$ for $i\in[m-1]$ using the convention $R_0 = [1]$ and $T_m^* = R_{m-1}^TT_m'$. Now we can prove by induction that $[T_1',\ldots,T_m'] = [T_1^*,\ldots, T_m^*]$ and $T^{'\leq i}R_i = T^{*\leq i}$ 
So we take $\calT^* = [T_1^*,\ldots, T_m^*]$ to be the left orthogonal decomposition, and it is the one such that $T^{*\leq i}$ and $T^{\leq i}$ are aligned in the sense that $d_c(T^{\leq i},T^{*\leq i}) = \fro{T^{\leq i} -T^{*\leq i}}$. And we can write $T^{\geq j+1} = \Lambda_{j+1}V_{j+1}^T$ and $T^{*\geq j+1} = \Lambda^*_{j+1}V_{j+1}^{*T}$ such that $\Lambda_{j+1},\Lambda_{j+1}\in\RR^{r_j}$ are invertible and $V_{j+1},V_{j+1}^*$ are orthogonal and  $d_c(V_{j+1},V_{j+1}^*) = \fro{V_{j+1}-V_{j+1}^*}$.

As a result from the above alignment process and Wedin's theorem, we have for all $i\in[m-1]$,
\begin{align}\label{chordaldistance}
	\max\{\fro{T^{\leq i}- T^{*\leq i}}, \fro{V_{i+1}-V_{i+1}^*}\} \leq \frac{2\fro{\calT- \calT^*}}{\sigmamin}.
\end{align}
We first derive the upper bounds for the terms in $\fro{\calT_{l+1} - \calT^*}^2$ in the following subsections.
\subsubsection{Estimation of $\inp{\calT-\calT^*}{\calP_{\Omega}(\calT - \calT^*)}$.}
Since the operator $\calP_{\Omega}$ is SPD, we have
\begin{align}\label{eq:p(t-t^*)}
	\inp{\calT-\calT^*}{\calP_{\Omega}(\calT - \calT^*)}\geq \frac{1}{2}\inp{\calP_{\Omega}\calP_{\TT^*}(\calT - \calT^*)}{\calP_{\TT^*}(\calT - \calT^*)} - \inp{\calP_{\Omega}\calP_{\TT^*}^{\perp}(\calT)}{\calP_{\TT^*}^{\perp}(\calT)}.
\end{align}
Since $\bcalE$ holds, we have
\begin{align}\label{pomegapt}
	\inp{\calP_{\Omega}\calP_{\TT^*}(\calT - \calT^*)}{\calP_{\TT^*}(\calT - \calT^*)} &\geq \frac{n}{2d^*}\fro{\calP_{\TT^*}(\calT-\calT^*)} = \frac{n}{2d^*}\fro{\calT-\calT^*}^2 - \frac{n}{2d^*}\fro{\calP_{\TT^*}^{\perp}(\calT)}^2\no\\
	&\overset{\text{Coro.}~\ref{coro:pt*perp}}{\geq}\frac{n}{2d^*}\fro{\calT-\calT^*}^2 - \frac{n}{d^*}\frac{200m^2\fro{\calT-\calT^*}^4}{\sigmamin^2}.
\end{align}

On the other hand, we consider the upper bound for $\inp{\calP_{\Omega}\calP_{\TT^*}^{\perp}\calT}{\calP_{\TT^*}^{\perp}\calT}$.
Since $\calT$ and $\Omega$ are dependent, we need to consider the empirical process \eqref{emp}.
First notice that
\begin{align*}
	\inp{\calP_{\Omega}\calP_{\TT^*}^{\perp}\calT}{\calP_{\TT^*}^{\perp}\calT} \leq \frac{n}{d^*}\fro{\calP_{\TT^*}^{\perp}\calT}^2 + \fro{\calP_{\TT^*}^{\perp}\calT}^2\beta_n\left(\frac{\|\calP_{\TT^*}^{\perp}\calT\|_{\ell_{\infty}}}{\fro{\calP_{\TT^*}^{\perp}\calT}},\frac{\nuc{\calP_{\TT^*}^{\perp}\calT}}{\fro{\calP_{\TT^*}^{\perp}\calT}}\right).
\end{align*}
Under $\bcalE_4$, we have for any $\calA\in\RR^{d_1\times\ldots\times d_m}$, 
\begin{align*}
	\fro{\calA}^2\beta_n\left(\frac{\|\calA\|_{\ell_{\infty}}}{\fro{\calA}}, \frac{\nuc{\calA}}{\fro{\calA}}\right) \lesssim_m \|\calA\|_{\ell_{\infty}}\nuc{\calA}&\left(\sqrt{\frac{n\dmax}{d^*}} + 1\right)\log^{m+2}(\dmax) \\
	&+ \|\calA\|_{\ell_{\infty}}\fro{\calA}\sqrt{\frac{n\log(\dmax)}{d^*}} + \|\calA\|_{\ell_{\infty}}^2\log(\dmax).
\end{align*}
And this implies that 
\begin{align}\label{inp pomega a,a}
	\inp{\calP_{\Omega}\calA}{\calA} \leq \frac{n}{d^*}\fro{\calA}^2 + C_m \|\calA\|_{\ell_{\infty}}\nuc{\calA}\left(\sqrt{\frac{n\dmax}{d^*}} + 1\right)\log^{m+2}(\dmax).
\end{align}
Now we focus on $\|\calP_{\TT^*}^{\perp}\calT\|_{\ell_{\infty}}, \nuc{\calP_{\TT^*}^{\perp}\calT}$.

\hspace{0.2cm} 

\noindent{\it Estimation of $\|\calP_{\TT^*}^{\perp}\calT\|_{\ell_{\infty}}$.} Notice that we have $\incoh(\calT)\leq 2\kappa_0^2\nu$ and $\incoh(\calT^*)\leq \kappa_0\nu$ from Lemma~\ref{lemma:spiki implies incoh}. Using triangle inequality, we have 
$\|\calP_{\TT^*}^{\perp}\calT\|_{\ell_{\infty}} \leq \|\calT\|_{\ell_{\infty}} + \|\calP_{\TT^*}\calT\|_{\ell_{\infty}}$. 
Meanwhile, for all $1\leq i\leq m-1$, we have 
\begin{align}\label{estimation of Lambda}
	\sigma_{\max}(\Lambda_{i+1}) \leq \sigma_{\max}(\Lambda_{i+1}^*)+\fro{\calT - \calT^*} \leq\frac{11}{10}\sigmamax.
\end{align}
As a result of the above inequality and Lemma~\ref{lemma:twoinf relation}, we have
\begin{align*}
	\|\calT\|_{\ell_{\infty}} = \|\calT\lr{i}\|_{\ell_{\infty}} = \|T^{\leq i}\Lambda_{i+1}V_{i+1}^T\|_{\ell_{\infty}} \leq \frac{11}{10}\sigmamax\twoinf{T^{\leq i}}\twoinf{V_{i+1}}\leq \frac{22}{5}\sigmamax\frac{\kappa_0^4\nu^2r_i}{\sqrt{d^*}}.
\end{align*}
As this holds for all $i\in[m-1]$, we have
\begin{align}\label{estimation of infinity norm of T}
	\|\calT\|_{\ell_{\infty}}\leq \frac{22}{5}\frac{\rmin\kappa_0^4\nu^2}{\sqrt{d^*}}\sigmamax.
\end{align}
On the other hand, we have $\calP_{\TT^*}\calT = \delta\calT_1 + \ldots + \delta\calT_m$, where $$\delta\calT_i = [T_1^*,\ldots,T_{i-1}^*,X_i, T_{i+1}^*,\ldots,T_m^*]$$ 
and the expression of $X_i$ are give in \eqref{tangent:W}. Now we would like to estimate $\|\delta\calT_i\|_{\ell_{\infty}}$. Since the reshape operation remains the infinity norm unchanged, we have for all $i\in[m-1]$,
\begin{align*}
	\|\delta\calT_i\|_{\ell_{\infty}} &= \|\delta\calT_i\lr{i}\|_{\ell_{\infty}}\\ 
	&= \|(T^{*\leq i-1}\otimes I)(I- L(T_i^*)L(T_i^*)^T)(T^{*\leq i-1}\otimes I)^T\calT\lr{i}V_{i+1}^*V_{i+1}^{*T}\|_{\ell_{\infty}}\\
	&\leq \|(T^{*\leq i-1}\otimes I)(T^{*\leq i-1}\otimes I)^T\calT\lr{i}V_{i+1}^*V_{i+1}^{*T}\|_{\ell_{\infty}} + \|T^{*\leq i}T^{*\leq iT}\calT\lr{i}V_{i+1}^*V_{i+1}^{*T}\|_{\ell_{\infty}},
\end{align*}
and for $i=m$,
\begin{align*}
	\|\delta\calT_m\|_{\ell_{\infty}} &= \|\delta\calT_m\lr{m}\|_{\ell_{\infty}}= \|(T^{*\leq m-1}\otimes I)(T^{*\leq m-1}\otimes I)^T\calT\lr{m}\|_{\ell_{\infty}}\\ &\overset{\text{Lemma}~\ref{lemma:reshape}}{=}\|T^{*\leq m-1}T^{*\leq m-1T}\calT\lr{m-1}\|_{\ell_{\infty}}.
\end{align*}

We check the term $\|(T^{*\leq i-1}\otimes I)(T^{*\leq i-1}\otimes I)^T\calT\lr{i}V_{i+1}^*V_{i+1}^{*T}\|_{\ell_{\infty}}$.
\begin{align}\label{infinity norm est:1}
	&~~~~\|(T^{*\leq i-1}\otimes I)(T^{*\leq i-1}\otimes I)^T\calT\lr{i}V_{i+1}^*V_{i+1}^{*T}\|_{\ell_{\infty}} \no\\
	&\overset{(a)}{\leq} \sqrt{d_{i+1}\ldots d_m}\|(T^{*\leq i-1}T^{*\leq i-1T}\otimes I)\calT\lr{i}\|_{\ell_{\infty}}\twoinf{V_{i+1}^*}\no\\
	&\overset{(b)}{\leq} \|T^{*\leq i-1}T^{*\leq i-1T}\calT\lr{i-1}\|_{\ell_{\infty}}\kappa_0\nu\sqrt{r_i},
\end{align}
where in $(a)$ we use Lemma~\ref{lemma:twoinf relation} and $\twoinf{V_{i+1}^*} = \twoinf{V_{i+1}^*V_{i+1}^{*T}}$; in $(b)$ we use Lemma~\ref{lemma:reshape} and $\incoh(\calT^*) \leq \kappa_0\nu$. Also, when $2\leq i\leq m$
\begin{align*}
	\|T^{*\leq i-1}T^{*\leq i-1T}\calT\lr{i-1}\|_{\ell_{\infty}} &= \|T^{*\leq i-1}T^{*\leq i-1T}T^{\leq i-1}\Lambda_iV_i^T\|_{\ell_{\infty}}\no\\
	&\leq \sigma_{\max}(\Lambda_{i})\twoinf{V_i}\cdot\max_j \|T^{\leq i-1T}T^{*\leq i-1}T^{*\leq i-1T}e_j\|_{\ell_2}\no\\
	&\overset{(a)}{\leq} \sigma_{\max}(\Lambda_{i})\twoinf{V_i} \twoinf{T^{\leq i-1}}\twoinf{T^{*\leq i-1}}\sqrt{d_1\ldots d_{i-1}}\no\\
	&\overset{(b)}{\leq} \frac{22}{5} \frac{\kappa_0^5\nu^3 r_{i-1}^{3/2}}{\sqrt{d^*}}\sigmamax,
\end{align*}
where in $(a)$ we use Lemma~\ref{lemma:twoinf relation} and in $(b)$ we use $\incoh(\calT^*)\leq \kappa_0\nu, \incoh(\calT)\leq 2\kappa_0^2\nu$ and \eqref{estimation of Lambda}. And when $i=1$, 
\begin{align*}
	\|T^{*\leq 0}T^{*\leq 0T}\calT\lr{0}\|_{\ell_{\infty}} = \|\calT\|_{\ell_{\infty}} \overset{\eqref{estimation of infinity norm of T}}{\leq}\frac{22}{5}\frac{\rmin\kappa_0^4\nu^2}{\sqrt{d^*}}\sigmamax.
\end{align*}
Combine these with \eqref{infinity norm est:1} and we get 
\begin{align*}
	\|(T^{*\leq i-1}\otimes I)(T^{*\leq i-1}\otimes I)^T\calT\lr{i}V_{i+1}^*V_{i+1}^{*T}\|_{\ell_{\infty}} &\leq
	\left\{\begin{array}{rlc}
		&\frac{22}{5}\frac{\kappa_0^5\nu^3 \rmin r_1^{1/2}}{\sqrt{d^*}}\sigmamax, \quad i=1\\
		&\frac{22}{5}\frac{\kappa_0^6\nu^4 r_{i-1}^{3/2}r_i^{1/2}}{\sqrt{d^*}}\sigmamax, \quad 2\leq i\leq m
	\end{array}
	\right.\\
	&\leq \frac{22}{5}\frac{\kappa_0^6\nu^4 \rmax^2}{\sqrt{d^*}}\sigmamax.
\end{align*}
Now for all $i\in [m-1]$, we check $\|T^{*\leq i}T^{*\leq iT}\calT\lr{i}V_{i+1}^*V_{i+1}^{*T}\|_{\ell_{\infty}}$.
\begin{align*}
	\|T^{*\leq i}T^{*\leq iT}\calT\lr{i}V_{i+1}^*V_{i+1}^{*T}\|_{\ell_{\infty}} &= \max_{j,k}|e_j^TT^{*\leq i}T^{*\leq iT}T^{\leq i}\Lambda_{i+1}V_{i+1}^TV_{i+1}^*V_{i+1}^{*T}e_k|\\
	&\leq \max_{j,k}\sigma_{\max}(\Lambda_{i+1})\|T^{\leq iT}T^{*\leq i}T^{*\leq iT}e_j\|_{\ell_2}\|V_{i+1}^TV_{i+1}^*V_{i+1}^{*T}e_k\|_{\ell_2}\\
	&\leq \sigma_{\max}(\Lambda_{i+1})\twoinf{T^{*\leq i}}\twoinf{T^{\leq i}}\twoinf{V_{i+1}}\twoinf{V_{i+1}^*}\sqrt{d^*}\\
	&\leq \frac{22}{5}\frac{\kappa_0^6\nu^4r_i^2}{\sqrt{d^*}}\sigmamax.
\end{align*}
Put these together and we have 
\begin{align}\label{ptperpt infinity}
	\|\calP_{\TT^*}^{\perp}\calT\|_{\ell_{\infty}} \leq \frac{44m}{5}\frac{\kappa_0^6\nu^4\rmax^2}{\sqrt{d^*}}\sigmamax.
\end{align}

\noindent{\it Estimation of $\nuc{\calP_{\TT^*}^{\perp}\calT}$.} First notice $\calP_{\TT^*}^{\perp}\calT = \calT - \calP_{\TT^*}\calT$. From Lemma~\ref{lemma:ttrank} and $\ranktt(\calT) \leq (r_1,\ldots,r_{m-1})$, we know $\ranktt(\calP_{\TT^*}^{\perp}\calT) \leq (3r_1, \ldots, 3r_{m-1})$. Now we use Lemma~\ref{lemma:nuclear norm and frobenius norm} and we get
\begin{align}\label{ptperpt nuclear}
	\nuc{\calP_{\TT^*}^{\perp}\calT} \leq 3^{(m-1)/2}\sqrt{r_1\ldots r_{m-1}}\fro{\calP_{\TT^*}^{\perp}\calT}.
\end{align}
Now we combine \eqref{ptperpt infinity}, \eqref{ptperpt nuclear}, \eqref{inp pomega a,a} and Lemma~\ref{lemma:estimationofptperp}, since $\fro{\calT-\calT^*}\leq\frac{\sigmamin}{20m}$,
\begin{align}\label{inp pomegaptperp ptperp}
	\inp{\calP_{\Omega}\calP_{\TT^*}^{\perp}\calT}{\calP_{\TT^*}^{\perp}\calT} 
	\leq &400m^2\frac{n}{d^*}\frac{\fro{\calT-\calT^*}^4}{\sigmamin^2} \no\\
	&+ C_m\frac{\kappa_0^7\nu^4\rmax^2}{\sqrt{d^*}}\sqrt{r_1\ldots r_{m-1}}\fro{\calT-\calT^*}^2\left(\sqrt{\frac{n\dmax}{d^*}}+1\right)\log^{m+2}(\dmax).
\end{align}
So as long as 
$n\geq C_m\left(\kappa_0^7\nu^4\sqrt{d^*}\rmax^2(r^*)^{1/2}\log^{m+2}(\dmax) + \kappa_0^{14}\nu^8\dmax\rmax^4r^*\log^{2m+4}(\dmax)\right)$,
we have from \eqref{pomegapt} and \eqref{inp pomegaptperp ptperp}, we have
$$\inp{\calP_{\Omega}\calP_{\TT^*}^{\perp}\calT}{\calP_{\TT^*}^{\perp}\calT} \leq \frac{1}{100}\inp{\calP_{\Omega}\calP_{\TT^*}(\calT-\calT^*)}{\calP_{\TT^*}(\calT-\calT^*)}.$$
Together with \eqref{eq:p(t-t^*)} and \eqref{pomegapt}, we get
\begin{align}\label{P(T-T^*)}
	\inp{\calP_{\Omega}\calP_{\TT^*}\calT}{\calP_{\TT^*}\calT}&\geq \frac{49}{100}\inp{\calP_{\Omega}\calP_{\TT^*}(\calT-\calT^*)}{\calP_{\TT^*}(\calT-\calT^*)}\no\\
	&\geq\frac{49}{100}(\frac{n}{2d^*}\fro{\calT-\calT^*}^2 - \frac{n}{d^*}\frac{200m^2\fro{\calT-\calT^*}^4}{\sigmamin^2})\no\\
	&\geq\frac{6n}{25d^*}\fro{\calT-\calT^*}^2,
\end{align}
where the last inequality holds since $\fro{\calT-\calT^*} \leq \frac{1}{600m}\sigmamin$.

\subsubsection{Estimation of $\fro{\calP_{\TT}\calP_{\Omega}(\calT - \calT^*)}^2$}
First notice that we have both $\calT$ and $\calT^*$ are of TT rank $(r_1,\ldots,r_{m-1})$. And $\incoh(\calT) \leq 2\kappa_0^2\nu$, $\incoh(\calT^*)\leq \kappa_0\nu \leq 2\kappa_0^2\nu =: \sqrt{\mu}$. First notice that
\begin{align}\label{P_TP t - t^*}
	\fro{\calP_{\TT}\calP_{\Omega}(\calT - \calT^*)}^2 \leq 1001\fro{\calP_{\TT}(\calP_{\Omega}-\frac{n}{d^*}\calI)(\calT - \calT^*)}^2 + 1.001\frac{n^2}{(d^*)^2}\fro{\calP_{\TT}(\calT - \calT^*)}^2.
\end{align} 
Now we check $\fro{\calP_{\TT}(\calP_{\Omega}-\frac{n}{d^*}\calI)(\calT - \calT^*)}$. From the variational representation of Frobenius norm, we can write it as 
$$\fro{\calP_{\TT}(\calP_{\Omega}-\frac{n}{d^*}\calI)(\calT - \calT^*)} = \inp{(\calP_{\Omega}-\frac{n}{d^*}\calI)(\calT - \calT^*)}{\calP_{\TT}(\calX_0)},$$ 
for some $\calX_0$ with $\fro{\calX_0}\leq 1$. Now we set 
$\calP_{\TT}(\calX_0) = \delta\calX_1 + \ldots + \delta\calX_m$, with $\delta\calX_i = [T_1,\ldots,X_i,\ldots,T_m]$. For all $i\in[m]$, we consider the bound for $\inp{(\calP_{\Omega}-\frac{n}{d^*}\calI)(\calT - \calT^*)}{\delta\calX_i}$. 
We can decompose $\calT- \calT^*$ as 
\begin{align}\label{T-T^*}
	\calT - \calT^* 
	&= [T_1^*,\ldots,\Delta_{i},\ldots,T_m^*] + \sum_{j=1}^{i-1}[T_1^*,\ldots,\Delta_j,\ldots,T_i,T_{i+1}^*,\ldots,T_m^*]\no\\
	&\qquad\qquad\qquad\qquad\qquad+ \sum_{j=i+1}^m [T_1,\ldots,T_i,T_{i+1},\ldots,\Delta_j,\ldots,T_m^*]\no\\
	&=: \calY_{i,i} + \sum_{j = 1}^{i-1}\calY_{i,j} + \sum_{j = i+1}^m\calY_{i,j},
\end{align}
where $\Delta_j = T_j - T_j^*$. 
Before the estimation, we need the following lemmas whose proofs are presented in the Section \ref{sec:proofsoftechnicallemmas}.

\begin{lemma}\label{lemma:inp:lowrank}
	Suppose that $\Omega$ is the set sampled uniformly with replacement with size $|\Omega| = n$. Then under the event $\bcalE_3$, we have for any tensors $\calA,\calB$ with TT rank $(r_1,\ldots,r_{m-1})$,
	\begin{align*}
		&~~~~|\inp{(\calP_{\Omega} - \frac{n}{d^*}\calI)\calA}{\calB}|\\
		&\leq C_m\left(\sqrt{\frac{n\dmax}{d^*}} + 1\right)\log^{m+2}(\dmax) \cdot
		\prod_{i=1}^{m}\left(\max_{x_i}\fro{A_i(:,x_i,:)}\cdot\fro{B_i}\wedge\max_{x_i}\fro{B_i(:,x_i,:)}\cdot\fro{A_i}\right),
	\end{align*}
	where $\calA = [A_1,\ldots,A_m]$ and $\calB = [B_1,\ldots,B_m]$ can be arbitrary decompositions such that $A_i,B_i\in\RR^{r_{i-1}\times d_i\times r_i}$.
\end{lemma}

\begin{lemma}\label{lemma:estimate T_i}
	Let $\calT$ be a tensor of rank $(r_1,\ldots,r_{m-1})$ such that $\incoh(\calT)\leq\sqrt{\mu}$,
	and it has a left orthogonal decomposition $\calT = [T_1,\ldots, T_m]$.
	Then we have
	$$\max_{x_i}\fro{T_i(:,x_i,:)}^2\leq \frac{\mu r_i}{d_i}, \quad \fro{T_i} = \sqrt{r_i}, i\in[m-1],$$
	$$\max_{x_m}\fro{T_m(:,x_m)}^2 \leq \sigma_{\max}^2(\calT)\frac{\mu r_{m-1}}{d_m}, \quad\fro{T_m} = \fro{\calT}\leq \sqrt{\rmin}\sigma_{\max}(\calT).$$
\end{lemma}
Now we present some bounds related to $\Delta_j$ and $X_j$.

\hspace{0.2cm}

\noindent{\it Properties for $\Delta_j$.}
For all $j\in[m-1]$, we estimate $\fro{\Delta_j} = \fro{L(T_j) - L(T_j^*)}$ as follows. Notice $\calT\lr{j} = (T^{\leq j-1}\otimes I)L(T_j)T^{\geq j+1},(\calT^*)\lr{j} = (T^{*\leq j-1}\otimes I)L(T_j^*)T^{*\geq j+1}$. So we have,
\begin{align*}
	\fro{L(T_j) - L(T_j^*)} &= \fro{(T^{\leq j-1}\otimes I)^T\calT\lr{j}V_{j+1}\Lambda_{j+1}^{-1} - (T^{*\leq j-1}\otimes I)^T(\calT^*)\lr{j}V_{j+1}^*(\Lambda_{j+1}^*)^{-1}}\no\\
	&\leq \fro{((T^{\leq j-1}\otimes I)^T - (T^{*\leq j-1}\otimes I)^T)\calT\lr{j}V_{j+1}\Lambda_{j+1}^{-1}} \no\\
	&~~~~+ \fro{(T^{*\leq j-1}\otimes I)^T(\calT\lr{j} - (\calT^*)\lr{j})V_{j+1}\Lambda_{j+1}^{-1}} \no\\
	&~~~~+ \fro{(T^{*\leq j-1}\otimes I)^T(\calT^*)\lr{j}(V_{j+1}\Lambda_{j+1}^{-1} -V_{j+1}^*(\Lambda_{j+1}^*)^{-1})}\no\\
	&\overset{(a)}{\leq} \sigma_{\min}^{-1}(\calT)\frac{\fro{\calT-\calT^*}}{2\sigma_{\min}(\calT^*)}\sqrt{r_j}\sigma_{\max}(\calT) + \sigma_{\min}^{-1}(\calT)\fro{\calT - \calT^*} \no\\
	&\quad+ \sqrt{\rmin}\sigma_{\max}(\calT^*)\frac{12\kappa_0\fro{\calT- \calT^*}}{\sigma_{\min}^2(\calT^*)}\no\\
	&\overset{(b)}{\leq}20\sqrt{r_j}\frac{\kappa_0^2\fro{\calT-\calT^*}}{\sigma_{\min}(\calT^*)},
\end{align*}
where in $(a)$ we use \eqref{chordaldistance} and the bound $$\fro{V_{j+1}\Lambda_{j+1}^{-1} -V_{j+1}^*(\Lambda_{j+1}^*)^{-1}} \leq \fro{(V_{j+1} - V_{j+1}^*)(\Lambda_{j+1}^*)^{-1}} + \fro{V_{j+1}(\Lambda_{j+1}^{-1} - (\Lambda_{j+1}^*)^{-1})}$$
and Lemma~\ref{lemma:distofinverse} and in $(b)$ we use $|\sigma_{\max}(\calT) - \sigma_{\max}(\calT^*)|\vee|\sigma_{\min}(\calT) - \sigma_{\min}(\calT^*)|\leq \frac{1}{10}\sigma_{\min}(\calT^*)$. When $i = m$, we have 
\begin{align*}
	\fro{T_m - T_m^*} &= \fro{T^{\leq m-1 T}\calT\lr{m-1} - T^{*\leq m-1 T}(\calT^*)\lr{m-1}}\no\\
	&\leq \fro{(T^{\leq m-1 T} - T^{*\leq m-1 T})(\calT^*)\lr{m-1}} + \fro{T^{\leq m-1 T}(\calT\lr{m-1} - (\calT^*)\lr{m-1})}\no\\
	&\leq 2\sqrt{r_{m-1}}\kappa_0\fro{\calT - \calT^*} + \fro{\calT - \calT^*}\leq 3\sqrt{r_{m-1}}\kappa_0\fro{\calT - \calT^*}.
\end{align*}
Therefore, we have
\begin{align}\label{eq:est:Delta}
	\fro{\Delta_i} \leq \left\{
	\begin{array}{rlc}
		&20\sigma_{\min}^{-1}(\calT^*)\sqrt{r_i}\kappa_0^2\fro{\calT-\calT^*}, i\in[m-1]\\
		&3\sqrt{r_{m-1}}\kappa_0\fro{\calT - \calT^*}, i = m.
	\end{array}
	\right.
\end{align}
On the other hand, from Lemma \ref{lemma:estimate T_i}, we have 
\begin{align}\label{eq:max:Delta}
	\max_{x_i}\fro{\Delta_i(:,x_i,:)} \leq \left\{
	\begin{array}{rlc}
		&2\sqrt{\mu r_i d_i^{-1}}, i\in[m-1]\\
		&2\sqrt{\mu r_{m-1} d_m^{-1}}\cdot\sigma_{\max}(\calT^*), i =m.
	\end{array}
	\right.
\end{align}

%Finally the second part follows from Lemma \ref{lemma:estimate T_i}.
%
%
%We also summarize the properties of $\Delta_i$ in the following lemma.
%\begin{lemma}\label{lemma:est:Delta}
%	Let $\calT, \calT^*\in\mfd$ be such that $\incoh(\calT), \incoh(\calT^*) \leq \sqrt{\mu}$. And let $\calT = [T_1,\ldots,T_m]$ and $\calT^* = [T_1^*,\ldots,T_m^*]$ be two left orthogonal decompositions.
%	Set $\Delta_i = T_i - T_i^*$, then we have
%	\begin{align*}
%		\fro{\Delta_i} \leq \left\{
%		\begin{array}{rlc}
%			&20\sigma_{\min}^{-1}(\calT^*)\sqrt{r_i}\kappa_0^2\fro{\calT-\calT^*}, i\in[m-1]\\
%			&3\sqrt{r_{m-1}}\kappa_0\fro{\calT - \calT^*}, i = m.
%		\end{array}
%		\right.
%	\end{align*}
%	And
%	\begin{align*}
%		\max_{x_i}\fro{\Delta_i(:,x_i,:)} \leq \left\{
%		\begin{array}{rlc}
%			&2\sqrt{\mu r_i d_i^{-1}}, i\in[m-1]\\
%			&2\sqrt{\mu r_{m-1} d_m^{-1}}\cdot\sigma_{\max}(\calT^*), i =m.
%		\end{array}
%		\right.
%	\end{align*}
%\end{lemma}

\hspace{0.2cm}

\noindent{\it Upper bound for $\fro{X_j}$.} For all $X_i, i\in[m-1]$, we have $L(X_i) = (I - L(T_i)L(T_i)^T)(T^{\leq i-1}\otimes I)^T\calX_0\lr{i}T^{\geq i+1T}(T^{\geq i+1}T^{\geq i+1T})^{-1}$, thus $\fro{X_i}\leq \sigma_{\min}^{-1}(\calT)\leq 2\sigma_{\min}^{-1}(\calT^*)$. For $i =m$, $L(X_m) = (T^{\leq m-1}\otimes I)^T\calX_0\lr{m}$, so $\fro{X_m}\leq 1$. Therefore, we obtain
\begin{align}\label{eq:upper:X}
	\fro{X_i}\leq 2\sigma_{\min}^{-1}(\calT^*),i\in[m-1], \quad \fro{X_m}\leq 1.
\end{align}
Now we consider the upper bound for $\inp{(\calP_{\Omega}-\frac{n}{d^*}\calI)(\calT-\calT^*)}{\delta\calX_i}$.

\hspace{0.2cm}

\noindent{\it When $i\in[m-1]$.}
Due to \eqref{T-T^*}, we can write it as
$$
\inp{(\calP_{\Omega}-\frac{n}{d^*}\calI)(\calT-\calT^*)}{\delta\calX_i} = \sum_{j = 1}^m\inp{(\calP_{\Omega}-\frac{n}{d^*}\calI)\calY_{i,j}}{\delta\calX_i}.
$$
Now for all $j\neq i,m$, we consider $\inp{(\calP_{\Omega}-\frac{n}{d^*}\calI)(\calY_{i,j})}{\delta\calX_i}$. 
By setting $\calA = \calY_{i,j}$, $\calB = \delta\calX_i$ in Lemma~\ref{lemma:inp:lowrank}, together with Lemma \ref{lemma:estimate T_i}, \eqref{eq:est:Delta}, \eqref{eq:max:Delta}, \eqref{eq:upper:X} and we have under the event $\bcalE_3$,
$$
|\inp{(\calP_{\Omega}-\frac{n}{d^*}\calI)(\calY_{i,j})}{\delta\calX_i}|\leq C_m\left(\sqrt{\frac{n\dmax}{d^*}} + 1\right)\log^{m+2}(\dmax)\kappa_0^4\mu^{m/2}\frac{r^*\cdot r_{m-1}}{\sqrt{r_i d^*}}\fro{\calT - \calT^*}.
$$
For $j = m$, we have under the event $\bcalE_3$,
$$
|\inp{(\calP_{\Omega}-\frac{n}{d^*}\calI)(\calY_{i,m})}{\delta\calX_i}|\leq  C_m\left(\sqrt{\frac{n\dmax}{d^*}} + 1\right)\log^{m+2}(\dmax)\kappa_0^2\mu^{m/2}\frac{r^*\cdot r_{m-1}}{\sqrt{r_i d^*}}\fro{\calT - \calT^*}.
$$
And when $j = i$, we estimate $|\inp{(\calP_{\Omega}-\frac{n}{d^*}\calI)(\calY_{i,i})}{\delta\calX_i}|$. We can write 
\begin{align}\label{partition deltaX}
	\delta\calX_i &= [T_1^*,\ldots, X_i,\ldots,T_m^*] + [\Delta_1, T_2,\ldots,X_i,\ldots,T_m] + \ldots + [T_1^*,\ldots,X_i,T_{i+1}^*,\ldots,\Delta_m] \no\\
	&=: \calX_{i,0} + \sum_{k = 1, k\neq i}^m \calX_{i,k}.
\end{align}
Then $\inp{(\calP_{\Omega}-\frac{n}{d^*}\calI)(\calY_{i,i})}{\delta\calX_i} = \inp{(\calP_{\Omega}-\frac{n}{d^*}\calI)(\calY_{i,i})}{\calX_{i,0}} + \sum_{k=1,k\neq i}^m\inp{(\calP_{\Omega}-\frac{n}{d^*}\calI)(\calY_{i,i})}{\calX_{i,k}}$. And here the first term can be bounded using Corollary~\ref{coro:inp:nodependence},
$$|\inp{(\calP_{\Omega}-\frac{n}{d^*}\calI)(\calY_{i,i})}{\calX_{i,0}}| \leq C_m\kappa_0^4\frac{\mu\rmax\sqrt{n\dmax\log(\dmax)}}{d^*}\sqrt{r_i}\fro{\calT-\calT^*} \leq C_m\kappa_0^4\frac{\mu\rmax^{3/2}\sqrt{n\dmax\log(\dmax)}}{d^*}\fro{\calT-\calT^*}.$$
When $k\neq i,m$, we have 
$$
|\inp{(\calP_{\Omega}-\frac{n}{d^*}\calI)(\calY_{i,i})}{\calX_{i,k}}| \leq 
C_m\left(\sqrt{\frac{n\dmax}{d^*}} + 1\right)\log^{m+2}(\dmax)\mu^{m/2}\frac{r^*\cdot r_{m-1}}{\sqrt{r_id^*}}\kappa_0^4\fro{\calT-\calT^*}.
$$
When $k =m$, we have
$$
|\inp{(\calP_{\Omega}-\frac{n}{d^*}\calI)(\calY_{i,i})}{\calX_{i,m}}| \leq C_m\left(\sqrt{\frac{n\dmax}{d^*}} + 1\right)\log^{m+2}(\dmax)\mu^{m/2}\frac{r^*\cdot r_{m-1}}{\sqrt{r_id^*}}\kappa_0^2\fro{\calT-\calT^*}.
$$
In summary, under the event $\bcalE_3$, we have 
$$
	|\inp{(\calP_{\Omega}-\frac{n}{d^*}\calI)(\calY_{i,i})}{\delta\calX_i}| \leq C_m\kappa_0^4\mu^{m/2}\log^{m+2}(\dmax)\left(\sqrt{\frac{n\dmax}{d^*}} + 1\right)\frac{r^*\cdot r_{m-1}r_i^{-1/2}}{\sqrt{d^*}}\fro{\calT-\calT^*}.
$$
Together with the estimation when $j\neq i$, we see that under the event $\bcalE_3$,
\begin{align*}
	\inp{(\calP_{\Omega}-\frac{n}{d^*}\calI)(\calT-\calT^*)}{\delta\calX_i} 
	\leq
	C_m\kappa_0^4\mu^{m/2}\log^{m+2}(\dmax)\left(\sqrt{\frac{n\dmax}{d^*}} + 1\right)\frac{r^*\cdot r_{m-1}r_i^{-1/2}}{\sqrt{d^*}}\fro{\calT-\calT^*}.
\end{align*}

\hspace{0.2cm}

\noindent{\it When $i=m$.}
Similarly, we write 
$$
\inp{(\calP_{\Omega}-\frac{n}{d^*}\calI)(\calT-\calT^*)}{\delta\calX_m} = \sum_{j = 1}^m\inp{(\calP_{\Omega}-\frac{n}{d^*}\calI)\calY_{m,j}}{\delta\calX_m}.
$$
First when $j\in[m-1]$, we have
\begin{align}\label{part:1-m-1}
	|\inp{(\calP_{\Omega}-\frac{n}{d^*}\calI)(\calY_{m,j})}{\delta\calX_m}|\leq  C_m\left(\sqrt{\frac{n\dmax}{d^*}} + 1\right)\log^{m+2}(\dmax)\kappa_0^3\mu^{m/2}\frac{r^*\cdot r_{m-1}^{1/2}}{\sqrt{d^*}}\fro{\calT - \calT^*}.
\end{align}
And when $j = m$, we consider $\inp{(\calP_{\Omega}-\frac{n}{d^*}\calI)(\calY_{m,m})}{\delta\calX_m} = \inp{(\calP_{\Omega}-\frac{n}{d^*}\calI)(\calY_{m,m})}{\calX_{m,0}} + \sum_{k=1}^{m-1}\inp{(\calP_{\Omega}-\frac{n}{d^*}\calI)(\calY_{m,m})}{\calX_{m,k}}$ as in \eqref{partition deltaX}. Similarly, using Lemma~\ref{lemma:inp:nodependence}, we have
\begin{align*}
	\inp{(\calP_{\Omega}-\frac{n}{d^*}\calI)(\calY_{m,m})}{\calX_{m,0}} &\leq C_m\kappa_0\frac{\mu\rmax\sqrt{n\dmax\log(\dmax)}}{d^*}\sqrt{r_{m-1}}\fro{\calT-\calT^*}\\
	&\leq  C_m\kappa_0\frac{\mu\rmax^{3/2}\sqrt{n\dmax\log(\dmax)}}{d^*}\fro{\calT-\calT^*}.
\end{align*}
And when $k\in[m-1]$, we have
$$
\inp{(\calP_{\Omega}-\frac{n}{d^*}\calI)(\calY_{m,m})}{\calX_{m,k}} \leq C_m\kappa_0^3\mu^{m/2}\log^{m+2}(\dmax)\left(\sqrt{\frac{n\dmax}{d^*}} + 1\right)\frac{r^*\cdot  r_{m-1}^{1/2}}{\sqrt{d^*}}\fro{\calT-\calT^*}.
$$
So under the event $\bcalE_3$, we have
$$
	|\inp{(\calP_{\Omega}-\frac{n}{d^*}\calI)(\calY_{m,m})}{\delta\calX_m}| \leq C_m\kappa_0^3\mu^{m/2}\log^{m+2}(\dmax)\left(\sqrt{\frac{n\dmax}{d^*}} + 1\right)\frac{r^*\cdot  r_{m-1}^{1/2}}{\sqrt{d^*}}\fro{\calT-\calT^*},
$$
and 
\begin{align}\label{part:m}
	\inp{(\calP_{\Omega}-\frac{n}{d^*}\calI)(\calT-\calT^*)}{\delta\calX_m} \leq C_m\kappa_0^3\mu^{m/2}\log^{m+2}(\dmax)\left(\sqrt{\frac{n\dmax}{d^*}} + 1\right)\frac{r^*\cdot r_{m-1}^{1/2}}{\sqrt{d^*}}\fro{\calT - \calT^*}.
\end{align}
Now we conclude from \eqref{part:1-m-1}, \eqref{part:m} and under the event $\bcalE_3$,
\begin{align*}
	\fro{\calP_{\TT}(\calP_{\Omega}-\frac{n}{d^*}\calI)(\calT - \calT^*)}\leq C_m\kappa_0^4\mu^{m/2}\log^{m+2}(\dmax)\left(\sqrt{\frac{n\dmax}{d^*}} + 1\right)\frac{r^* r_{m-1}\cdot\sum_{i=1}^{m-1}r_i^{-1/2}}{\sqrt{d^*}}\fro{\calT - \calT^*}.
\end{align*}
Together with \eqref{P_TP t - t^*}, as long as $$n\geq C_m\kappa_0^8\mu^m\log^{2m+4}(\dmax)\cdot \dmax(r^*)^2r_{m-1}^2(\sum_{i=1}^{m-1}r_i^{-1}) + C_m\kappa_0^4\mu^{m/2}\log^{m+2}(\dmax)\cdot (d^*)^{1/2}r^*r_{m-1}(\sum_{i=1}^{m-1}r_i^{-1/2}),$$
we have under the event $\bcalE_3$,
\begin{align}\label{est:ptpomega T-T^*}
	\fro{\calP_{\TT}\calP_{\Omega}(\calT - \calT^*)}^2 \leq 1.002\frac{n^2}{(d^*)^2}\fro{\calT - \calT^*}^2.
\end{align}

\subsubsection{Estimation of $\inp{\calP_{\TT}^{\perp}(\calT - \calT^*)}{\calP_{\Omega}(\calT - \calT^*)}$}
Now we derive the upper bound for $\inp{\calP_{\TT}^{\perp}(\calT - \calT^*)}{\calP_{\Omega}(\calT - \calT^*)}$. First we have
$$
\inp{\calP_{\TT}^{\perp}(\calT - \calT^*)}{\calP_{\Omega}(\calT - \calT^*)} = \inp{\calP_{\TT}^{\perp}(\calT - \calT^*)}{(\calP_{\Omega} - \frac{n}{d^*}\calI)(\calT - \calT^*)} + \frac{n}{d^*}\fro{\calP_{\TT}^{\perp}\calT^*}^2.
$$
As in \eqref{T-T^*}, we write $\calT - \calT^*$ as $\calT - \calT^* = \sum_{j=1}^m\calY_{m,j}$. Notice as estimated above \eqref{ptperpt nuclear}, we have $\ranktt(\calP_{\TT}^{\perp}(\calT - \calT^*)) \leq (3r_1,\ldots,3r_{m-1})$. So there exists a TT decomposition of $\calP_{\TT}^{\perp}(\calT - \calT^*) = [Y_1,\ldots,Y_{m-1}, Y_m]$ such that $\fro{Y_i} = \sqrt{3r_i}, i\in[m-1]$ and $\fro{Y_m} = \fro{\calP_{\TT}^{\perp}(\calT - \calT^*)}$. 

When $j\in[m-1]$, we estimate $\inp{\calP_{\TT}^{\perp}(\calT - \calT^*)}{(\calP_{\Omega} - \frac{n}{d^*}\calI)\calY_{m,j}}$. Using Lemma~\ref{lemma:inp:lowrank}, Lemma \ref{lemma:estimate T_i}, \eqref{eq:est:Delta}, \eqref{eq:max:Delta}, \eqref{eq:upper:X} and Lemma \ref{lemma:estimationofptperp}, and we have under the event $\bcalE_3$,
\begin{align*}
	|\inp{\calP_{\TT}^{\perp}(\calT - \calT^*)}{(\calP_{\Omega} - \frac{n}{d^*}\calI)\calY_{m,j}}| \leq C_m\left(\sqrt{\frac{n\dmax}{d^*}} + 1\right)\log^{m+2}(\dmax)\mu^{m/2}\kappa_0\frac{r^*\cdot r_{m-1}^{1/2}}{\sqrt{d^*}}\fro{\calT-\calT^*}^2.
\end{align*}
And when $j = m$, we have under the event $\bcalE_3$,
\begin{align*}
	|\inp{\calP_{\TT}^{\perp}(\calT - \calT^*)}{(\calP_{\Omega} - \frac{n}{d^*}\calI)\calY_{m,m}}| \leq C_m\left(\sqrt{\frac{n\dmax}{d^*}} + 1\right)\log^{m+2}(\dmax)\mu^{m/2}\kappa_0\frac{r^*\cdot r_{m-1}^{1/2}}{\sqrt{d^*}}\fro{\calT-\calT^*}^2.
\end{align*}
So we conclude under the event $\bcalE_3$, 
\begin{align}\label{ptperp p t-t^*}
	&~~~~|\inp{\calP_{\TT}^{\perp}(\calT - \calT^*)}{\calP_{\Omega}(\calT - \calT^*)}|\no\\ 
	&\leq C_m\left(\sqrt{\frac{n\dmax}{d^*}} + 1\right)\log^{m+2}(\dmax)\mu^{m/2}\kappa_0\frac{r^*\cdot r_{m-1}^{1/2}}{\sqrt{d^*}}\fro{\calT-\calT^*}^2 + \frac{n}{d^*}\frac{300m^2\fro{\calT-\calT^*}^4}{\sigma_{\min}^2(\calT^*)}\no\\
	&\leq \frac{n}{1000d^*}\fro{\calT_l - \calT^*}^2,
\end{align}
where the last inequality holds as long as $$n\geq C_m\log^{m+2}(\dmax)\mu^{m/2}\kappa_0r^* \rmax^{1/2}\sqrt{d^*} + C_m \dmax\rmax (r^*)^2\mu^m\kappa_0^2\log^{2m+4}(\dmax).$$

\subsubsection{Contraction}\label{sec:contraction}
Now we consider the error $\fro{\calT_{l+1} - \calT^*}$ assuming $\bcalE$ holds.
From Algorithm~\ref{alg:rgradtt}, we have
\begin{align}\label{T_{l+1} - T^*}
	\fro{\calT_{l+1} - \calT^*}^2 &= \fro{\ttsvd_{\r}(\wt\calW_l) - \calT^*}^2
	\overset{\text{Lemma}~\ref{lemma:tt-perturbation}}{\leq} \fro{\wt\calW_l - \calT^*}^2 + \frac{600m\fro{\wt\calW_l - \calT^*}^3}{\sigmamin}.
\end{align}
From the way we choose $\zeta_l$, we have $\fro{\wt\calW_l - \calT^*}\leq \fro{\calW_l - \calT^*}$. Now we estimate $\fro{\calW_l - \calT^*}^2$.
\begin{align*}
	\fro{\calW_l - \calT^*}^2 &= \fro{\calT_l - \calT^* - \alpha_l\calP_{\TT_l}\calP_{\Omega}(\calT_l - \calT^*)}^2\\
	&= \fro{\calT_l-\calT^*}^2 - 2\alpha_l\underbrace{\inp{\calT_l-\calT^*}{\calP_{\Omega}(\calT_l - \calT^*)}}_{\RN{1}} \\
	&\qquad+ 2\alpha_l\underbrace{\inp{\calP_{\TT_l}^{\perp}(\calT_l - \calT^*)}{\calP_{\Omega}(\calT_l - \calT^*)}}_{\RN{2}}
	+ \alpha_l^2\underbrace{\fro{\calP_{\TT_l}\calP_{\Omega}(\calT_l - \calT^*)}^2}_{\RN{3}}.
\end{align*}
From \eqref{P(T-T^*)},\eqref{ptperp p t-t^*} and \eqref{est:ptpomega T-T^*}, when
\begin{align*}
	n \geq C_m\bigg(\kappa_0^{2m+4}\nu^{m+1}\log^{m+2}(\dmax)\cdot(d^*)^{1/2}((r^*)^{1/2}\rmax^2\vee r^*r_{m-1}\rmin^{-1/2}\vee r^*\rmax^{1/2})\\
	+\kappa_0^{4m+8}\nu^{2m+2}\log^{2m+4}(\dmax)\cdot\dmax(r^*\rmax^4\vee (r^*)^2r^2_{m-1}\rmin^{-1}\vee (r^*)^2\rmax)\bigg),
\end{align*}
we obtain $$\RN{1} \geq \frac{6n}{25d^*}\fro{\calT_l-\calT^*}^2,\quad |\RN{2}| \leq \frac{n}{1000d^*}\fro{\calT_l - \calT^*}^2, \quad |\RN{3}| \leq 1.002\frac{n^2}{(d^*)^2}\fro{\calT_l - \calT^*}^2.$$
From these estimations and we get,
\begin{align*}
	\fro{\calW_l - \calT^*}^2 \leq \fro{\calT_l - \calT^*}^2(1 - 0.239\alpha_l\frac{n}{d^*} + 1.002\alpha_l^2\frac{n^2}{(d^*)^2}).
\end{align*}
Now we set $\alpha_l = 0.12\frac{d^*}{n}$, and we get $\fro{\calW_l - \calT^*} \leq 0.986\fro{\calT_l - \calT^*}$.
When $\fro{\calT_l - \calT^*}\leq \frac{\sigmamin}{600000m}$, we have from \eqref{T_{l+1} - T^*}
\begin{align*}
	\fro{\calT_{l+1} - \calT^*}^2 \leq 0.975\fro{\calT_l - \calT^*}^2.
\end{align*}
And this implies $\fro{\calT_{l+1} - \calT^*}\leq\frac{\sigmamin}{600000m\kappa_0\sqrt{\rmax}}$ and $\incoh(\calT_{l+1})\leq 2\kappa_0^2\nu$ is implied by Lemma \ref{lemma:ttsvd+trim}. So we finish the proof of Lemma \ref{thm:localconvergence:detail}.

\subsection{Proof of Lemma \ref{thm:init}}
Before we start the proof, we give a detailed version of the theorem, notice here we set $|\Omega_i| = n_i$.
\begin{lemma}[Restate of Lemma \ref{thm:init}]\label{thm:init:detail}
	Suppose the conditions of $\calT^{\ast}$ from Theorem \ref{thm:main} hold. For any absolute constant $C>0$, there exists an absolute constant $C_m>0$ depending only on $m$ such that if 
\begin{align*}
	n &= \sum_{i=1}^{2m-1}n_i\geq C_m\nu^{m+3}\kappa_0^{4m-4}\log^2(\dmax)\bigg((d^\ast)^{1/2}(r^*\rmax\rmin^2)^{1/2} + \dmax r^*\rmax\rmin^2\\
	&+\sum_{k=1}^{m-2}\big((d_k\cdots d_m)^{1/2}(r_1\cdots r_{k-1})^3r_k(r_{k+1}\cdots r_{m-1})^{1/2}(\rmax\rmin^2)^{1/2} 
	+ \dmax(r_1\cdots r_{k-1})^2r^*\rmax\rmin^2\big)\\
	&+(\dmax r^*\rmin\rmax r_{m-1})^{1/2} + \dmax\rmax\rmin r_{m-1} + r^*\rmin\rmax r_{m-1}\bigg),
\end{align*}
	then with probability at least $1 - m\dmax^{-m}$, the output of Algorithm \ref{alg:init} satisfies 
	$$
	\fro{\calT_0-\calT^*} \leq \frac{\sigmamin}{Cm\kappa_0^2\rmax^{1/2}}\quad \text{~and~}\quad \incoh(\calT_0)\leq 2\kappa_0^2\nu.
	$$
\end{lemma}

\noindent{\it Step 0:} 
We denote $R_i = \arg\min_{R\in\OO_{r_i}}\fro{\hat T^{\leq i} - T^{*\leq i}R}$ and we set $\sqrt{\mu} = 2\kappa_0^2\nu$. 

\hspace{0.2cm}

\noindent{\it Step 1:}  When $i =1$.\newline
Firstly from Wedin's sin$\Theta$ theorem, and from Lemma \ref{lemma:init}, we see that when $n_1,n_2\geq C_m\nu^2\kappa_0^{2m-2}\cdot (d^*)^{1/2}(r^*\rmax)^{1/2}\rmin\log^2(\dmax) + C_m\nu^4\kappa_0^{4m-4}\dmax r^*\rmax \rmin^2$, we have 
\begin{align}\label{bound:N1-N1*}
	d_p(\wt T_1,T_1^*) \leq \frac{2\sqrt{r_1}\op{N_1 - N_1^*}}{\sigmamin^2} \leq (C_m\kappa_0^{2m-4}(r_2\cdots r_{m-1})^{1/2}\rmax^{1/2})^{-1}.
\end{align}
Now from Lemma \ref{lemma:trim}, we know $\incoh(\hat T_1)\leq \sqrt{3\mu}$ and 
$d_c(\hat T_1,T_1^*)\leq (C_m\kappa_0^{2m-4}(r_2\cdots r_{m-1})^{1/2}\rmax^{1/2})^{-1}$.

\hspace{0.2cm}

\noindent{\it Step 2:} When $2\leq i\leq m-1$.\newline
Suppose we already have $\incoh(\hat T^{\leq i-1})\leq (3\mu)^{(i-1)/2} (r_1\cdots r_{i-2})^{3/2} =: \sqrt{\mu_{i-1}}$ and $d_c(\hat T^{\leq i-1},T^{*\leq i-1}) \leq (C_{i-1}m^2\kappa_0^{2t_{i-1}}\sqrt{\rmax\cdot r_{i}})^{-1}$.
From Lemma \ref{lemma:init:1}, we see that,
\begin{align*}
	&~~~~d_p(L(\wt T_i), (R_{i-1}\otimes I)^TL(T_i^*)) \\
	&\leq \frac{2\sqrt{r_i}\op{(\hat T^{\leq i-1}\otimes I)^TN_i(\hat T^{\leq i-1}\otimes I) -(T^{*\leq i-1}R_{i-1}\otimes I)^TN_i^*(T^{*\leq i-1}R_{i-1}\otimes I)}}{\sigmamin^2}.
\end{align*}
Notice that 
\begin{align*}
	&~~~~\op{(\hat T^{\leq i-1}\otimes I)^TN_i(\hat T^{\leq i-1}\otimes I) -(T^{*\leq i-1}R_{i-1}\otimes I)^TN_i^*(T^{*\leq i-1}R_{i-1}\otimes I)}\\
	&\leq \op{(\hat T^{\leq i-1}\otimes I)^T(N_i - N_i^*)(\hat T^{\leq i-1}\otimes I)} + 2\op{\hat T^{\leq i-1} - T^{*\leq i-1}R_{i-1}}\cdot\sigmamax^2.
\end{align*}
So we have 
\begin{align}\label{dp dc}
	d_p(L(\wt T_i), (R_{i-1}\otimes I)^TL(T_i^*)) &\leq 2\sqrt{r_i}\sigmamin^{-2}\op{(\hat T^{\leq i-1}\otimes I)^T(N_i - N_i^*)(\hat T^{\leq i-1}\otimes I)} \no\\
	&\qquad+ 4\sqrt{r_i}\kappa_0^2 d_c(\hat T^{\leq i-1}, T^{*\leq i-1})\no\\
	&=: b_{i-1} + 4\sqrt{r_i}\kappa_0^2x_{i-1}.
\end{align}

Now we derive the chordal distance between $\hat T^{\leq i}$ and $T^{*\leq i}$. Notice that
\begin{align*}
	d_{c}(\hat T^{\leq i}, T^{*\leq i}) &\leq \sqrt{r_i}d_{c}(\hat T^{\leq i-1}, T^{*\leq i-1}) + d_{c}(L(\hat T_i), (R_{i-1}\otimes I)^TL(T_i^*))\no\\
	&\leq \sqrt{r_i}d_{c}(\hat T^{\leq i-1}, T^{*\leq i-1}) + \sqrt{2}d_p(L(\hat T_i), (R_{i-1}\otimes I)^TL(T_i^*))\no\\
	&\overset{\text{Lemma~}\ref{lemma:trim}}{\leq} \sqrt{r_i}d_{c}(\hat T^{\leq i-1}, T^{*\leq i-1}) + 4\sqrt{2}\pi\cdot d_{p}(L(\wt T_i), (R_{i-1}\otimes I)^TL(T_i^*))
\end{align*}
Together with \eqref{dp dc}, and denote by $q_{i-1} = 80\sqrt{r_i}\kappa_0^2$, we have $x_i\leq q_{i-1}x_{i-1} + b_{i-1}$. Sum this up and we have
$$
x_{m-1}\leq q_{m-1}\cdots q_1 x_1 + \sum_{k = 1}^{m-2}q_{m-2}\cdots q_{k+1} b_k.
$$
From Lemma \ref{lemma:init:1}, as long as 
\begin{align*}
	n_{2k+1},n_{2k+2} \geq C_m \nu^{k+2}\kappa_0^{2m}\log^2(\dmax)\cdot (d_k\cdots d_m)^{1/2}(r_1\cdots r_{k-1})^3r_k(r_{k+1}\cdots r_{m-1})^{1/2}(\rmax\rmin^2)^{1/2}\\
	+C_m\nu^{k+4}\kappa_0^{4m-2k}\log(\dmax)\cdot \dmax(r_1\cdots r_{k-1})^3r_kr_{k+1}\cdots r_{m-1}(\rmax\rmin^2),
\end{align*}
we have 
$q_{m-2}\cdots q_{k+1} b_k \leq \frac{1}{Cm^2\kappa_0^2\sqrt{\rmax}}$. Together with the estimation in Step 1, we have 
\begin{align}\label{dc T m-1}
	d_c(\hat T^{\leq m-1}, T^{*\leq m-1}) = x_{m-1} \leq (Cm\kappa_0^2\sqrt{\rmax})^{-1}.
\end{align}

\noindent{\it Step 3:}  When $i = m$. 
We have 
\begin{align*}
	\fro{\calT^* - \hat \calT} &= \fro{T^{*\leq m-1}R_{m-1}R_{m-1}^TT_m^* - \hat T^{\leq m-1}\hat T_m}\no\\
	&\leq \fro{(T^{*\leq m-1}R_{m-1} - \hat T^{\leq m-1})R_{m-1}^TT_m^*} + \fro{R_{m-1}^TT_m^* - \hat T_m}\no\\
	&\leq d_c(T^{*\leq m-1}, \hat T^{\leq m-1})\cdot\sigmamax + \fro{R_{m-1}^TT_m^* - \hat T_m}.
\end{align*}

Notice $d_c(T^{*\leq m-1}, \hat T^{\leq m-1})$ is estimated in $\eqref{dc T m-1}$. On the other hand, we have
\begin{align}\label{init:est}
	\fro{R_{m-1}^TT_m^* - \hat T_m} &= \fro{(T^{*\leq m-1}R_{m-1})^T(\calT^*)\lr{m-1} - (\hat T^{\leq m-1})^T(\frac{d^*}{n_{2m-1}}\calP_{\Omega_{2m-1}}(\calT^*))\lr{m-1}}\no\\
	&\leq \fro{(\hat T^{\leq m-1})^T((\calT^*)\lr{m-1} - (\frac{d^*}{n_{2m-1}}\calP_{\Omega_{2m-1}}(\calT^*))\lr{m-1})} \no\\
	&\qquad+ \fro{(T^{*\leq m-1}R_{m-1} - \hat T^{\leq m-1})^T(\calT^*)\lr{m-1}}\no\\
	&\leq \fro{(\hat T^{\leq m-1})^T((\calT^*)\lr{m-1} - (\frac{d^*}{n_{2m-1}}\calP_{\Omega_{2m-1}}(\calT^*))\lr{m-1})}\no\\ 
	&\qquad+ d_c(T^{*\leq m-1}, \hat T^{\leq m-1})\cdot\sigmamax.
\end{align}
Together with Lemma \ref{lemma:sample bound}, we have as long as 
\begin{align*}
	n_{2m-1}&\geq C_m\nu^{(m+1)/2}\kappa_0^{m+1}\log(\dmax)\cdot(\dmax r^*\rmin\rmax r_{m-1})^{1/2} + C_m\nu\kappa_0^{4}\log(\dmax)\cdot\dmax \rmin\rmax r_{m-1} \\
	&\qquad+ C_m\nu^{m}\kappa_0^{2m+2}\log(\dmax)\cdot r^*\rmin\rmax r_{m-1},
\end{align*}
we have 
\begin{align}\label{eq:last step}
	\fro{(\hat T^{\leq m-1})^T((\calT^*)\lr{m-1} - (\frac{d^*}{n_{2m-1}}\calP_{\Omega_{2m-1}}(\calT^*))\lr{m-1})} \leq \frac{\sigmamin}{Cm\kappa_0\sqrt{\rmax}}.
\end{align}
From \eqref{dc T m-1} - \eqref{init:est} and we conclude with probability exceeding $1-m\dmax^{-m}$,
$$
\fro{\calT^* - \hat \calT} \leq \frac{\sigmamin}{Cm\kappa_0\sqrt{\rmax}}.
$$
Finally, together with Lemma \ref{lemma:ttsvd+trim}, we conclude that the output $\calT_0$ satisfies 
$$
\fro{\calT_0-\calT^*}\leq \frac{\sigmamin}{Cm\kappa_0^2\sqrt{\rmax}} \text{~and~} \incoh(\calT_0)\leq 2\kappa_0^2\nu.
$$
And this finishes the proof of the lemma.

\subsection{Proof of Lemma \ref{lemma:init:noise}}\label{sec:pf:init:noise}
\begin{lemma}[Restate of Lemma \ref{lemma:init:noise}]
Suppose the conditions of $\calT^{\ast}$ from Theorem \ref{thm:main} hold and $\{\xi_i\}_{i=1}^n$ are i.i.d. $\sigma_s$ subgaussian random variables with variance $\text{Var~}\xi^2\leq C_1\sigma_s^2$ for some absolute constant $C_1>0$. Suppose the sample size 
\begin{align*}
	n &= \sum_{i=1}^{2m-1}n_i\geq C_m\nu^{m+3}\kappa_0^{4m-4}\log^2(\dmax)\bigg((d^\ast)^{1/2}(r^*\rmax\rmin^2)^{1/2} + \dmax r^*\rmax\rmin^2\\
	&+\sum_{k=1}^{m-2}\big((d_k\cdots d_m)^{1/2}(r_1\cdots r_{k-1})^3r_k(r_{k+1}\cdots r_{m-1})^{1/2}(\rmax\rmin^2)^{1/2} 
	+ \dmax(r_1\cdots r_{k-1})^2r^*\rmax\rmin^2\big)\\
	&+(\dmax r^*\rmin\rmax r_{m-1})^{1/2} + \dmax\rmax\rmin r_{m-1} + r^*\rmin\rmax r_{m-1}\bigg), 
\end{align*}
where $C_m>0$ depends only on $m$. Also we assume the signal-to-noise ratio satisfies
\begin{align*}
	&\sigmamin/\sigma_s\geq C_m\max\{\kappa_0^{2m-3}(r^*\rmax\rmin)^{1/2}(\frac{d_1d^*\log(d)}{n})^{1/2},\kappa_0^{m-2}(r^*\rmax)^{1/4}\frac{(d^*)^{3/4}\log^{3/2}(d)}{n^{1/2}}(1+\frac{d_1^{1/2}}{(d^*)^{1/4}})\}\\
	&\hspace{0.5cm}+\sum_{i=2}^{m-1}\Bigg[C_m\kappa_0^{2(m-i)+5}\nu^2(r_i\cdots r_{m-1}\rmax\rmin)^{1/2}\Big(\frac{(d^*)^{1/2}d_i}{n}r_{i-1}\log(\dmax) + \frac{(d^*)^{1/2}d_i\cdots d_m}{n^2}r_{i-1}\log^2(\dmax) \\
	&\hspace{7cm}+ \frac{(d^*d_i)^{1/2}}{\sqrt{n}}\sqrt{r_{i-1}}\log^{1/2}(\dmax)\Big)\\
	&\hspace{0.5cm} + C_m\kappa_0^{m-i+2}\nu(r_i\cdots r_{m-1}\rmax\rmin)^{1/4}\Big(\frac{(d^*d_i)^{1/2}}{n^{1/2}}r_{i-1}^{1/2}\log^{5/4}(\dmax) + \frac{d_i^{1/2}(d_{i+1}\cdots d_m)^{1/4}}{n^{1/2}}r_{i-1}^{3/4}\log^{5/4}(\dmax)\Big)\Bigg]\\
	&\hspace{0.5cm}+C_m\kappa_0\sqrt{\rmax}\Big(\sqrt{\frac{d^*d_mr_{m-1}}{n}\log^{1/2}(\dmax)} + \frac{d^*}{n}\sqrt{\frac{r_{m-1}}{d_1\cdots d_{m-1}}}\kappa_0^2\nu\log(\dmax)\Big),
\end{align*}
Then with probability at least $1 - 10m\dmax^{-m}$, the output of Algorithm \ref{alg:init} satisfies 
$$
\fro{\calT_0-\calT^*} \leq \frac{\sigmamin}{Cm\kappa_0^2\rmax^{1/2}}\quad \text{~and~}\quad \incoh(\calT_0)\leq 2\kappa_0^2\nu.
$$
\end{lemma}

The proof is similar to that of Lemma \ref{thm:init}. We only need to modify the bound of $\op{(\hat T^{\leq i-1}\otimes I)^T(N_i - N_i^*)(\hat T^{\leq i-1}\otimes I)}$. To this end, we need a noisy version of Lemma \ref{lemma:init} and Lemma \ref{lemma:init:1}. For $\op{N_1 - N_1^*}$, we use the Theorem 2 from \cite{xia2017statistically}.
\begin{lemma}[\cite{xia2017statistically}, Theorem 2]\label{lemma:matrix:noise:1}
	There exists absolute constant $C_1,C_2$ such that for any $\alpha \geq 1$, if 
	$$n = n_1 = n_2 \geq C_1\alpha(\sqrt{d^*}\log(\dmax)+\dmax\log^2(\dmax)),$$
	then with probability exceeding $1-\dmax^{-\alpha}$,
	\begin{align*}
		&\op{N_1-N_1^*}\leq C_2\Big((\sigma_s + \|\calT^*\|_{\infty})\fro{\calT^*}\sqrt{\frac{\alpha md_1d^*\log(\dmax)}{n}} \\
		&\hspace{4cm}+\alpha^3(\sigma_s^2 + \|\calT^*\|_{\infty}^2\log^2(\dmax))\frac{(md^*)^{3/2}\log^3(\dmax)}{n}(1 + \frac{d_1}{\sqrt{d^*}})\Big).
	\end{align*}
\end{lemma}
Using this lemma and the sample size condition, under the SNR condition,
$$\sigmamin/\sigma_s\geq C_m\max\{\kappa_0^{2m-3}(r^*\rmax\rmin)^{1/2}(\frac{d_1d^*\log(d)}{n_2})^{1/2},\kappa_0^{m-2}(r^*\rmax)^{1/4}\frac{(d^*)^{3/4}\log^{3/2}(d)}{n_2^{1/2}}(1+\frac{d_1^{1/2}}{(d^*)^{1/4}})\},$$
\eqref{bound:N1-N1*} still holds with probability exceeding $1-\dmax^{-m}$. For $2\leq i\leq m-1$, we need to use the following lemma.
\begin{lemma}\label{lemma:matrix:noise}
	Let $M\in\RR^{p_1\times p_2}$ and $X_i = p_1p_2(M_{\omega_i} + \xi_i)E_{\omega_i}, Y_j = p_1p_2(M_{\omega_j'} + \xi_j')E_{\omega_j'}$, where $\omega_i\in\Omega_1, \omega_j'\in\Omega_2$ are independently and uniformly sampled from $[p_1] \times [p_2]$ and $|\Omega_1| = |\Omega_2| = n$ with $n\leq p_1p_2$, and $\{E_{\omega}\}_{\omega\in[p_1]\times[p_2]}$ is the standard basis for $\RR^{p_1\times p_2}$ and $\xi,\{\xi_i\}_{i=1}^n,\{\xi_j'\}_{j=1}^n$ are i.i.d. $\sigma_s$ subgaussian random variables with variance $\text{Var~}\xi^2\leq C_1\sigma_s^2$ for some absolute constant $C_1>0$. Let $U\in\RR^{p_1\times r}$ be the orthogonal matrix such that $\incoh(U)\leq \sqrt{\mu}$. Then for any $\alpha\geq 1$, with probability exceeding $1- 9p^{-\alpha}$ where $p = \max\{p_1,p_2\}$, we have for some absolute constant $C_2>0$,
%	\begin{align*}
%		&\qquad	\op{\frac{1}{2n^2}\sum_{i,j}(U^TX_iY_j^TU + U^TY_jX_i^TU) - U^TMM^TU} \\
%		&\leq C\alpha^2\log^2(p)\frac{p_1p_2\|M\|_{\infty}^2}{n}\left(\mu rp_2^{1/2} + \frac{\mu r p_2}{n} + (\frac{\mu rn}{\log^3(p)})^{1/2}\right) \\
%		&\quad + C\alpha^2\|M\|_{\infty}\sigma_s\mu r\left(\frac{p_1p_2}{n}\sqrt{p}\log(p)+ \frac{p_1^{3/2}p_2^2}{n^2}\log^2(p) + \frac{p_1p_2}{\sqrt{n}}\log(p)\right)\\
%		&\quad + C\alpha^2\sigma_s^2\frac{p_1p_2}{n}\left(\mu r \log(p)\sqrt{p} + \mu r\log^{3/2}(p)\frac{\sqrt{p_1}p_2}{n}\right).
%	\end{align*}
	\begin{align*}
	&\qquad	\op{\frac{1}{2n^2}\sum_{i,j}(U^TX_iY_j^TU + U^TY_jX_i^TU) - U^TMM^TU} \\
	&\leq C_{\alpha}\|M\|_{\infty}\sigma_s\left(\mu r\frac{p_1p_2}{n}\log(p)+\mu r \frac{p_1p_2^2}{n^2}\log^2(p)+\frac{p_1p_2}{\sqrt{n}}\sqrt{\mu r\log(p)}\right)\\
	&\quad +C_{\alpha}\sigma_s^2\frac{p_1p_2}{n}\left(\mu r \log^{3/2}(p) + \mu r\log^{5/2}(p)\frac{\sqrt{p_2}\sqrt{p_2\vee r}}{n}\right)\\
	&\quad +C_{\alpha}\log^2(p)\frac{p_1p_2\|M\|_{\infty}^2}{n}\left(\mu rp_2^{1/2} + \frac{\mu r p_2}{n} + (\frac{\mu rn}{\log^3(p)})^{1/2}\right).
%	&\leq C_2\alpha^2(\sigma_s^2+\|\calT^*\|_{\infty}^2)\left(\mu r\frac{p_1p_2}{n}\sqrt{p}\log^2(p)+ \mu r\frac{p_1^{3/2}p_2^2}{n^2}\log^2(p) + \mu r\frac{p_1p_2}{\sqrt{n}}\log(p)\right).
	\end{align*}
\end{lemma}
We can use this lemma to bound $\op{(\hat T^{\leq i-1}\otimes I)^T(N_i - N_i^*)(\hat T^{\leq i-1}\otimes I)}$. Using the same notations as in the proof of Lemma \ref{thm:init}, we denote $b_{i-1} = 2\sqrt{r_i}\sigmamin^{-2}\op{(\hat T^{\leq i-1}\otimes I)^T(N_i - N_i^*)(\hat T^{\leq i-1}\otimes I)}$ and $q_{i-1} = 80\sqrt{r_i}\kappa_0^2$.
For any $2\leq i\leq m-1$, under the sample size condition and the SNR condition,
\begin{align*}
	&\sigmamin/\sigma_s \geq C_m\kappa_0^{2(m-i)+5}\nu^2(r_i\cdots r_{m-1}\rmax\rmin)^{1/2}\Big(\frac{(d^*)^{1/2}d_i}{n_{2i}}r_{i-1}\log(\dmax) + \frac{(d^*)^{1/2}d_i\cdots d_m}{n_{2i}^2}r_{i-1}\log^2(\dmax) \\
	&\hspace{7cm}+ \frac{(d^*d_i)^{1/2}}{\sqrt{n_{2i}}}\sqrt{r_{i-1}}\log^{1/2}(\dmax)\Big)\\
	&\hspace{0.5cm} + C_m\kappa_0^{m-i+2}\nu(r_i\cdots r_{m-1}\rmax\rmin)^{1/4}\Big(\frac{(d^*d_i)^{1/2}}{n_{2i}^{1/2}}r_{i-1}^{1/2}\log^{5/4}(\dmax) + \frac{d_i^{1/2}(d_{i+1}\cdots d_m)^{1/4}}{n_{2i}^{1/2}}r_{i-1}^{3/4}\log^{5/4}(\dmax)\Big),
\end{align*}
with probability exceeding $1-9\dmax^{-m}$, we have $q_{m-2}\cdots q_ib_{i-1}\leq \frac{1}{Cm^2\kappa_0^2\sqrt{\rmax}}$ holds. And we obtain \eqref{dc T m-1}, that is,
$$d_c(\hat T^{\leq m-1}, T^{*\leq m-1}) \leq (Cm\kappa_0^2\sqrt{\rmax})^{-1}.$$

We still need to bound $n_{2m-1}^{-1}d^*\fro{T^{\leq m-1 \top}\sum_{i\in\Omega_{2m-1}}\xi_i\calE_{\omega_i}\lr{m-1}}$. Notice this is already bounded in \eqref{bound:1}. We get with probability exceeding $1-\dmax^{-m}$,
\begin{align*}
	\op{T^{\leq m-1 \top}\sum_{i\in\Omega_{2m-1}}\xi_i\calE_{\omega_i}\lr{m-1}} \leq C_m\Big(\sqrt{\frac{d^*d_m}{n_{2m-1}}}\log^{1/2}(\dmax)+\frac{d^*}{n_{2m-1}}\kappa_0^2\nu\sqrt{\frac{r_{m-1}}{d_1\cdots d_{m-1}}}\log(\dmax)\Big)\sigma_s.
\end{align*}
Therefore under the SNR condition,
\begin{align*}
	\sigmamin/\sigma_s \geq C_m\kappa_0\sqrt{\rmax}\Big(\sqrt{\frac{d^*d_mr_{m-1}}{n_{2m-1}}\log^{1/2}(\dmax)} + \frac{d^*}{n}\sqrt{\frac{r_{m-1}}{d_1\cdots d_{m-1}}}\kappa_0^2\nu\log(\dmax)\Big),
\end{align*}
we get \eqref{eq:last step} holds with probability exceeding $1-\dmax^{-m}$. Put all these together and we get with probability exceeding $1-(9(m-2)+2)\dmax^{-m}$, 
$$
\fro{\calT^* - \hat \calT} \leq \frac{\sigmamin}{Cm\kappa_0\sqrt{\rmax}}.
$$
Finally, together with Lemma \ref{lemma:ttsvd+trim}, we conclude that the output $\calT_0$ satisfies 
$$
\fro{\calT_0-\calT^*}\leq \frac{\sigmamin}{Cm\kappa_0^2\sqrt{\rmax}} \text{~and~} \incoh(\calT_0)\leq 2\kappa_0^2\nu.
$$
And this finishes the proof of the lemma.

\subsection{Proof of Lemma \ref{lemma:localconvergence:noise}}
For notion simplicity, we denote $\calP_{\Omega}(\calS) = \sum_{j=1}^n\xi_j\calE_{\omega_j}$.
We only need to slightly modify the proof of Lemma \ref{thm:localconvergence} in Section \ref{sec:contraction}. Notice now we have
\begin{align}\label{eq:Tl+1-T:noise}
	\fro{\calT_{l+1} - \calT^*}^2 &= \fro{\ttsvd_{\r}(\wt\calW_l) - \calT^*}^2
	\overset{\text{Lemma}~\ref{lemma:tt-perturbation}}{\leq} \fro{\wt\calW_l - \calT^*}^2 + \frac{600m\fro{\wt\calW_l - \calT^*}^3}{\sigmamin},
\end{align}
where Lemma \ref{lemma:tt-perturbation} is valid only when $\fro{\wt\calW_l - \calT^*} \leq c_m\sigmamin$ for some sufficient $c_m>0$ depending only on $m$, which we will verify momentarily.
From the way we choose $\zeta_l$, we have $\fro{\wt\calW_l - \calT^*}\leq \fro{\calW_l - \calT^*}$. Now we estimate $\fro{\calW_l - \calT^*}^2$.
\begin{align}\label{eq:W-T:noise}
	\fro{\calW_l - \calT^*}^2 &= \fro{\calT_l - \calT^* - \alpha_l\calP_{\TT_l}\calP_{\Omega}(\calT_l - \calT^*) + \alpha_l\calP_{\TT_l}\calP_{\Omega}(\calS)}^2\notag\\
	&\leq (1+\delta)\fro{\calT_l - \calT^* - \alpha_l\calP_{\TT_l}\calP_{\Omega}(\calT_l - \calT^*)}^2 + (1+\delta^{-1})\alpha_l^2\fro{\calP_{\TT_l}\calP_{\Omega}(\calS)}^2.
\end{align}
Now we need to bound $\fro{\calP_{\TT_l}\calP_{\Omega}(\calS)}$. 
In the following we shall drop the subscript of $\calT_l$ for simplicity. We have 
$\fro{\calP_{\TT}\calP_{\Omega}(\calS)}^2 \leq 2\fro{(\calP_{\TT}-\calP_{\TT^*})\calP_{\Omega}(\calS)}^2 + 2\fro{\calP_{\TT^*}\calP_{\Omega}(\calS)}^2$. And 
\begin{align*}
	\fro{(\calP_{\TT}-\calP_{\TT^*})\calP_{\Omega}(\calS)} = \inp{\calP_{\Omega}(\calS)}{(\calP_{\TT}-\calP_{\TT^*})(\calX_0)}\leq \op{\calP_{\Omega}(\calS)}\cdot\nuc{(\calP_{\TT}-\calP_{\TT^*})(\calX_0)}
\end{align*}
for some $\calX_0$ with $\fro{\calX_0}\leq 1$. Notice $(\calP_{\TT}-\calP_{\TT^*})(\calX_0)$ has TT-rank at most $4(r_1,\cdots,r_{m-1})$ from Lemma \ref{lemma:ttrank}. Applying Lemma \ref{lemma:nuclear norm and frobenius norm} and we obtain 
\begin{align*}
	\nuc{(\calP_{\TT}-\calP_{\TT^*})(\calX_0)} \leq 2^{m-1}\sqrt{r^*}\fro{(\calP_{\TT}-\calP_{\TT^*})(\calX_0)}.
\end{align*}
Denote $\calP_{\TT}(\calX_0) = \sum_{i=1}^m\delta\calX_i$ and $\calP_{\TT^*}(\calX_0) = \sum_{i=1}^m\delta\calX_i^*$ with 
\begin{align*}
	&\delta\calX_i\lr{i} = (T^{\leq i-1}\otimes I)(I - L(T_i)L(T_i)^T)(T^{\leq i-1}\otimes I)^T\calX_0\lr{i}V_{i+1}V_{i+1}^T,\\
	&\delta\calX_i^{*\langle i\rangle} = (T^{*\leq i-1}\otimes I)(I - L(T_i^*)L(T_i^*)^T)(T^{*\leq i-1}\otimes I)^T\calX_0\lr{i}V_{i+1}^*V_{i+1}^{*T}.
\end{align*}
In the following for an orthogonal matrix $U$, we denote $P_U = UU^T$.
Now for $1\leq i\leq m$, we consider $\fro{\delta\calX_i - \delta\calX_i^*}$. Notice
\begin{align*}
	&\hspace{0.5cm} P_{T^{\leq i-1}\otimes I}\calX_0\lr{i}P_{V_{i+1}} - P_{T^{*\leq i-1}\otimes I}\calX_0\lr{i}P_{V_{i+1}^*}\\
	&= (P_{T^{\leq i-1}\otimes I} - P_{T^{*\leq i-1}\otimes I})\calX_0\lr{i}P_{V_{i+1}} + P_{T^{*\leq i-1}\otimes I}\calX_0\lr{i}(P_{V_{i+1}} - P_{V_{i+1}^*}).
\end{align*}
Since $T^{\leq i-1}\otimes I$ is the top $d_ir_{i-1}$ left singular vectors of $\calT\lr{i}$ and $T^{*\leq i-1}\otimes I$ is the top $d_ir_{i-1}$ left singular vectors of $\calT^{*\langle i\rangle}$, we have 
\begin{align*}
	\op{P_{T^{\leq i-1}\otimes I} - P_{T^{*\leq i-1}\otimes I}}\leq \frac{2\fro{\calT-\calT^*}}{\sigmamin},\quad \op{P_{V_{i+1}} - P_{V_{i+1}^*}} \leq \frac{2\fro{\calT-\calT^*}}{\sigmamin}.
\end{align*}
This implies 
\begin{align*}
	\fro{(P_{T^{\leq i-1}\otimes I} - P_{T^{*\leq i-1}\otimes I})\calX_0\lr{i}P_{V_{i+1}} + P_{T^{*\leq i-1}\otimes I}\calX_0\lr{i}(P_{V_{i+1}} - P_{V_{i+1}^*})}^2\leq \frac{16\fro{\calT-\calT^*}^2}{\sigmamin^2}.
\end{align*}
And similarly we can bound 
\begin{align*}
	&\hspace{0.5cm}\Big\|(T^{\leq i-1}\otimes I)L(T_i)L(T_i)^T(T^{\leq i-1}\otimes I)^T\calX_0\lr{i}V_{i+1}V_{i+1}^T \\
	&\hspace{4cm}- (T^{*\leq i-1}\otimes I)L(T_i^*)L(T_i^*)^T(T^{*\leq i-1}\otimes I)^T\calX_0\lr{i}V_{i+1}^*V_{i+1}^{*T}\Big\|_{\rm F}\\
	&=\fro{P_{T^{\leq i}}\calX_0\lr{i}P_{V_{i+1}} - P_{T^{*\leq i}}\calX_0\lr{i}P_{V_{i+1}^*}}\\
	&\leq \frac{16\fro{\calT-\calT^*}^2}{\sigmamin^2}.
\end{align*}
Therefore 
\begin{align}\label{bound:PTX-PT*X}
	\fro{\calP_{\TT}(\calX_0) - \calP_{\TT^*}(\calX_0)}^2 = \fro{\sum_{i=1}^m\delta\calX_i - \delta\calX_i^*}^2\leq m\sum_{i=1}^m\fro{\delta\calX_i - \delta\calX_i^*}^2\leq 32m^2\frac{\fro{\calT-\calT^*}^2}{\sigmamin^2}.
\end{align}
On the other hand, it is proved in (Theorem 1, \cite{xia2017statistically}) that with probability exceeding $1-\dmax^{-m}$,
\begin{align}\label{bound:POmegaS}
	\op{\calP_{\Omega}(\calS)} \leq C_m\sigma_s(\sqrt{\frac{n\dmax}{d^*}}+1)\log^{m+2}(\dmax).
\end{align}
According to \eqref{bound:PTX-PT*X} and \eqref{bound:POmegaS}, we get 
\begin{align*}
	\fro{(\calP_{\TT}-\calP_{\TT^*})\calP_{\Omega}(\calS)} \leq C_m\frac{\sigma_s}{\sigmamin}(\sqrt{\frac{n\dmax}{d^*}}+1)\sqrt{r^*}\log^{m+2}(\dmax)\fro{\calT-\calT^*}.
\end{align*}

Now we bound $\fro{\calP_{\TT^*}\calP_{\Omega}(\calS)}$.
From the definition of $\calP_{\TT^*}$ in 
Section \ref{sec:computation of pta}, we have 
\begin{align*}
	\fro{\calP_{\TT^*}\calP_{\Omega}(\calS)}^2 = \fro{\delta\calA_1}^2 + \ldots + \fro{\delta\calA_m}^2,
\end{align*}
such that for $i\in[m-1]$,
\begin{align*}
\fro{\delta\calA_i}^2 &= \fro{(T^{*\leq i -1}\otimes I)(I - L(T_i^*)L(T_i^*)^T)(T^{*\leq i-1}\otimes I)^T\calP_{\Omega}(\calS)\lr{i}V_{i+1}^*(V_{i+1}^*)^T}^2\\
&\leq \fro{(T^{*\leq i}\otimes I)^T\calP_{\Omega}(\calS)\lr{i}V_{i+1}^*}^2,
\end{align*}
and $\fro{\delta\calA_m}^2 = \fro{(T^{*\leq m-1})^T\calP_{\Omega}(\calS)\lr{m-1}}^2$.
Recall from induction we assume $\incoh(\calT^*) \leq 2\kappa_0^2\nu =: \sqrt{\mu}$. Now we shall use Bernstein's inequality to bound $\op{(T^{*\leq i-1}\otimes I)^T\calP_{\Omega}(\calS)\lr{i}V_{i+1}^*}$. We denote $\calP_{\Omega}(\calS)\lr{i} = \sum_{j=1}^n\xi_j E_j$, where $E_j = \calE_{\omega_j}\lr{i}$. Then we have 
\begin{align*}
	\op{(T^{*\leq i-1}\otimes I)^TE_jV_{i+1}^*} \leq \|V_{i+1}^*\|_{2,\infty}\|T^{*\leq i-1}\otimes I\|_{2,\infty} = \|V_{i+1}^*\|_{2,\infty}\|T^{*\leq i-1}\|_{2,\infty} \leq \sqrt{\frac{r_{i-1}d_ir_i}{d^*}}\mu,
\end{align*}
where the last inequality holds since $\incoh(\calT^*) \leq \sqrt{\mu}$. And thus for all $j\in[n]$, we have 
\begin{align*}
	\big\|\op{\xi_j(T^{*\leq i-1}\otimes I)^TE_jV_{i+1}^*}\big\|_{\psi_2} \leq \sigma_s\sqrt{\frac{r_{i-1}d_ir_i}{d^*}}\mu.
\end{align*}
On the other hand, we have 
$$\op{\EE\xi_j^2(T^{*\leq i-1}\otimes I)^TE_jV_{i+1}^*(V_{i+1}^*)^TE_j^T(T^{*\leq i-1}\otimes I)}\leq \sigma_s^2\frac{r_{i-1}d_ir_i}{d^*}\mu^2$$
and 
$$\op{\EE\xi_j^2(V_{i+1}^*)^TE_j^T(T^{*\leq i-1}\otimes I)(T^{*\leq i-1}\otimes I)^TE_jV_{i+1}^*}\leq \sigma_s^2\frac{r_{i-1}d_ir_i}{d^*}\mu^2.$$
Now from Theorem \ref{thm:concentration}, we have with probability exceeding $1-\dmax^{-m}$,
$$\op{(T^{*\leq i-1}\otimes I)^T\calP_{\Omega}(\calS)\lr{i}V_{i+1}^*} \leq C_m\mu\sigma_s\sqrt{\frac{nr_{i-1}d_ir_i}{d^*}\log(\dmax)},$$
where $C_m>0$ is some absolute constant depending only on $m$.
Therefore we have 
$$\fro{\delta\calA_i}^2 \leq C_m\kappa_0^8\nu^4\sigma_s^2\frac{nr_{i-1}d_ir_i^2}{d^*}\log(\dmax).$$
Similarly, we have $\fro{\delta\calA_m}^2 \leq C_m\kappa_0^4\nu^2\sigma_s^2\frac{nr_{m-1}^2d_m}{d^*}\log(\dmax)$. Combine these bounds and we have 
\begin{align*}
	\fro{\calP_{\TT^*}\calP_{\Omega}(\calS)}^2 &\leq \sum_{i=1}^{m-1}C_m\kappa_0^8\nu^4\sigma_s^2\frac{nr_{i-1}d_ir_i^2}{d^*}\log(\dmax) + C_m\kappa_0^4\nu^2\sigma_s^2\frac{nr_{m-1}^2d_m}{d^*}\log(\dmax)\\
	&\leq C_m\kappa_0^8\nu^4\frac{n}{d^*}\sigma_s^2\rmax\cdot\dof,
\end{align*}
where $\dof = \sum_{i=1}^m r_{i-1}d_ir_i$ is the degree of freedom of a TT-rank $\boldsymbol{r}$ tensor. 
\begin{align}\label{bound:PTl POmega S}
	\alpha_l^2\fro{\calP_{\TT}\calP_{\Omega}(\calS)}^2 &\leq C_m\frac{\sigma_s^2}{\sigmamin^2}(\frac{d^*\dmax}{n}+\frac{(d^*)^2}{n^2})r^*\log^{2m+4}(\dmax)\fro{\calT-\calT^*}^2 + C_m\kappa_0^8\nu^4\frac{d^*}{n}\sigma_s^2\rmax\cdot\dof.
%	&\leq 0.0001\fro{\calT-\calT^*}^2 + C_m\kappa_0^8\nu^4\frac{d^*}{n}\sigma_s^2\rmax\cdot\dof.
\end{align}
Under the SNR condition,
$\sigmamin/\sigma_s \geq C_m\left(\sqrt{\frac{d^*\dmax}{n}} + \frac{d^*}{n}\right)r^*\log^{m+2}(\dmax)$,
we go back to \eqref{eq:W-T:noise}, and we have 
\begin{align}\label{eq:W-T:noise:2}
	\fro{\calW_l - \calT^*}^2 &\leq (1+\delta)\fro{\calT_l - \calT^* - \alpha_l\calP_{\TT_l}\calP_{\Omega}(\calT_l - \calT^*)}^2 + C_{m}(1+\delta^{-1}) \kappa_0^8\nu^4\frac{d^*}{n}\sigma_s^2\rmax\cdot\dof\cdot\log(\dmax)\notag\\
	&\hspace{0.5cm}+10^{-6}\cdot(1+\delta^{-1})\fro{\calT-\calT^*}^2.
\end{align}
Since we already bound $\fro{\calT_l - \calT^* - \alpha_l\calP_{\TT_l}\calP_{\Omega}(\calT_l - \calT^*)}^2$ in Section \ref{sec:contraction}. Now we choose $\delta = 0.001$, together with \eqref{eq:Tl+1-T:noise}, \eqref{eq:W-T:noise:2} and under the SNR condition, we have 
\begin{align*}
	\fro{\calT_{l+1} - \calT^*}^2 \leq 0.98\fro{\calT_l - \calT^*}^2 + C_m\kappa_0^8\nu^4\frac{d^*}{n}\sigma_s^2\rmax\cdot\dof\cdot\log(\dmax).
\end{align*}
After at most $l_{\max} = \Omega\left(\log(C_m\kappa_0^{10}\nu^4\rmax^2\frac{d^*}{n}(\frac{\sigma_s}{\sigmamin})^2\cdot\dof)\right)$ iterations, we have 
$$\fro{\calT_{l_{\max}} - \calT^*}^2 \leq C_m\kappa_0^8\nu^4\frac{d^*}{n}\sigma_s^2\rmax\cdot\dof\cdot\log(\dmax).$$

\section{Technical Lemmas}
In this section, we provide some technical lemmas. Some proofs for the lemmas are placed to the next section.

\begin{theorem}[Wedin's sin$\Theta$ theorem]
	Let $M^*$, $M = M^*+E\in\RR^{p_1\times p_2}$ be two matrices. Let $U^*, U$ be the top $r$ left singular vectors of $M^*,M$ respectively. If $\op{E}< \sigma_r^*-\sigma_{r+1}^*$, then
	$$\min_{R\in\OO_r}\op{UR-R^*} \leq \frac{\sqrt{2}\op{E}}{\sigma_{r}^*-\sigma_{r+1}^*-\op{E}},$$
	where $\sigma_i^*$ is the $i$-th largest singular value of $M^*$, and $\OO_r$ is the set of $r\times r$ orthogonal matrices.
\end{theorem}
%==================================================================
\subsection{Lemmas about linear algebra}
\begin{lemma}\label{lemma:reshape}
	Let $N\in\RR^{d_N\times d_1\cdots d_i}$ and $M\in\RR^{d_{i+2}\ldots d_m\times d_M}$, then we have the following relations:
	\begin{align*}
		&\reshape(N\calT^{\lrangle{i}}, [d_Nd_{i+1}, d_{i+2}\cdots d_{m}]) = (N\otimes I_{d_{i+1}})\calT^{\lrangle{i+1}},\\
		&\reshape\left((N\otimes I_{d_{i+1}})\calT^{\lrangle{i+1}},[d_N, d_{i+1}\cdots d_{m}]\right) = N\calT^{\lrangle{i}},\\
		&\reshape(\calT\lr{i+1}M, [d_1\ldots d_i,d_{i+1}d_M]) = \calT\lr{i}(I_{d_{i+1}}\otimes M),\\
		&\reshape(\calT\lr{i}(I_{d_{i+1}}\otimes M),[d_1\ldots d_{i+1}, d_M]) = \calT\lr{i+1}M.
	\end{align*}
\end{lemma}
\begin{proof}
	We first show the first equation. For all $x_N\in[d_N]$ and $x_j\in [d_j]$, we have
	\begin{align*}
		N\calT^{\lrangle{i}}(x_N;x_{i+1}\cdots x_{m}) &= \sum_{x_{1},\cdots, x_{i}} N(x_N; x_{1},\cdots, x_{i}) \calT(x_1,\cdots,x_m)\no\\
		&=\sum_{x_{1},\cdots, x_{i}, x_{i+1}'} N(x_N; x_{1},\cdots, x_{i})I(x_{i+1},x_{i+1}') \calT(x_1,\cdots,x_{i+1}',\cdots,x_m)\no\\
		&= \sum_{x_{1},\cdots, x_{i}, x_{i+1}'} (N\otimes I_{d_{i+1}})(x_N,x_{i+1}; x_{1},\cdots, x_{i}, x_{i+1}') \calT(x_1,\cdots,x_{i+1}',\cdots,x_m).
	\end{align*}
	So from this we get the desired result. Now the second equation is just the inverse statement of the first one. And the third and fourth equations are similar to the first one.
\end{proof}

\begin{lemma}\label{lemma:twoinf relation}
	For any matrix $A\in\RR^{n\times m}, x\in\RR^{n}$, we have
	$$\|A^Tx\|_{\ell_2}\leq \sqrt{n}\twoinf{A}\|x\|_{\ell_2}.$$
	Let another matrix $B\in\RR^{p\times m}$, then we have 
	$$\|AB^T\|_{\ell_{\infty}} \leq \sqrt{m}\|A\|_{\ell_{\infty}}\twoinf{B}.$$
	When we have another matrix $X\in\RR_{m\times m}$ and $m\leq p,m\leq n$, we have 
	$$
	\|AXB^T\|_{\ell_{\infty}} \leq \op{X}\cdot\twoinf{A}\twoinf{B}.
	$$
\end{lemma}
\begin{proof}
	Set $A = \mat{a_1^T\\ \vdots\\a_n^T}\in\RR^{n\times m}$, then $A^Tx = \sum_{i=1}^n x_i a_i$. And from Cauchy-Schwartz inequality, we have
	$$\|A^Tx\|_{\ell_2}^2 \leq \|x\|_{\ell_2}^2\cdot\sum_{i=1}^n\|a_i\|_{\ell_2}^2\leq n\|x\|_{\ell_2}^2\twoinf{A}^2.$$
	Also set $B = \mat{b_1^T\\ \vdots\\b_p^T}$, then 
	$$\|AB^T\|_{\ell_{\infty}} = \max_{j,k}|a_j^Tb_k|\leq \max_{j,k}\|a_j\|_{\ell_{\infty}}\|b_k\|_{\ell_1}\leq \sqrt{m}\|A\|_{\ell_{\infty}}\max_{k}\|b_k\|_{\ell_2}\leq \sqrt{m}\|A\|_{\ell_{\infty}}\cdot\twoinf{B}.$$
	Finally, let $X = U\Sigma V^T$ be its SVD,
	\begin{align*}
		\|AXB^T\|_{\ell_{\infty}} &\leq \twoinf{AX}\twoinf{B} = \twoinf{AU\Sigma}\twoinf{B}\\
		&\leq  \op{X}\twoinf{AU}\twoinf{B} = \op{X}\twoinf{A}\twoinf{B}.
	\end{align*}
\end{proof}

\begin{lemma}\label{lemma:distofinverse}
	Let $A_1,A_2\in\RR^{m\times n}$ be two rank $r$ matrices with the decomposition $A_1 = U_1\Sigma_1V_1^T, A_2 = U_2\Sigma_2V_2^T$ such that $U_1^TU_1 = U_2^TU_2 = V_1^TV_1 = V_2^TV_2 = I_r$ and $\Sigma_1,\Sigma_2\in\RR^{r\times r}$ be invertible and $U_1,U_2$, $V_1,V_2$ are well aligned in the sense that $d_c(U_1,U_2) = \fro{U_1 - U_2}$ and $d_c(V_1,V_2) = \fro{V_1 - V_2}$. 
	Suppose that $\fro{A_1 - A_2} \leq \frac{1}{10}\min\{\sigma_{\min}(A_1),\sigma_{\min}(A_2)\}$, then we have
	$$\fro{\Sigma_1 - \Sigma_2} \leq 4\sqrt{r}\kappa\op{A_1-A_2},$$
	where $\kappa = \max\{\kappa(A_1),\kappa(A_2)\}$ and $\sigma_{\min}(A)$ is the smallest nonzero singular value of $A$.
\end{lemma}
\begin{proof}
	We can decompose $\Sigma_1 - \Sigma_2$ as follows,
	\begin{align*}
		\fro{\Sigma_1 - \Sigma_2} &\leq \fro{(U_1 - U_2)^T A_1V_1} + \fro{U_2^T(A_1 - A_2)V_1} + \fro{U_2^TA_2(V_1-V_2)}\\
		&\leq (\sqrt{r}+\sqrt{2r}\kappa(A_1) + \sqrt{2r}\kappa(A_2))\op{A_1-A_2}\leq 4\sqrt{r}\kappa\op{A_1-A_2},
	\end{align*}
	where the first inequality in the second line follows Wedin's sin$\Theta$ theorem.
\end{proof}

\begin{lemma}\label{lemma:product of two incoherent matrix}
	Let $A\in\RR^{m\times n}$ and $B\in\RR^{np\times q}$ be such that $A^TA = I_n, B^TB = I_q$ and $\incoh(A)\leq \sqrt{\mu_1}, \incoh(B)\leq\sqrt{\mu_2}$. Then we have $\incoh((A\otimes I_p)B) \leq \sqrt{\mu_1\mu_2 n}$.
\end{lemma}
\begin{proof}
	Consider for any $k\in[mp]$, $\|e_k^T(A\otimes I_p)B\|_{\ell_2}$, denote by $a_k^T$ the $k$-th row of $A\otimes I_p$, and thus $\|a_k\|_{\ell_0}\leq n$. So we have
	$$
	\|e_k^T(A\otimes I_p)B\|_{\ell_2} = \|a_k^TB\|_{\ell_2} \leq \sqrt{\frac{\mu_1n}{m}}\sqrt{n\frac{\mu_2q}{np}} = \sqrt{\mu_1\mu_2n\frac{q}{mp}}.
	$$
\end{proof}

%==================================================================
\subsection{Lemmas concerning TT format}
\begin{lemma}[Facts about TT rank]\label{lemma:ttrank}
	\begin{enumerate}
		\item Let $\calA,\calB\in\RR^{d_1\times \ldots \times d_m}$ be two tensors satisfies $\ranktt(\calA) = (r_1,\ldots,r_{m-1})$ and $\ranktt(\calB) = (s_1,\ldots,s_{m-1})$. Then we have
		$$\ranktt(\calA + \calB) \leq (r_1+s_1,\ldots,r_{m-1}+s_{m-1}).$$
		\item Let $\calT^*\in\mfd$ and $\TT^*$ be the corresponding tangent space. Let $\calT\in\RR^{d_1\times\ldots\times d_m}$ be an arbitrary tensor. Then the rank of $\calP_{\TT^*}(\calT)$ satisfies $\ranktt(\calP_{\TT^*}(\calT)) \leq (2r_1,\ldots,2r_{m-1})$.
	\end{enumerate}
	\begin{proof}
		\begin{enumerate}
			\item It follows from $\rank((\calA + \calB)\lr{i}) \leq \rank(\calA\lr{i}) + \rank(\calB\lr{i}) = r_i + s_i$.
			\item Since $\calP_{\TT^*}(\calT) = \delta\calT_1 + \ldots + \delta\calT_{m}$, where $\delta\calT_i = [T_1^*,\ldots,T_{i-1}^*,X_i, T_{i+1}^*,\ldots,T_m^*]$ and the expression of $X_i$ are give in \eqref{tangent:W}. Now we consider the $i-th$ separation rank of $\calP_{\TT^*}(\calT)$. Notice that for all $j\leq i$, we have $\delta\calT_j\lr{i} = T^{*\leq i}_j T^{*\geq i+1}$, where $T^{*\leq i}_j\in\RR^{d_1\ldots d_i\times r_i}$ is the matrix generated by $T_1^*,\ldots, X_j, \ldots, T_i^*$. And for $j\geq i+1$, we have $\delta\calT_j\lr{i} = T^{*\leq i} T^{*\geq i+1}_j$, where $T^{*\geq i}_j\in\RR^{r_i \times d_{i+1}\ldots d_m}$ is the matrix generated by $T_{i+1}^*,\ldots, X_j, \ldots, T_m^*$. So we have
			$$\calP_{\TT^*}(\calT)\lr{i} = (\sum_{j=1}^i T^{*\leq i}_j)T^{*\geq i+1} + T^{*\leq i}(\sum_{j=i+1}^m T^{*\geq i+1}_j).$$
			And thus $\rank(\calP_{\TT^*}(\calT)\lr{i}) \leq 2r_i$.
		\end{enumerate}
	\end{proof}
\end{lemma}

\begin{lemma}\label{lemma:nuclear norm and frobenius norm}
	Let $\calT\in\mfd$ be a tensor of TT rank $(r_1,\ldots,r_{m-1})$ with a left orthogonal decomposition $\calT = [T_1,\ldots,T_m]$. Then we have
	$$\nuc{\calT} \leq \sqrt{r_1\ldots r_{m-1}}\fro{\calT}.$$
\end{lemma}
\begin{proof}
	Using the alternative definition for the tensor nuclear norm from in  \citet{friedland2014computational}, we have
	\begin{align}\label{eq:alter nuclear norm}
		\nuc{\calT} = \min\{\sum_{i=1}^s |\lambda_i| : \calT = \sum_{i=1}^s \lambda_i u_{1,i}\otimes \ldots \otimes u_{m,i}, \|u_{l,i}\|_{\ell_2} = 1, l\in[m], i\in[s], s\in\NN\}.
	\end{align}
	So we can write $\calT$ as sum of rank one tensors in the following form
	$$\calT =\sum_{k_1 = 1}^{r_1}\cdots\sum_{k_{m-1} = 1}^{r_{m-1}} T_1(\cdot,k_1)\otimes T_2(k_1,\cdot,k_2)\otimes\ldots \otimes T_m(k_{m-1},\cdot),$$
	where each $T_i(k_{i-1},\cdot,k_i)\in\RR^{d_i}$ is a vector. For each fixed $k_1,\ldots, k_{m-1}$, we have 
	$$\fro{T_1(\cdot,k_1)\otimes T_2(k_1,\cdot,k_2)\otimes\ldots \otimes T_m(k_{m-1},\cdot)} = \prod_{i=1}^m \|T_i(k_{i-1},\cdot,k_i)\|_{\ell_2}.$$
	And 
	\begin{align*}
		\big(\sum_{k_1 = 1}^{r_1}\cdots\sum_{k_{m-1} = 1}^{r_{m-1}}\prod_{i=1}^m&\|T_i(k_{i-1},\cdot,k_i)\|_{\ell_2}\big)^2
		\leq
		r_1\ldots r_{m-1}\sum_{k_1 = 1}^{r_1}\cdots\sum_{k_{m-1} = 1}^{r_{m-1}}\prod_{i=1}^m \|T_i(k_{i-1},\cdot,k_i)\|_{\ell_2}^2 \\
		&\overset{(a)}{=} r_1\ldots r_{m-1}\cdot\sum_{k_1,\ldots,k_{m-1}}\prod_{i=2}^m \|T_i(k_{i-1},\cdot,k_i)\|_{\ell_2}^2\\
		&= \ldots\\
		&\overset{(b)}{=} r_1\ldots r_{m-1}\cdot\sum_{k_{m-1}}\|T_m(k_{m-1},\cdot)\|_{\ell_2}^2 = r_1\ldots r_{m-1}\fro{\calT}^2,
	\end{align*}
	where $(a)$ holds since we have $\|T_1(\cdot,k_1)\|_{\ell_2} = 1$ since $L(T_1)L(T_1)^T = I$ and $(b)$ holds for similar reason and we use $L(T_i)^TL(T_i) = I_{r_i}$ for all $i\in[m-1]$.
	Together with $\fro{T_m} = \fro{\calT}$ and \eqref{eq:alter nuclear norm}, we finish the proof.
\end{proof}

\begin{lemma}[Perturbation bound for TT SVD]\label{lemma:tt-perturbation}
	Let $\calT^*\in \mfd$ be the m-way tensor and $\sigmamin := \min_{i=1}^{m-1}\sigma_{\min}(\calT^*)\lr{i}$. And we denote the tensor $\calT = \calT^* +\calD$.
	Then suppose $\sigmamin\geq C_m \fro{\calD}$ for some constant $C_m \geq 500m$ depending only on $m$, we have
	$$\fro{\ttsvd_{\r}(\calT) - \calT^*}^2 \leq \fro{\calD}^2 + \frac{600m\fro{\calD}^3}{\sigmamin}.$$
\end{lemma}
\begin{proof}
	See Section \ref{sec:tt-perturbation}.
\end{proof}

\begin{lemma}\label{lemma:estimationofptperp}
	Let $\calT,\calT^*\in\mfd$ be two TT-rank $\r$ tensors. Suppose we have $8\fro{\calT-\calT^*}\leq \sigmamin$, then we have
	$$\fro{\calP_{\TT}^{\perp}(\calT^*)}\leq \frac{12\sqrt{2}m\fro{\calT-\calT^*}^2}{\sigmamin},$$
	where $\TT$ is the tangent space at the point $\calT$.
\end{lemma}
\begin{proof}
	See Section \ref{sec:lemma:estimationofptperp}.
\end{proof}
Interchanging the roles of $\calT$ and $\calT^*$ in the theorem and using Weyl's inequality and we get the following corollary.
\begin{corollary}\label{coro:pt*perp}
	Under the setting of Lemma \ref{lemma:estimationofptperp}, let $\TT^*$ be the corresponding tangent space at $\calT^*$, then we have
	$$\fro{\calP_{\TT^*}^{\perp}(\calT)} \leq \frac{20m\fro{\calT-\calT^*}^2}{\sigmamin}.$$
\end{corollary}

\begin{lemma}[TTSVD + Trim implies incoherence]\label{lemma:ttsvd+trim}
	Let $\calT^*\in\mfd$ be such that $\spiki(\calT^*) \leq \nu$. Suppose that $\calW$ satisfies $\fro{\calW-\calT^*} \leq \frac{\sigmamin}{600m\sqrt{\rmax}\kappa_0}$, then we have $\incoh\left(\ttsvd_{\r}(\trim_{\zeta}(\calW))\right) \leq 2\kappa_0^2\nu$ if we choose $\zeta = \frac{10\fro{\calW}}{9\sqrt{d^*}}\nu$. Furthermore, 
	$$
	\fro{\ttsvd_{\r}(\trim_{\zeta}(\calW)) - \calT^*} \leq \sqrt{2}\fro{\calW - \calT^*}.
	$$
\end{lemma}
\begin{proof}
	See Section \ref{sec:lemma:ttsvd+trim}.
\end{proof}

%==================================================================
\subsection{Concentration inequalities}
Define the upper bound of the $\psi_{\alpha}$ norm of $\op{Z}$ as 
$$U^{(\alpha)} := \inf\{u>0:\EE\exp(\op{Z}^{\alpha}/u^{\alpha})\leq 2\},\quad \alpha\geq 1$$ and 
$$\sigma^2 = \max\left\{\op{\frac{1}{n}\sum_{i=1}^n\EE Z_i Z_i^T}, \op{\frac{1}{n}\sum_{i=1}^n\EE Z_i^T Z_i}\right\}.$$
\begin{theorem}[\cite{koltchinskii2011nuclear}]\label{thm:concentration}
	Let $Z,Z_1,\ldots,Z_n$ be i.i.d. random matrices with dimensions $d_1\times d_2$ that satisfy $\EE Z= 0$ and $U^{(\alpha)}<\infty$ for some $\alpha\geq1$. Then there exists some absolute constant $C>0$ such that for all $t>0$, with probability at least $1-e^{-t}$, we have 
	$$\op{\frac{1}{n}\sum_{i=1}^n Z_i}\leq C\max\left\{\sigma\sqrt{\frac{t+\log d}{n}},\quad U^{(\alpha)}\left(\log\frac{U^{(\alpha)}}{\sigma}\right)^{1/\alpha}\cdot\frac{t + \log d}{n}\right\},$$
	where $d = d_1+d_2$.
\end{theorem}

First we introduce some operators for the following lemma. For all $x\in[d_1]\times\ldots\times[d_m]$, let the operator $\calP_x: \RR^{d_1\times\ldots\times d_m}\rightarrow\RR^{d_1\times\ldots\times d_m}$ be defined by $\calP_x(\calX) = \inp{\calX}{\calE_x}\calE_x$. And let $\Omega = \{\omega_1,\ldots,\omega_n\}$ be the sampling set and define $\calP_{\Omega}: \RR^{d_1\times\ldots\times d_m}\rightarrow\RR^{d_1\times\ldots\times d_m}$ by 
$\calP_{\Omega}(\calX) = \sum_{i=1}^m \inp{\calX}{\calE_{\omega_i}}\calE_{\omega_i}$.
\begin{lemma}[Concentration inequality]\label{lemma:concentration}
	Suppose $\Omega$ with $|\Omega| = n$ is a set of indices sampled independently and uniformly with replacement. Suppose $\spiki(\calT^*) \leq \nu$ for some $\nu > 0$. When 
	%	$n\geq C_1m(\nu\kappa_0)^4\bar d\bar r^2 \log(d^*)$ 
	$n\geq C_m(\nu\kappa_0)^4\bar d\bar r^2 \log(d^*)$ 
	for some large constant $C_1 >0$, with probability exceeding $1-(\dmax)^{-m}$, we have
	$$\op{\frac{d^*}{n}\calP_{\TT^*}\calP_{\Omega}\calP_{\TT^*} - \calP_{\TT^*}} \leq \frac{1}{2}.$$
\end{lemma}
\begin{proof}
	First we define the operators:
	$$\calZ_x := \frac{d^*}{n}\calP_{\TT^*}\calP_x\calP_{\TT^*} - \frac{1}{n}\calP_{\TT^*}, \quad x\in[d_1]\times\ldots\times[d_m].$$
	Then $\frac{d^*}{n}\calP_{\TT^*}\calP_{\Omega}\calP_{\TT^*} - \calP_{\TT^*} = \sum_{i=1}^n\calZ_{\omega_i}$. We first estimate an upper bound for $\|\calZ_{\omega_i}\|$. First notice
	$$(\calP_{\TT^*}\calP_x\calP_{\TT^*})^2(\calX) = \fro{\calP_{\TT^*}(\calE_x)}^2(\calP_{\TT^*}\calP_x\calP_{\TT^*})(\calX).$$
	This implies $\op{\calP_{\TT^*}\calP_x\calP_{\TT^*}} \leq \fro{\calP_{\TT^*}(\calE_x)}^2$. And
	\begin{align*}
		\op{\calZ_x} \leq \frac{d^*}{n}\op{\calP_{\TT^*}\calP_x\calP_{\TT^*}} +\frac{1}{n} \leq \frac{d^*}{n}\fro{\calP_{\TT^*}(\calE_x)}^2+ \frac{1}{n} \leq \frac{2m(\nu\kappa_0)^4\bar d\bar r^2}{n},
	\end{align*}
	where the last inequality we use $\max_x \fro{\calP_{\TT^*}(\calE_x)}^2 \leq \frac{m(\nu\kappa_0)^4\bar d\bar r^2}{n}$, which comes from Lemma~\ref{lemma:spiki implies incoh}. Now we bound $\op{\EE\sum_{i=1}^n \calZ_{\omega_i}^2}$. Simple calculations show that 
	\begin{align*}
		\EE\sum_{i=1}^n \calZ_{\omega_i}^2 = \frac{d^*}{n}\sum_{x\in[d_1]\times\ldots\times[d_m]}(\calP_{\TT^*}\calP_x\calP_{\TT^*})^2 - \frac{1}{n}\calP_{\TT^*}.
	\end{align*}
	And this implies 
	\begin{align*}
		\op{\EE\sum_{i=1}^n \calZ_{\omega_i}^2} &\leq \frac{d^*}{n}\op{\sum_x(\calP_{\TT^*}\calP_x\calP_{\TT^*})^2} + \frac{1}{n}\no\\
		&\overset{(a)}{\leq} \frac{d^*}{n}\max_x\fro{\calP_{\TT^*}(\calE_{x})}^2\cdot\op{\sum_x \calP_{\TT^*}\calP_x\calP_{\TT^*}} + \frac{1}{n}\no\\
		&\leq \frac{2m(\nu\kappa_0)^4\bar d\bar r^2}{n},
	\end{align*}
	where in $(a)$ we use the fact that $(\calP_{\TT^*}\calP_x\calP_{\TT^*})^2\leq \fro{\calP_{\TT^*}(\calE_x)}^2 \calP_{\TT^*}\calP_x\calP_{\TT^*}$, where $\calA\leq\calB$ means $\calB -\calA$ is an SPD operator. So we conclude using operator Bernstein inequality,
	$$\op{\frac{d^*}{n}\calP_{\TT^*}\calP_{\Omega}\calP_{\TT^*} - \calP_{\TT^*}} \leq C\left(\sqrt{\frac{m(\nu\kappa_0)^4\bar d\bar r^2 \log(d^*)}{n}} + \frac{m(\nu\kappa_0)^4\bar d\bar r^2 \log(d^*)}{n}\right)$$ 
	with probability exceeding $1-(\dmax)^{-m}$. When $n\geq C_m(\nu\kappa_0)^4\bar d\bar r^2 \log(d^*)$, the right hand side is less than $1/2$. And this finishes the proof of the lemma.
\end{proof}

\begin{lemma}[\citealt{xia2017statistically}, Theorem 1]\label{lemma:operatornorm}
	Suppose $\Omega$ is the set sampled uniformly with replacement with size $|\Omega| = n$, then we have with probability exceeding $1-(\dmax)^{-m}$, the following holds
	$$\op{(\calP_{\Omega} - \frac{n}{d^*}\calI)(\calJ)} \leq C_m\left(\sqrt{\frac{n\dmax}{d^*}} + 1\right)\log^{m+2}(\dmax),$$
	where $\calJ\in\RR^{d_1\times\cdots\times d_m}$ is the tensor with all its entries $1$.
\end{lemma}

\begin{lemma}\label{lemma:repetition}
	Let $\Omega=\{\omega_i: i=1,\cdots,n\}$ where $\omega_i$ is independently and uniformly sampled from the set of collections $[d_1]\times\cdots\times [d_m]$.
	With probability at least $1 - \dmax^{-m}$, the maximum number of repetitions of any entry in $\Omega$ is less than $2m\log(\dmax)$.
\end{lemma}
\begin{proof}
	This is a result of standard Chenroff bound. See for example (\citealt{recht2011simpler}, Proposition 5) for a proof.
\end{proof}

%==================================================================
\subsection{Lemmas for initialization}\label{subsec:init}
\begin{lemma}[\citealt{xia2017polynomial}, Theorem 2]\label{lemma:init}
	Let $M\in\RR^{p_1\times p_2}$ and $X_i = p_1p_2\calP_{\omega_i}M$, $Y_j = p_1p_2\calP_{\omega_j'}M$, where $\omega_i\in\Omega_1, \omega_j'\in\Omega_2$ are independently and uniformly sampled from $[p_1] \times [p_2]$ and $|\Omega_1| = |\Omega_2| = n$. Denote by $N = MM^T$ and $\wt N = \frac{1}{2n^2}\sum_{i,j}(X_iY_j^T + Y_jX_i^T)$, then with probability exceeding $1-p^{-\alpha}$ with $p = \max\{p_1,p_2\}$, we have
	$$
	\op{\wt N - N} \leq C\alpha^2\frac{p_1^{3/2}p_2^{3/2}\log(p)}{n}\left[\left(1+\frac{p_1}{p_2}\right)^{1/2} + \frac{p_1^{1/2}p_2^{1/2}}{n} + \left(\frac{n}{p_2\log(p)}\right)^{1/2}\right]\cdot\|M\|_{\infty}^2.
	$$
\end{lemma}

\begin{lemma}\label{lemma:init:1}
	Let $M\in\RR^{p_1\times p_2}$ and $X_i = p_1p_2\calP_{\omega_i}M, Y_j = p_1p_2\calP_{\omega_j'}M$, where $\omega_i\in\Omega_1, \omega_j'\in\Omega_2$ are independently and uniformly sampled from $[p_1] \times [p_2]$ and $|\Omega_1| = |\Omega_2| = n$. Let $U\in\RR^{p_1\times r}$ be the orthogonal matrix such that $\incoh(U)\leq \sqrt{\mu}$. Then with probability exceeding $1- p^{-\alpha}$, we have
	\begin{align*}
	&\qquad	\op{\frac{1}{2n^2}\sum_{i,j}(U^TX_iY_j^TU + U^TY_jX_i^TU) - U^TMM^TU} \\
		&\leq C\alpha^2\log^2(p)\frac{p_1p_2\|M\|_{\infty}^2}{n}\left(\mu rp_2^{1/2} + \frac{\mu r p_2}{n} + (\frac{\mu rn}{\log^3(p)})^{1/2}\right).
	\end{align*}
\end{lemma}
\begin{proof}
	See Section \ref{sec:lemma:init:1}.
\end{proof}
When $\spiki(M)\leq \nu$, we have
\begin{align*}
	&\qquad\op{\frac{1}{2n^2}\sum_{i,j}(U^TX_iY_j^TU + U^TY_jX_i^TU) - U^TMM^TU} \no\\
	&\leq C\alpha^2\log^2(p)\nu^2\fro{M}^2\left(\frac{\mu rp_2^{1/2}}{n} + \frac{\mu r p_2}{n^2} + (\frac{\mu r}{n\log^3(p)})^{1/2}\right)
\end{align*}
holds with probability exceeding $1-p^{-\alpha}$.

\begin{lemma}\label{lemma:sample bound}
	Let $M\in\RR^{p_1\times p_2}$ and $X_i = p_1p_2\calP_{\omega_i}(M)$, where $\omega_i\in\Omega$ is independently and uniformly sampled in $[p_1]\times[p_2]$ and $|\Omega| = n$. Let $U\in\RR^{p_1\times r}$ be the orthogonal matrix such that $\incoh(U)\leq \sqrt{\mu}$. Then with probability exceeding $1-p^{-\alpha}$, we have 
	$$
	\op{U^T(\frac{p_1p_2}{n}\calP_{\Omega}(M) - M)} \leq C\alpha\left(\frac{\sqrt{p_1}p_2\sqrt{\mu r}\|M\|_{\infty}\log(p)}{n} + \sqrt{\frac{p_1p_2\|M\|_{\infty}^2(\mu r\vee p_2)\log(p)}{n}}\right).
	$$
\end{lemma}
\begin{proof}
	See appendix \ref{sec:lemma:sample bound}.
\end{proof}

\begin{lemma}[\citealt{keshavan2010matrix}, Remark 6.2]\label{lemma:trim}
Let $U,X\in\RR^{p\times r}$ be orthogonal and $\incoh(U)\leq\sqrt{\mu_0}$ and $d_p(U,X)\leq \delta\leq\frac{1}{16\pi}$. Then $\bar X$ satisfies $\incoh(\hat X)\leq \sqrt{3\mu_0}$ and $d_p(\hat X, U)\leq 4\pi\delta$, where 
$$
\overline{X}^i = \frac{X^i}{\|X^i\|_{\ell_2}}\cdot\min\{\|X^i\|_{\ell_2},\sqrt{\frac{\mu r}{p}}\},\quad \hat X = \overline X (\overline{X}^T\overline{X})^{-1/2}.
$$
\end{lemma}

\section{Proofs of technical lemmas}\label{sec:proofsoftechnicallemmas}
In this section, we provide the proofs for the technical lemmas.
\subsection{Proof of Lemma~\ref{lemma:spiki implies incoh}}
Since $(\calT^*)\lr{i} = T^{*\leq i}\Lambda_{i+1}^*V_{i+1}^{*T}$, we have $T^{*\leq i} = (\calT^*)\lr{i}V_{i+1}^{*}(\Lambda_{i+1}^{*})^{-1}$. And thus for any $k\in[d_1\ldots d_i]$,
\begin{align*}
	\|e_k^TT^{*\leq i}\|_{\ell_2} = \|e_k^T(\calT^*)\lr{i}V_{i+1}^{*}(\Lambda_{i+1}^{*})^{-1}\|_{\ell_2}\leq \frac{1}{\sigma_{\min}(\Lambda_{i+1}^*)}\|e_k^T(\calT^*)\lr{i}\|_{\ell_2} \leq \frac{\sqrt{d_{i+1}\ldots d_m}}{\sigma_{\min}(\Lambda_{i+1}^*)}\|\calT^*\|_{\ell_{\infty}}.
\end{align*}
And the spikiness condition implies $\|\calT^*\|_{\ell_{\infty}} \leq \frac{\nu}{\sqrt{d^*}}\fro{\calT^*}$, and together with $\fro{\calT^*} \leq \sqrt{r_i}\sigma_1(\Lambda_{i+1}^*)$, we obtain
$$\|e_k^TT^{*\leq i}\|_{\ell_2}\frac{\sqrt{d_1\ldots d_i}}{\sqrt{r_i}} \leq \nu\kappa_0.$$
The bound for $\|e_k^TV_{i+1}^*\|_{\ell_2}$ can be similarly derived. And this finishes the proof.
\subsection{Proof of Lemma~\ref{lemma:beta_n}}
For simplicity denote the random tensor which is uniformly distributed in $\{\calE_{\omega}\}_{\omega\in[d^*]}$ by $\calE$ and let $\calE_1,\ldots,\calE_n$ be $n$ i.i.d. copies of $\calE$. Also define $\delta_{1,j} = 2^j\delta_1^-, j = 0,\ldots, \lceil \log_2(\frac{\delta_1^+}{\delta_1^-})\rceil =: j_0$, $\delta_{2,k} = 2^k\delta_2^-, k = 0,\ldots, \lceil \log_2(\frac{\delta_2^+}{\delta_2^-})\rceil =: k_0$. For each $j,k$, we derive an upper bound for $\beta_n(\gamma_1,\gamma_2)$ with $\gamma_1 = \delta_{1,j}, \gamma_2 = \delta_{2,k}$. 

We observe that 
$$\sup_{\calA\in\KK_{\gamma_1,\gamma_2}}\left|\inp{\calA}{\calE}^2 - \EE\inp{\calA}{\calE}^2\right|\leq \gamma_1^2,$$
and 
$$
\sup_{\calA\in\KK_{\gamma_1,\gamma_2}} \text{Var}\inp{\calA}{\calE}^2 \leq \sup_{\calA\in\KK_{\gamma_1,\gamma_2}} \inp{\calA}{\calE}^4\leq \frac{\gamma_1^2\fro{\calA}^2}{d^*}\leq \frac{\gamma_1^2}{d^*}.
$$
Now apply Bousquet's version of Talagrand concentration inequality (see \citealt{gine2021mathematical}, Theorem 3.3.9), and we get with probability at least $1-e^{-t}$ for any $t>0$,
$$
\beta_n(\gamma_1,\gamma_2)\leq 2\EE\beta_n(\gamma_1,\gamma_2) + 2\gamma_1\sqrt{\frac{nt}{d^*}} + 2\gamma_1^2t.
$$
Using symmetric inequality, we have
$$
\EE\beta_n(\gamma_1,\gamma_2) \leq 2n\EE\sup_{\calA\in\KK_{\gamma_1,\gamma_2}}\left|\frac{1}{n}\sum_{i=1}^n\epsilon_i\inp{\calA}{\calE_i}^2\right|,
$$
where $\epsilon_1,\ldots,\epsilon_n$ are i.i.d. Rademacher random variables. Since $|\inp{\calA}{\calE}|\leq \gamma_1$, we have from contraction inequality 
$$\EE\beta_n(\gamma_1,\gamma_2) \leq 4n\gamma_1\EE\sup_{\calA\in\KK_{\gamma_1,\gamma_2}}\left|\frac{1}{n}\sum_{i=1}^n\epsilon_i\inp{\calA}{\calE_i}\right|.$$
Now we denote $\calY = \frac{1}{n}\sum_{i=1}^n\epsilon_i\calE_i$. Then we have
$$
\EE\sup_{\calA\in\KK_{\gamma_1,\gamma_2}}\left|\frac{1}{n}\sum_{i=1}^n\epsilon_i\inp{\calA}{\calE_i}\right| \leq \EE\sup_{\calA\in\KK_{\gamma_1,\gamma_2}}\op{\calY}\nuc{\calA}\leq \gamma_2\EE\op{\calY}.
$$
The estimation for $\op{\calY}$ derived in the Theorem 1 of \citet{xia2017statistically} gives
$$
\EE\op{\calY} \leq C_m\left(\sqrt{\frac{\dmax}{nd^*}} + \frac{1}{n}\right)\log^{m+2}(\dmax).
$$
As a result, with probability exceeding $1-e^{-t}$, we have
\begin{align}\label{eq:est beta_n}
	\beta_n(\gamma_1,\gamma_2) \leq C_m\gamma_1\gamma_2\left(\sqrt{\frac{n\dmax}{d^*}} + 1\right)\log^{m+2}(\dmax)+ 2\gamma_1\sqrt{\frac{nt}{d^*}} + 2\gamma_1^2t.
\end{align}

Now we take union bound and we get with probability exceeding 
$1-2\log_2\left(\frac{\delta_1^+}{\delta_1^-}\right) \log_2\left(\frac{\delta_2^+}{\delta_2^-}\right)e^{-t}$,
and for all $\gamma_1\in\{\delta_{1,0},\ldots,\delta_{1,j_0}\}, \gamma_2\in\{\delta_{2,0},\ldots,\delta_{2,k_0}\}$, \eqref{eq:est beta_n} holds. Now we consider arbitrary $\gamma_1\in[\delta_1^-, \delta_1^+], \gamma_2\in[\delta_2^-, \delta_2^+]$. Then there exists some $j,k$, such that $\gamma_1\in[\delta_{1,j-1},\delta_{1,j}], \gamma_2\in[\delta_{2,k-1}, \delta_{2,k}]$. Together with the fact that $\delta_{1,j} \leq 2\gamma_1$ and $\delta_{2,k} \leq 2\gamma_2$ we get
$$\beta_n(\gamma_1, \gamma_2) \leq \beta_n(\delta_{1,j},\delta_{2,k}) \leq C_m\gamma_1\gamma_2\left(\sqrt{\frac{n\dmax}{d^*}} + 1\right)\log^{m+2}(\dmax)+ 4\gamma_1\sqrt{\frac{nt}{d^*}} + 8\gamma_1^2t.$$
This finishes the proof of the lemma.

\subsection{Proof of Lemma~\ref{lemma:inp:nodependence}}
Let $\omega\in[d_1]\times \ldots\times [d_m]$, then 
\begin{align*}
	\fro{\calP^{(i)}\calE_\omega}^2 &\leq \left\{
	\begin{array}{rlc}
		&\frac{\mu r_1 d_1}{d^*}, &i=1\\
		&\frac{\mu^2r_{i-1}r_{i}d_i}{d^*},&2\leq i \leq m-1\\
		&\frac{\mu r_{m-1}d_m}{d^*}, &i = m
	\end{array}
	\right.
	\leq\frac{\mu^2\rmax^2\dmax}{d^*}.
\end{align*}
Set $\calZ_j = \calP^{(i)}(\calP_{\omega_j}-\frac{1}{d^*}\calI)\calP^{(i)}$ for all $j\in[n]$, then $\calP^{(i)}(\calP_{\Omega}-\frac{n}{d^*}\calI)\calP^{(i)} = \sum_{j = 1}^n \calZ_j$. First for arbitrary $\calX$, we have
$$
(\calP^{(i)}\calP_{\omega}\calP^{(i)})^2\calX = \fro{\calP^{(i)}\calE_{\omega}}^2\cdot \calP^{(i)}\calP_{\omega}\calP^{(i)}\calX.
$$
Therefore $\op{\calP^{(i)}\calP_{\omega}\calP^{(i)}} \leq \fro{\calP^{(i)}\calE_{\omega}}^2$ and this implies that 
\begin{align*}
	\op{\calZ_j} \leq \max_{\omega}\fro{\calP^{(i)}\calE_{\omega}}^2 + \frac{1}{d^*} \leq \frac{2\mu^2\rmax^2\dmax}{d^*}.
\end{align*}
On the other hand, since $\calZ_j$ is an symmetric operator, we consider $\EE\sum_{j=1}^n\calZ_j^2$.
\begin{align*}
	\op{\EE\sum_{j=1}^n\calZ_j^2} &= n\op{\frac{1}{d^*}\sum_{\omega}(\calP^{(i)}\calP_{\omega}\calP^{(i)})^2 - \frac{1}{(d^*)^2}\calP^{(i)}}\no\\
	&\leq \frac{n}{d^*}\max_{\omega}\fro{\calP^{(i)}\calE_{\omega}}^2 + \frac{n}{(d^*)^2} \leq \frac{2\mu^2\rmax^2n\dmax}{(d^*)^2}.
\end{align*}

Now using operator Bernstein inequality, with probability exceeding $1-\dmax^{-m}$, 
$$\op{\calP^{(i)}(\calP_{\Omega}-\frac{n}{d^*}\calI)\calP^{(i)}} \leq C_m\left(\frac{\mu^2\rmax^2\dmax\log(\dmax)}{d^*} + \sqrt{\frac{\mu^2\rmax^2\dmax n\log(\dmax)}{(d^*)^2}}\right) \leq C_m\sqrt{\frac{\mu^2\rmax^2\dmax n\log(\dmax)}{(d^*)^2}},$$
where the last inequality holds as long as $n\geq C\mu^2\rmax^2\dmax\log(\dmax)$. 
\subsection{Proof of Lemma~\ref{lemma:inp:lowrank}}
Notice that 
$$\inp{(\calP_{\Omega} - \frac{n}{d^*}\calI)\calA}{\calB} = \inp{(\calP_{\Omega} - \frac{n}{d^*}\calI)\calJ}{\calA\odot\calB}\leq \op{(\calP_{\Omega} - \frac{n}{d^*}\calI)\calJ}\cdot\nuc{\calA\odot\calB}.$$
where $\calJ$ is the tensor with all its entries one. And from the definition of nuclear norm in \eqref{eq:alter nuclear norm}, we have
\begin{align*}
	\nuc{\calA\odot\calB}^2 &\leq \left(\sum_{\substack{k_1,\ldots,k_{m-1}\\ k_1',\ldots,k_{m-1}'}}\|A_1(:,k_1)\odot B_1(:,k_1')\|_{\ell_2}\ldots\|A_m(k_{m-1},:)\odot B_1(k_{m-1}',:)\|_{\ell_2} \right)^2\no\\
	&\leq \sum_{x_1}\fro{A_1(x_1,:)}^2\fro{B_1(x_1,:)}^2\ldots\sum_{x_m}\fro{A_m(:,x_m)}^2\fro{B_m(:,x_m)}^2,
\end{align*}
where the last inequality comes from using Cauchy-Schwartz inequality $m-1$ times. 
Since $\bcalE_3$ holds and 
$$
\sum_{x_i}\fro{A_i(:,x_i,:)}^2\fro{B_i(:,x_i,:)}^2 \leq
\max_{x_i}\fro{A_i(:,x_i,:)}^2\cdot\fro{B_i}^2\wedge\max_{x_i}\fro{B_i(:,x_i,:)}^2\cdot\fro{A_i}^2
$$
we get the desired result.

\subsection{Proof of Lemma \ref{lemma:estimate T_i}}
Consider for all $i\in[m-1]$, notice that $L(T_i) = (T^{\leq i-1}\otimes I)^TT^{\leq i}$. We have 
$$L(T_i)(k_{i-1},x_i;k_i) = \sum_{y_{i-1}\in[d_1\cdots d_{i-1}]}T^{\leq i-1}(y_{i-1},k_i)T^{\leq i}(y_{i-1},x_i;k_i).$$ This implies 
$$
\fro{T_i(:,x_i,:)}^2 = \fro{(T^{\leq i-1})^T\cdot T^{\leq i}(:,x_i,:)}^2 \leq \fro{T^{\leq i}(:,x_i,:)}^2 \leq \frac{\mu r_i}{d_i},
$$
where $T^{\leq i}(:,x_i,:)$ is viewed as a matrix of size $d_1\cdots d_{i-1}\times r_i$ by extracting $d_1\cdots d_{i-1}$ rows of $T^{\leq i}$. And since the decomposition is left orthogonal, we have $\fro{T_i} = \fro{L(T_i)} = r_i$.

When $i = m$, since $T_m = \Lambda_m R_{m}^T$ for some $\Lambda_m\in\RR^{r_{m-1}\times r_{m-1}}$ invertible and orthogonal $R_m$ with $\incoh(R_m)\leq \sqrt{\mu}$. So we have 
$$
\max_{x_m}\fro{T_m(:,x_m)}^2\leq \sigma_{\max}^2(\calT)\frac{\mu r_{m-1}}{d_m}.
$$
And $\fro{\calT} = \fro{\calT\lr{m-1}} = \fro{T^{\leq m-1}T_m} = \fro{T_m}$ since $T^{\leq m-1}$ is orthogonal.

 \subsection{Proof of Lemma \ref{lemma:matrix:noise}}
We first introduce some notations. Denote $\calP_{\Omega}(M) = \sum_{\omega_i\in\Omega}M_{\omega_i}E_{\omega_i}$ and $\calP_{\Omega'}(M) = \sum_{\omega_j'\in\Omega}M_{\omega_j'}E_{\omega_j'}$, and their mean zero version $\Delta = \frac{p_1p_2}{n}\calP_{\Omega}(M) - M$, $\Delta' = \frac{p_1p_2}{n}\calP_{\Omega'}(M) - M$ denote also $\Xi = \frac{p_1p_2}{n}\sum_{i = 1}^n\xi_iU^TE_{\omega_i}$ and $\Xi' = \frac{p_1p_2}{n}\sum_{j = 1}^n\xi_i'U^TE_{\omega_j'}$. Also we will encounter norms $\|A\|_{\infty} = \max_{i,j}|A_{ij}|$ and $\|\cdot\|_{k,\infty}$ for $k=1,2$ defined as follows
$$\|A\|_{k,\infty}^k:= \max_{i\in[p_1]}\sum_{j=1}^{p_2}|A_{ij}|^k.$$

Notice that
\begin{align}
	&\quad\op{\frac{1}{2n^2}\sum_{i,j}(U^TX_iY_j^TU + U^TY_jX_i^TU) - U^TMM^TU} \no\\
	&= \op{\frac{(p_1p_2)^2}{2n^2}\sum_{i,j}(M_{\omega_i}M_{\omega_j'} + M_{\omega_i}\xi_j' + \xi_iM_{\omega_j'} + \xi_i\xi_j')U^T(E_{\omega_i}E_{\omega_j'}^T + E_{\omega_j'}E_{\omega_i}^T)U - U^TMM^TU} \no\\
	&\leq \underbrace{\op{\frac{(p_1p_2)^2}{2n^2}(U^T\calP_{\Omega}(M)\calP_{\Omega'}(M)^TU + U^T\calP_{\Omega'}(M)\calP_{\Omega}(M)^TU)- U^TMM^TU}}_{\beta_1} \no\\
	&\quad + \underbrace{\frac{p_1p_2}{2n}\op{U^T\calP_{\Omega}(M)\Xi'^TU +U^T\Xi'\calP_{\Omega}(M)^TU }}_{\beta_2} + \underbrace{\frac{p_1p_2}{2n}\op{U^T\Xi\calP_{\Omega'}(M)^TU +U^T\calP_{\Omega'}(M)\Xi^TU }}_{\beta_3}\no\\
	&\quad + \underbrace{\op{\Xi\Xi'^T+\Xi'\Xi^T}}_{\beta_4}.
\end{align}
Notice $\beta_1$ can be controlled using Lemma \ref{lemma:init:1}. And $\beta_2$ and $\beta_3$ can be bounded similarly. 

We begin with several preliminary facts which can be easily proved by matrix Bernstein inequalities (Theorem \ref{thm:concentration}). We have with probability exceeding $1 - 2p^{-\alpha}$, where $p = \max\{p_1,p_2\}$ and for any $\alpha\geq 1$,
\begin{align}\label{bound:1}
	\max\{\op{\Xi},\op{\Xi'}\} \leq C\alpha\left(\sqrt{\frac{p_1p_2(p_2\vee r)}{n}\log p}\sigma_s + \frac{p_1p_2}{n}\sqrt{\frac{\mu r}{p_1}}\log p \sigma_s\right).
\end{align}
Meanwhile, using Theorem \ref{thm:concentration}, we also have with probability exceeding $1 - 2p^{-\alpha}$,
\begin{align}\label{bound:2}
\max\{\op{U^TM\Xi^T},\op{U^TM\Xi'^T}\} \leq C\alpha\left(\sqrt{\frac{\mu rp_1^2p_2^2}{n}\log p} + \frac{p_1p_2}{n}\mu r\log p\right)\|M\|_{\infty}\sigma_s.
\end{align}

Notice 
$\beta_2 \leq \op{U^T\Delta\Xi'^T} + \op{U^TM\Xi'^T}$.
And $U^T\Delta\Xi'^T = \frac{p_1p_2}{n}\sum_j U^T\Delta\xi_j'E_{\omega_j'}^TU$. Let $\Omega = \{(i_k,j_k)\}_{k=1}^n$, then using Chernoff bound, with probability exceeding $1-n^{-\alpha}$,
\begin{align*}
	\max_{l\in[p_2]}\sum_{k = 1}^n1(j_k = l) \leq (3\alpha+7)(\frac{n}{p_2} + \log(p)).
\end{align*}
Therefore,
\begin{align*}
	\max_{l\in[p_2]}\|U^T\calP_{\Omega}(M) e_l\|_{\ell_2} = \max_{l\in[p_2]}\|U^T\sum_{k=1}^nM_{i_k,j_k}e_{i_k}e_{j_k}^T e_l\|_{\ell_2}\leq (3\alpha+7)(\frac{n}{p_2} + \log(p))\|M\|_{\infty}\sqrt{\frac{\mu r}{p_1}}.
\end{align*}
Meanwhile,
\begin{align*}
	\max_{l\in[p_2]}\|U^TMe_l\|_{\ell_2} = \max_{l\in[p_2]}\|U^T\sum_{i,j}M_{i,j}e_ie_j^Te_l\|_{\ell_2}=\max_{l\in[p_2]}\|U^T\sum_{i}M_{i,l}e_i\|_{\ell_2}\leq p_1\|M\|_{\infty}\sqrt{\frac{\mu r}{p_1}}.
\end{align*}
Using these two bounds on $\|M^TU\|_{2,\infty}$ and $\|\calP_{\Omega}(M)^TU\|_{2,\infty}$, we get
\begin{align*}
	\op{U^T\Delta\xi_j'E_{\omega_j'}^TU} &\leq \|U\|_{2,\infty}\cdot \|(\frac{p_1p_2}{n}\calP_{\Omega}(M) - M)^TU\|_{2,\infty}\\
	&\leq \|U\|_{2,\infty}\cdot (\|M^TU\|_{2,\infty}+\frac{p_1p_2}{n}\|\calP_{\Omega}(M)^TU\|_{2,\infty})\\
%	&\leq \frac{p_1p_2}{n}\|U\|_{2,\infty}\cdot \max_{l\in[p_2]}\|U^T\calP_{\Omega}(M) e_l\|_{\ell_2} +\|U\|_{2,\infty}\cdot \max_{l\in[p_2]}\|U^TMe_l\|_{\ell_2}\\
%	&\leq \frac{p_1p_2}{n}\sqrt{\frac{\mu r}{p_1}} \sqrt{\mu r}
%	(3\alpha+7)(\frac{n}{p_2}+\log p)\|M\|_{\infty} + \sqrt{\frac{\mu r}{p_1}}\sqrt{\mu r}\sqrt{p_1}\|M\|_{\infty}\\
	&\leq (3\alpha+8)(1+\frac{p_2}{n}\log p)\mu r\|M\|_{\infty}.
\end{align*}
where in the fourth line we use Lemma \ref{lemma:twoinf relation}, $\|\cdot\|_{\ell_2}\leq\|\cdot\|_{\ell_1}$and Chernoff bound. Therefore we obtain 
$$\left\|\op{\xi_j'U^T(\frac{p_1p_2}{n}\calP_{\Omega}(M) - M)\xi_j'E_{\omega_j'}^TU}\right\|_{\psi_2} \leq (3\alpha+8)(1+\frac{p_2}{n}\log p)\mu r\|M\|_{\infty}\sigma_s.$$
On the other hand, notice with probability $1-p^{-\alpha}$,
\begin{align*}
	&\hspace{0.5cm}\max\{\op{\EE \xi_j'^2U^T\Delta E_{\omega_j'}^TUU^TE_{\omega_j'}\Delta^TU}, \op{\EE \xi_j'^2U^TE_{\omega_j'}\Delta^TUU^T\Delta E_{\omega_j'}^TU}\}\\
	&\leq \frac{r\sigma_s^2}{p_1p_2}\op{U^T\Delta}^2\\
	&\leq C\alpha^2\frac{r\sigma_s^2}{p_1p_2}\left(\frac{p_1p_2^2\mu r\|M\|_{\infty}^2\log^2(p)}{n^2} + \frac{p_1p_2\|M\|_{\infty}^2(\mu r\vee p_2)\log(p)}{n}\right),
\end{align*}
where the last inequality follows from Lemma \ref{lemma:sample bound}. Therefore when $n\leq p_1p_2$, we have with probability exceeding $1-p^{-\alpha}$,
$$\op{U^T\Delta\Xi'^T} \leq C_{\alpha}\frac{p_1p_2}{n}\|M\|_{\infty}\sigma_s\left(\mu r\log(p) + \frac{p_2}{n}\mu r\log^2(p)\right).$$
Together with \eqref{bound:2}, we see with probability exceeding $1-3p^{-\alpha}-n^{-\alpha}$, 
\begin{align}\label{eq:beta2}
	\beta_2 \leq C_{\alpha}\|M\|_{\infty}\sigma_s\left(\mu r\frac{p_1p_2}{n}\log(p)+\mu r \frac{p_1p_2^2}{n^2}\log^2(p)+\frac{p_1p_2}{\sqrt{n}}\sqrt{\mu r\log(p)}\right).
\end{align}
And we can bound $\beta_3$ similarly.

Now we consider $\beta_4$. For any fixed $\Xi'$, we need to control $\frac{p_1p_2}{n}\op{\sum_i\xi_iU^TE_{\omega_i}\Xi'^T}$. 
Notice using Chernoff bound and the fact that $\max_i|\xi_i|\leq C\alpha \sigma_s\log^{1/2}(n)$ with probability exceeding $1-n^{-\alpha}$, we have with probability exceeding $1-2n^{-\alpha}$,
\begin{align*}
	\|\Xi'^T\|_{2,\infty} = \frac{p_1p_2}{n}\max_{l\in[p_2]}\|\sum_{j=1}^n\xi_j' U^TE_{\omega_j'}e_l\|_{\ell_2} \leq C\alpha(3\alpha+7)\frac{p_1p_2}{n}(\frac{n}{p_2}+\log(p))\sigma_s\log^{1/2}(n)\sqrt{\frac{\mu r}{p_1}}.
\end{align*}
Therefore
\begin{align*}
	\op{U^TE_{\omega_i}\Xi'^T} &\leq \|U\|_{2,\infty}\cdot \|\Xi'^T\|_{2,\infty} \leq  C\alpha\frac{p_1p_2}{n}(\frac{n}{p_2}+\log(p))\sigma_s\log^{1/2}(n)\frac{\mu r}{p_1}
\end{align*}
where in the last inequality of the first line we use Lemma \ref{lemma:twoinf relation}.
And thus
$$\left\|\op{\xi_iU^TE_{\omega_i}\Xi'^T} \right\|_{\psi_2} \leq C_{\alpha}\frac{\mu r}{p_1}\sigma_s^2\log^{1/2}(n)\frac{p_1p_2}{n}(\frac{n}{p_2}+\log(p)).$$
On the other hand, 
\begin{align*}
	\frac{1}{n}\op{\sum_{i}\EE \xi_i^2U^TE_{\omega_i}\Xi'^T\Xi'E_{\omega_i}^TU} = \frac{\sigma_s^2}{p_1p_2}\op{\sum_{k,l}U^Te_ke_l^T\Xi'^T\Xi'e_le_k^TU}  = \frac{\sigma_s^2}{p_1p_2}\op{\fro{\Xi'}^2\cdot I}\leq \frac{\sigma_s^2r}{p_1p_2}\op{\Xi'}^2,
\end{align*}
and similarly $\frac{1}{n}\op{\sum_{i}\EE \xi_i^2\Xi'E_{\omega_i}^TUU^TE_{\omega_i}\Xi'^T} \leq \frac{\sigma_s^2r}{p_1p_2}\op{\Xi'}^2$. Now from \eqref{bound:2} and Theorem \ref{thm:concentration}, when $n\leq p_1p_2$, we have with probability exceeding $1-2n^{-\alpha}-2p^{-\alpha}$,
\begin{align}\label{eq:beta4}
	\beta_4 \leq C_{\alpha}\sigma_s^2\frac{p_1p_2}{n}\left(\mu r \log^{3/2}(p) + \mu r\log^{5/2}(p)\frac{\sqrt{p_2}\sqrt{p_2\vee r}}{n}\right).
\end{align}
We get the desired result combining \eqref{eq:beta2}, \eqref{eq:beta4} and Lemma \ref{lemma:init:1}.

\subsection{Proof of Lemma~\ref{lemma:tt-perturbation}}\label{sec:tt-perturbation}
Denote $\hat \calT = \ttsvd_{\r}(\calT) = [\hat T_1,\ldots,\hat T_m]$. Let $\calT^* = [T_1^{**},\ldots,T_m^{**}]$ be a TT decomposition that is left orthogonal and $R_i = \arg\min_{R\in\OO_{r_i}}\fro{(T^{**})^{\leq i}R - \hat T^{\leq i}}$. Now we set $L(T_i^*) = (R_{i-1}\otimes I)^TL(T_i^{**})R_i$. Then $\calT^* = [T_1^*,\ldots, T_m^*]$ is another left orthogonal decomposition so that $T^{*\leq i}$ and $\hat T^{\leq i}$ are well aligned in terms of chordal distance. Also, let $(\calT^*)\lr{i} = T^{*\leq i}\Lambda_{i+1}V_{i+1}^{*T}$ be such that $V_{i+1}^{*T}V_{i+1}^{*} = I_{r_i}$ and $\Lambda_{i+1}\in\RR^{r_i\times r_i}$ be invertible.

From Algorithm \ref{alg:svd}, we have $(\hat T^{\leq m-1})^T\calT\lr{m-1} = \hat T_m$. Now using the notations $\hat\calP_{i} = \hat T^{\leq m-1}\hat T^{\leq m-1 T}$ and $\calP_{i}^* = T^{*\leq i}T^{*\leq i T}$, we have 
\begin{align}\label{m>=3:main}
	\fro{\ttsvd_{\r}(\calT) &- \calT^*} = \fro{\hat \calT^{\lrangle{m-1}} - (\calT^*)^{\lrangle{m-1}}} = \fro{(\hat\calP_{m-1}-I)(\calT^*)^{\lrangle{m-1}} + \hat\calP_{m-1}\calD^{\lrangle{m-1}}}\no\\
	&\overset{(a)}{=} \fro{(\hat\calP_{m-1} - \calP_{m-1}^*)(\calT^*)^{\lrangle{m-1}} + \calP_{m-1}^*\calD^{\lrangle{m-1}}+ (\hat\calP_{m-1}-\calP_{m-1}^*)\calD^{\lrangle{m-1}}}\no\\
	&\overset{(b)}{=}\fro{\underbrace{(I-\calP^*_{m-1})\Delta_{m-1}V_{m}^*V_{m}^{*T} + \calP_{m-1}^*\calD^{\lrangle{m-1}}}_{(\RN{1}.1)}\no\\
		&\qquad+ \underbrace{H(\hat\calP_{m-1},\calP_{m-1}^*)(\calT^*)^{\lrangle{m-1}}+(\hat\calP_{m-1}-\calP_{m-1}^*)\calD^{\lrangle{m-1}}}_{\text{high order terms}=:(\RN{1}.2)}},
\end{align}
where in $(a)$ we use the fact that $(I-\calP_{m-1}^*)(\calT^*)\lr{m-1} = 0$ and $(b)$ follows from the following equation when $i=m-1$, which is a result of Lemma \ref{lemma:operatornorm},
\begin{align*}
	\hat\calP_i - \calP^*_i = T^{*\leq i}(\Lambda_{i+1}^*)^{-1}(V_{i+1}^*)^T\Delta_i^T(I-\calP^*_i) + (I-\calP^*_i)\Delta_iV_{i+1}^*(\Lambda_{i+1}^*)^{-1}(T^{*\leq i})^T + H(\hat\calP_i, \calP_i^*),
\end{align*}
and since $T^{*\leq i}$ is the top $r_i$ left singular vectors of $(\calT^*)\lr{i}$ and $\hat T^{\leq i}$ is the top $r_i$ left singular vectors of $(\hat\calP_{i-1}\otimes I)\calT\lr{i}$, so
$\Delta_i = (\hat\calP_{i-1}\otimes I)\calT\lr{i} - (\calP_{i-1}^*\otimes I)(\calT^*)\lr{i}$, and thus 
\begin{align}\label{def:Delta}
	\Delta_i &= \left((\hat \calP_{i-1} - \calP_{i-1}^*)\otimes I_{d_i}\right)(\calT^*)^{\lrangle{i}} + \left(\hat \calP_{i-1}\otimes I_{d_i}\right)\calD^{\lrangle{i}}\no\\
	&=  \left((\hat \calP_{i-1} - \calP_{i-1}^*)\otimes I_{d_i}\right)(\calT^*)^{\lrangle{i}} + \left(\calP^*_{i-1}\otimes I_{d_i}\right)\calD^{\lrangle{i}} + \left((\hat\calP_{i-1}-\calP^*_{i-1})\otimes I_{d_i}\right)\calD^{\lrangle{i}}\no\\
	&= \left((I-\calP^*_{i-1})\Delta_{i-1}V_{i}^*(\Lambda_{i}^*)^{-1}(T^{*\leq i-1})^T\otimes I_{d_{i}}\right)(\calT^*)^{\lrangle{i}}+ \left(\calP^*_{i-1}\otimes I_{d_i}\right)\calD^{\lrangle{i}}\no\\
	&~~~~ + \underbrace{\left(H(\hat\calP_{i-1}, \calP_{i-1}^*)\otimes I_{d_{i}}\right)(\calT^*)^{\lrangle{i}}+ \left((\hat\calP_{i-1}-\calP^*_{i-1})\otimes I_{d_i}\right)\calD^{\lrangle{i}}}_{\text{high order terms}=: H_i}\no\\
	&=: L_i + H_i,
\end{align}
here we use $L_i$ and $H_i$ to denote the leading terms and high order terms of $\Delta_i$ respectively. Now we derive first the upper bound for $\fro{H_i}$.

\hspace{0.2cm}

\noindent{\it Bound for $\fro{H_i}$.} Notice by triangle inequality, we have
\begin{align*}
	\fro{H_i} \leq \fro{\left(H(\hat\calP_{i-1}, \calP_{i-1}^*)\otimes I_{d_{i}}\right)(\calT^*)^{\lrangle{i}}} + \fro{\left((\hat\calP_{i-1}-\calP^*_{i-1})\otimes I_{d_i}\right)\calD^{\lrangle{i}}}.
\end{align*}
From the assumption we have $\sigma_{\min}(\Lambda_i^*) \geq 8\fro{\Delta_{i-1}}$, we have
\begin{align*}
	\fro{\left(H(\hat\calP_{i-1}, \calP_{i-1}^*)\otimes I_{d_{i}}\right)(\calT^*)^{\lrangle{i}}} \leq \frac{12\fro{\Delta_{i-1}}^2}{\sigma_{\min}(\Lambda_i^*)},
\end{align*}
and
\begin{align*}
	\fro{\left((\hat\calP_{i-1}-\calP^*_{i-1})\otimes I_{d_i}\right)\calD^{\lrangle{i}}} \leq \frac{4\fro{\calD}\fro{\Delta_{i-1}}}{\sigma_{\min}(\Lambda_i^*)}.
\end{align*}
So combine these two estimations, and we have
\begin{align}\label{upperbound:Hi}
	\fro{H_i} \leq \frac{12\fro{\Delta_{i-1}}^2}{\sigma_{\min}(\Lambda_i^*)} + \frac{4\fro{\calD}\fro{\Delta_{i-1}}}{\sigma_{\min}(\Lambda_i^*)}.
\end{align}

\hspace{0.2cm}

\noindent{\it Upper bound for $\fro{\Delta_i}$.} 
We first show the following for all $2\leq i\leq m-1$.
\begin{align}\label{recurrence:projection:Delta}
	&~~~~\fro{(I-\calP_i^*)\Delta_i}^2 - \fro{(I-\calP^*_{i-1})\Delta_{i-1}}^2 \no\\
	&\leq \fro{(T^{*\leq i -1}\otimes I_{d_i})(I-L(T_i^*)L(T_i^*)^T)((T^{*\leq i -1})^T\otimes I_{d_i})\calD^{\lrangle{i}}}^2 + 2\fro{(I-\calP_i^*)L_i}\fro{H_i} + \fro{H_i}^2.
\end{align}
As a consequence of $\calP_i^* = (U^{*\leq i-1}\otimes I_{d_i}) L(T_i^*)L(T_i^*)^T((T^{*\leq i-1})^T\otimes I_{d_i})$, we have
\begin{align*}
	\calP_i^*\left((I-\calP^*_{i-1})\Delta_{i-1}V_{i}^*(\Lambda_{i}^*)^{-1}(T^{*\leq i-1})^T\otimes I_{d_{i}}\right)(\calT^*)^{\lrangle{i}} = 0.
\end{align*}
So the leading term of $(I-\calP_i^*)\Delta_i$ is 
\begin{align}
	&~~~~(I-\calP_i^*)L_i =\left((I-\calP^*_{i-1})\Delta_{i-1}V_{i}^*(\Lambda_{i}^*)^{-1}(T^{*\leq i-1})^T\otimes I_{d_{i}}\right)(\calT^*)^{\lrangle{i}}+ (I-\calP^*_i)\left(\calP^*_{i-1}\otimes I_{d_i}\right)\calD^{\lrangle{i}}\no\\
	&\overset{(a)}{=} \left((I-\calP^*_{i-1})\Delta_{i-1}V_{i}^*(\Lambda_{i}^*)^{-1}(T^{*\leq i-1})^T\otimes I_{d_{i}}\right)(\calT^*)^{\lrangle{i}} \no\\
	&\qquad+ (T^{*\leq i -1}\otimes I_{d_i})(I-L(T_i^*)L(T_i^*)^T)((T^{*\leq i -1})^T\otimes I_{d_i})\calD^{\lrangle{i}}\label{leading:orthogonal}\\
	&\overset{(b)}{=} \reshape\left((I-\calP^*_{i-1})\Delta_{i-1}V_{i}^*V_i^{*T}\right)+ (T^{*\leq i -1}\otimes I_{d_i})(I-L(T_i^*)L(T_i^*)^T)((T^{*\leq i -1})^T\otimes I_{d_i})\calD^{\lrangle{i}}\no,
\end{align}
where in $(a)$ we use
\begin{align*}
	&~~~~(I-\calP^*_i)\left(\calP^*_{i-1}\otimes I_{d_i}\right) \no\\
	&= \left(I - (T^{*\leq i-1}\otimes I_{d_i}) L(T_i^*)L(T_i^*)^T((T^{*\leq i-1})^T\otimes I_{d_i})\right)\left(T^{*\leq i-1}(T^{*\leq i-1})^T\otimes I_{d_i}\right)\no\\
	&= T^{*\leq i-1}(T^{*\leq i-1})^T\otimes I_{d_i} - (T^{*\leq i-1}\otimes I_{d_i}) L(T_i^*)L(T_i^*)^T(T^{*\leq i-1}\otimes I_{d_i})^T\no\\
	&= (T^{*\leq i-1}\otimes I_{d_i})(I-L(T_i^*)L(T_i^*)^T)(T^{*\leq i-1}\otimes I_{d_i})^T.
\end{align*}
And in $(b)$ we use Lemma~\ref{lemma:reshape}.
Notice the two terms in \eqref{leading:orthogonal} are mutually orthogonal, so we have 
\begin{align}\label{projection:L}
	&~~~~\fro{(I-\calP_i^*)L_i}^2 \no\\
	&= \fro{(I-\calP^*_{i-1})\Delta_{i-1}V_{i}^*V_i^{*T}}^2 + \fro{(T^{*\leq i -1}\otimes I_{d_i})(I-L(T_i^*)L(T_i^*)^T)((T^{*\leq i -1})^T\otimes I_{d_i})\calD^{\lrangle{i}}}^2\no\\
	&\leq \fro{(I-\calP^*_{i-1})\Delta_{i-1}}^2 + \fro{(T^{*\leq i -1}\otimes I_{d_i})(I-L(T_i^*)L(T_i^*)^T)((T^{*\leq i -1})^T\otimes I_{d_i})\calD^{\lrangle{i}}}^2.
\end{align}
Meanwhile, from \eqref{def:Delta}, we have 
\begin{align}\label{projection:Delta}
	\fro{(I-\calP_i^*)\Delta_i}^2 \leq \fro{(I-\calP_i^*)L_i}^2 + 2\fro{(I-\calP_i^*)L_i}\fro{H_i} + \fro{H_i}^2.
\end{align}
Combine \eqref{projection:L}, \eqref{projection:Delta} and we get \eqref{recurrence:projection:Delta}.

\hspace{0.2cm}

\noindent{\it Upper bound for $\fro{(I-\calP_{k}^*)\Delta_{k}}^2 + \fro{\calP_{k}^*\calD\lr{k}}^2$.}
From the recurrence relation \eqref{recurrence:projection:Delta}, we have
\begin{align*}
	&~~~~\fro{(I-\calP_{k}^*)\Delta_{k}}^2 + \fro{\calP_{k}^*\calD\lr{k}}^2\no\\
	&\leq \fro{\calP_{k}^*\calD^{\lrangle{k}}}^2 + \sum_{i=2}^{k}\fro{(T^{*\leq i -1}\otimes I_{d_i})(I-L(T_i^*)L(T_i^*)^T)((T^{*\leq i -1})^T\otimes I_{d_i})\calD^{\lrangle{i}}}^2 \no\\
	&~~~~+ \fro{(I-\calP_1^*)\calD\lr{1}}+\underbrace{\sum_{i=2}^{k}\left(2\fro{(I-\calP_i^*)L_i}\fro{H_i} + \fro{H_i}^2\right)}_{\text{high order terms} =: \xi_{k,2}}\no\\
	&=: \xi_{k,1} + \xi_{k,2}.
\end{align*}
Now we first show $\xi_{k,1} = \fro{\calD}^2$. The key point of the proof lies in the following equation:
\begin{align}\label{eq:lemma:PkDk:main}
	\fro{\calP_i^*\calD\lr{i}}^2 + \fro{(T^{*\leq i -1}\otimes I_{d_i})(I-L(T_i^*)L(T_i^*)^T)((T^{*\leq i -1})^T\otimes I_{d_i})\calD^{\lrangle{i}}}^2 = \fro{\calP_{i-1}^*\calD\lr{i-1}}^2.
\end{align}
Suppose this holds, then we have the left hand side is equal to $\fro{\calP_1^*\calD\lr{1}}^2 + \fro{(I-\calP_1^*)\calD\lr{1}}^2 = \fro{\calD}^2$ and this finishes the proof. So now we verify \eqref{eq:lemma:PkDk:main}.
\begin{align*}
	\text{LHS of }\eqref{eq:lemma:PkDk:main} &= \fro{(T^{*\leq i -1}\otimes I_{d_i})L(T_i^*)L(T_i^*)^T((T^{*\leq i -1})^T\otimes I_{d_i})\calD^{\lrangle{i}}}^2 \no\\&~~~~+ \fro{(T^{*\leq i -1}\otimes I_{d_i})(I-L(T_i^*)L(T_i^*)^T)((T^{*\leq i -1})^T\otimes I_{d_i})\calD^{\lrangle{i}}}^2\no\\
	&= \fro{(T^{*\leq i -1}\otimes I_{d_i})((T^{*\leq i -1})^T\otimes I_{d_i})\calD^{\lrangle{i}}}^2\no\\
	&\overset{(a)}{=} \fro{\calP_{i-1}^*\calD\lr{i-1}}^2,
\end{align*}
where in $(a)$ we use Lemma~\ref{lemma:reshape}.

On the other hand, for $\xi_{k,2}$, we have from \eqref{projection:L}, $\fro{(I-\calP_i^*)L_i} \leq 3\fro{\calD}$.
And from \eqref{upperbound:Hi}, we have
\begin{align*}
	\fro{H_i}\leq \frac{12\fro{\Delta_{i-1}}^2}{\sigma_{\min}(\Lambda_i^*)} + \frac{4\fro{\calD}\fro{\Delta_{i-1}}}{\sigma_{\min}(\Lambda_i^*)}\leq \frac{56\fro{\calD}^2}{\sigma_{\min}(\Lambda_i^*)}.
\end{align*}
Combine the above two estimations, and we have
\begin{align*}
	\xi_{k,2} \leq \sum_{i=2}^k \left(2\cdot 3\fro{\calD}\frac{56\fro{\calD}^2}{\sigma_{\min}(\Lambda_i^*)} + \frac{56^2\fro{\calD}^4}{\sigma_{\min}^2(\Lambda_i^*)}\right) \leq \frac{350(k-1)\fro{\calD}^3}{\sigmamin}.
\end{align*}
And thus
\begin{align}\label{lemma:upperbound:DeltaD}
	\fro{(I-\calP_{k}^*)\Delta_{k}}^2 + \fro{\calP_{k}^*\calD\lr{k}}^2 \leq \fro{\calD}^2 + \frac{350(k-1)\fro{\calD}^3}{\sigmamin}.
\end{align}

\noindent{\it Upper bound for $\fro{\Delta_i}$.}
We shall bound $\fro{\Delta_i}$ by induction.
For the base case when $i = 1$, we have $\fro{\Delta_1} = \fro{\calD\lr{1}} = \fro{\calD}\leq 2\fro{\calD}$. Now suppose we have the bound for $\fro{\Delta_{i}}\leq 2\fro{\calD}$ for all $1\leq i\leq k$. 
Then from the definition of $L_{k+1}$, we have
\begin{align*}
	\fro{L_{k+1}}^2 &= \fro{(I-\calP_{k}^*)\Delta_{k}V_{k+1}^*V_{k+1}^{*T}}^2 + \fro{(\calP_{k}^*\otimes I_{d_{k+1}})\calD\lr{k+1}}^2\leq  \fro{(I-\calP_{k}^*)\Delta_{k}}^2 + \fro{\calP_{k}^*\calD\lr{k}}^2\no\\
	&\overset{(a)}{\leq}\fro{\calD}^2 + \frac{350(k-1)\fro{\calD}^3}{\sigma_{\min}(\Lambda_k^*)},
\end{align*}
where $(a)$ follows from \eqref{lemma:upperbound:DeltaD}.
And the upper bound for $\fro{H_{k+1}}$ is already derived as in \eqref{upperbound:Hi}, 
since $\sigma_{\min}(\Lambda_{k+1}^*) \geq C_1\fro{\calD} \geq 8\fro{\Delta_k}$,
we have
\begin{align*}
	\fro{H_{k+1}} \leq \frac{12\fro{\Delta_{k}}^2}{\sigma_{\min}(\Lambda_{k+1}^*)} + \frac{4\fro{\calD}\fro{\Delta_{k}}}{\sigma_{\min}(\Lambda_{k+1}^*)}\leq\frac{56\fro{\calD}^2}{\sigma_{\min}(\Lambda_{k+1}^*)}.
\end{align*}
From \eqref{def:Delta}, we have
\begin{align*}
	\fro{\Delta_{k+1}} &= \fro{L_{k+1} + H_{k+1}}\leq \sqrt{1+\frac{350(k-1)\fro{\calD}}{\sigma_{\min}(\Lambda_{k+1}^*)}}\fro{\calD} + \frac{56\fro{\calD}}{\sigma_{\min}(\Lambda_{k+1}^*)}\fro{\calD}\leq 2\fro{\calD}.
\end{align*}
And this finishes the induction. So it holds for all $i\in[m-1]$ that $\fro{\Delta_i}\leq 2\fro{\calD}$.

\hspace{0.2cm}

\noindent{\it Estimation of $\fro{\ttsvd_{\r}(\calT) - \calT^*}$.} Now we go back to \eqref{m>=3:main} and we bound $\fro{(\RN{1}.1)}$ and $\fro{(\RN{1}.2)}$ separately. For the term $\fro{(\RN{1}.2)}$, we have
\begin{align*}
	\fro{(\RN{1}.2)} &= \fro{H(\hat\calP_{m-1},\calP_{m-1}^*)(\calT^*)^{\lrangle{m-1}}+(\hat\calP_{m-1}-\calP_{m-1}^*)\calD^{\lrangle{m-1}}}\overset{(a)}{=} \fro{H_m},
\end{align*}
where $(a)$ follows from \eqref{def:Delta} and Lemma~\ref{lemma:reshape}. So from \eqref{upperbound:Hi}, we have
\begin{align*}
	\fro{(\RN{1}.2)} \leq \frac{12\fro{\Delta_{m-1}}^2}{\sigma_{\min}(\Lambda_{m}^*)} + \frac{4\fro{\calD}\fro{\Delta_{m-1}}}{\sigma_{\min}(\Lambda_{m}^*)}\leq \frac{56\fro{\calD}^2}{\sigma_{\min}(\Lambda_{m}^*)}.
\end{align*}
And for the term $\fro{(\RN{1}.1)}$, we have from \eqref{lemma:upperbound:DeltaD},
\begin{align*}
	\fro{(\RN{1}.1)}^2 \leq \fro{(I-\calP_{m-1}^*)\Delta_{m-1}}^2 + \fro{\calP_{m-1}^*\calD\lr{m-1}}^2 \leq \fro{\calD}^2 + \frac{350(m-2)\fro{\calD}^3}{\sigma_{\min}(\Lambda_{m-1}^*)}
\end{align*}
So together with \eqref{m>=3:main}, we have
\begin{align*}
	\fro{\ttsvd_{\r}(\calT) - \calT^*}^2 &\leq \fro{(\RN{1}.1)}^2 + 2\fro{(\RN{1}.1)}\fro{(\RN{1}.2)} + \fro{(\RN{1}.2)}^2\no\\
	&\leq \fro{\calD}^2 + \frac{350(m-2)\fro{\calD}^3}{\sigma_{\min}(\Lambda_{m-1}^*)} + \frac{250\fro{\calD}^3}{\sigma_{\min}(\Lambda_{m}^*)}\no\\
	&\leq \fro{\calD}^2 + \frac{600m\fro{\calD}^3}{\sigmamin}.
\end{align*}
And this finishes the proof of the lemma.

\subsection{Proof of Lemma~\ref{lemma:estimationofptperp}}\label{sec:lemma:estimationofptperp}
First we introduce some notations,
$$\calT = [T_1,\ldots,T_m], \calT^* = [T_1^*,\ldots,T_m^*], \text{~and~}(\calT^*)\lr{i} = T^{*\leq i}\Lambda_{i+1}^*V^*_{i+1}, \calT\lr{i} = T^{\leq i}\Lambda_{i+1}V_{i+1}, i\in[m-1].$$
Also we denote $\calP_i = T^{\leq i}T^{\leq i T}$, $\calQ_{i+1} = V_{i+1}V_{i+1}^T$ as shorthand notations, and define similar notations for $\calT^*$. Now we take $\calA = \calT^*$ and denote $\calP_{\TT}(\calT^*) = \delta\calT_1 + \ldots + \delta\calT_m$, then we have for all $i\in[m-1]$
\begin{align*}
	(\delta\calT_i)\lr{i} &= (T^{\leq i - 1}\otimes I)(I - L(T_i)L(T_i)^T)(T^{\leq i - 1}\otimes I)^T(\calT^*)\lr{i}V_{i+1}V_{i+1}^T\notag\\
	&= (I - \calP_i)(\calP_{i-1}\otimes I)(\calT^*)\lr{i}\calQ_{i+1}.
\end{align*}
And 
\begin{align*}
	-\fro{\delta\calT_i}^2 &= -\inp{(I - \calP_i)(\calP_{i-1}\otimes I)(\calT^*)\lr{i}\calQ_{i+1}}{(I - \calP_i)(\calP_{i-1}\otimes I)(\calT^*)\lr{i}\calQ_{i+1}}\no\\
	&= -\inp{(\calT^*)\lr{i}}{(\calP_{i-1}\otimes I)(I - \calP_i)(\calP_{i-1}\otimes I)(\calT^*)\lr{i}\calQ_{i+1}}\notag\\
	&= -\inp{(\calT^*)\lr{i}}{(\calP_{i-1}\otimes I)(\calT^*)\lr{i}\calQ_{i+1}} + \inp{(\calT^*)\lr{i}}{\calP_i(\calT^*)\lr{i}\calQ_{i+1}}\no\\
	&= -\inp{(\calT^*)\lr{i}}{(\calP_{i-1}\otimes I)(\calT^*)\lr{i}(\calQ_{i+1} - \calQ_{i+1}^*)} + \inp{(\calT^*)\lr{i}}{\calP_i(\calT^*)\lr{i}(\calQ_{i+1}-\calQ_{i+1}^*)}\no\\
	&~~~~ -\inp{(\calT^*)\lr{i}}{(\calP_{i-1}\otimes I)(\calT^*)\lr{i}} + \inp{(\calT^*)\lr{i}}{\calP_i(\calT^*)\lr{i}}\no\\
	&= -\inp{(\calT^*)\lr{i}}{((\calP_{i-1}-\calP_{i-1}^*)\otimes I)(\calT^*)\lr{i}(\calQ_{i+1} - \calQ_{i+1}^*)} \no\\
	&~~~~+ \inp{(\calT^*)\lr{i}}{(\calP_i - \calP_i^*)(\calT^*)\lr{i}(\calQ_{i+1}-\calQ_{i+1}^*)}\no\\
	&~~~~ -\inp{(\calT^*)\lr{i}}{(\calP_{i-1}\otimes I)(\calT^*)\lr{i}} + \inp{(\calT^*)\lr{i}}{\calP_i(\calT^*)\lr{i}}\no\\
\end{align*}
Meanwhile, when $i = m$, we have
\begin{align*}
	-\fro{\delta\calT_m}^2 = -\fro{(\calP_{m-1}\otimes I)(\calT^*)\lr{m}}^2 = -\fro{\calP_{m-1}(\calT^*)\lr{m-1}}^2,
\end{align*}
where the last equality is from Lemma~\ref{lemma:reshape}.
Now we first estimate 
\begin{align*}
	&~~~~\sum_{i=1}^{m-1}\left(-\inp{(\calT^*)\lr{i}}{(\calP_{i-1}\otimes I)(\calT^*)\lr{i}} + \inp{(\calT^*)\lr{i}}{\calP_i(\calT^*)\lr{i}}\right) -\fro{\calP_{m-1}(\calT^*)\lr{m-1}}^2\no\\
	&=\sum_{i=1}^{m-1}\left(-\fro{(\calP_{i-1}\otimes I)(\calT^*)\lr{i}}^2 + \fro{\calP_i(\calT^*)\lr{i}}^2\right) -\fro{\calP_{m-1}(\calT^*)\lr{m-1}}^2\no\\
	&= \sum_{i=1}^{m-1}\left(-\fro{\calP_{i-1}(\calT^*)\lr{i-1}}^2 + \fro{\calP_i(\calT^*)\lr{i}}^2\right) -\fro{\calP_{m-1}(\calT^*)\lr{m-1}}^2 = -\fro{\calT^*}^2.
\end{align*}
And now we estimate $\inp{(\calT^*)\lr{i}}{(\calP_i - \calP_i^*)(\calT^*)\lr{i}(\calQ_{i+1}-\calQ_{i+1}^*)}$ first,
\begin{align*}
	&~~~~\inp{(\calT^*)\lr{i}}{(\calP_i - \calP_i^*)(\calT^*)\lr{i}(\calQ_{i+1}-\calQ_{i+1}^*)} \no\\
	&= \inp{(\calP_i - \calP_i^*)(\calT^*)\lr{i}}{(\calT^*)\lr{i}(\calQ_{i+1}-\calQ_{i+1}^*)}\no\\
	&= \inp{\calP_i^*\Delta_i(I-\calQ_{i+1}^*)+(\calT^*)\lr{i}H(\calQ_{i+1},\calQ_{i+1}^*)}{(I-\calP_i^*)\Delta_i\calQ_{i+1}^* + H(\calP_i,\calP_i^*)(\calT^*)\lr{i}}\no\\
	&=\inp{(\calT^*)\lr{i}H(\calQ_{i+1},\calQ_{i+1}^*)}{H(\calP_i,\calP_i^*)(\calT^*)\lr{i}}\no\\
	&\leq \fro{(\calT^*)\lr{i}H(\calQ_{i+1},\calQ_{i+1}^*)}\fro{H(\calP_i,\calP_i^*)(\calT^*)\lr{i}}
\end{align*}
Notice here we use 
\begin{align*}
	\calP_{i} - \calP_{i}^* &= T^{*\leq i}(\Lambda_{i+1}^*)^{-1}(V_{i+1}^*)^T\Delta_i^T(I-\calP^*_i) + (I-\calP^*_i)\Delta_iV_{i+1}^*(\Lambda_{i+1}^*)^{-1}(T^{*\leq i})^T + H(\calP_i, \calP_i^*),
\end{align*}
and
\begin{align*}
\calQ_{i+1} - \calQ_{i+1}^* &= V_{i+1}^*(\Lambda_{i+1}^*)^{-1}(T^{*\leq i})^T\Delta_i(I-\calQ_{i+1}^*) + (I-\calQ_{i+1}^*)\Delta_i^TT^{*\leq i}(\Lambda_{i+1}^*)^{-1} V_{i+1}^{*T} \\
&~~~~+ H(\calQ_{i+1}, \calQ_{i+1}^*),
\end{align*}
where $\Delta_i = (\calT-\calT^*)\lr{i}$ and $H(\cdot,\cdot)$ denotes high order terms. 

Similarly for $\inp{(\calT^*)\lr{i}}{((\calP_{i-1}-\calP_{i-1}^*)\otimes I)(\calT^*)\lr{i}(\calQ_{i+1} - \calQ_{i+1}^*)}$, we have
\begin{align*}
	&~~~~\inp{(\calT^*)\lr{i}}{((\calP_{i-1}-\calP_{i-1}^*)\otimes I)(\calT^*)\lr{i}(\calQ_{i+1} - \calQ_{i+1}^*)}\no\\
	&= \inp{(\calT^*)\lr{i}H(\calQ_{i+1},\calQ_{i+1}^*)}{(H(\calP_{i-1},\calP_{i-1}^*)\otimes I)(\calT^*)\lr{i}}\no\\
	&\leq \fro{(\calT^*)\lr{i}H(\calQ_{i+1},\calQ_{i+1}^*)}\fro{(H(\calP_{i-1},\calP_{i-1}^*)\otimes I)(\calT^*)\lr{i}}\no\\
	&= \fro{(\calT^*)\lr{i}H(\calQ_{i+1},\calQ_{i+1}^*)}\fro{H(\calP_{i-1},\calP_{i-1}^*)(\calT^*)\lr{i-1}}
\end{align*}

Now as long as $\sigmamin\geq 8\fro{\calT - \calT^*}$, we have
\begin{align*}
	\fro{H(\calP_i,\calP_i^*)(\calT^*)\lr{i}} &\leq \frac{12\fro{\calT- \calT^*}^2}{\sigmamin},\\
	\fro{(\calT^*)\lr{i}H(\calQ_{i+1},\calQ_{i+1}^*)} &\leq \frac{12\fro{\calT- \calT^*}^2}{\sigmamin}.
\end{align*}
Finally, we have
\begin{align*}
	\fro{\calP_{\TT}^{\perp}(\calT^*)}^2 = \fro{\calT^*}^2 - \sum_{i=1}^m \fro{\delta\calT_i}^2
	&= \sum_{i=1}^{m-1}\big(-\inp{(\calT^*)\lr{i}}{((\calP_{i-1}-\calP_{i-1}^*)\otimes I)(\calT^*)\lr{i}(\calQ_{i+1} - \calQ_{i+1}^*)} \no\\
	&~~~~~~+ \inp{(\calT^*)\lr{i}}{(\calP_i - \calP_i^*)(\calT^*)\lr{i}(\calQ_{i+1}-\calQ_{i+1}^*)}\big)\no\\
	&\leq \frac{288m^2\fro{\calT-\calT^*}^4}{\sigmamin^2}.
\end{align*}
And this finishes the proof.

\subsection{Proof of Lemma~\ref{lemma:ttsvd+trim}}\label{sec:lemma:ttsvd+trim}
We begin the proof introducing some notations. For simplicity, we denote $\wt\calW = \trim_{\zeta}(\calW)$. And $\ttsvd_{\r}(\wt\calW) = \hat\calW = [\hat W_1,\ldots, \hat W_m]$. And we denote for all $i\in[m-1]$, $\hat\calW\lr{i} = \hat W^{\leq i}\hat W^{\geq i+1} = \hat W^{\leq i} \Sigma_{i+1} N_{i+1}^T$ where $\Sigma_{i+1}\in\RR^{r_i\times r_i}$ is invertible and $N_{i+1}^TN_{i+1} = I_{r_i}$. And we also introduce some notations for the process of TTSVD. From Algorithm~\ref{alg:svd}, we know $(\hat W^{\leq i-1}\otimes I)^T \wt\calW\lr{i} = L(\hat W_i)S_{i+1}R_{i+1}^T + E_i$ is how we get the estimation $\hat W_i$ once we know $\hat W_1,\ldots, \hat W_{i-1}$, where $S_{i+1}\in\RR^{r_i\times r_i}$ is invertible and $R_{i+1}^TR_{i+1} = I_{r_i}$. To estimate the incoherence of $\hat\calW$, we check $\max_{j}\|e_j^T\hat W^{\leq i}\|_{\ell_2}$ and $\max_{j}\|N_{i+1}^Te_j\|_{\ell_2}$. 

We first estimate $\sigma_{\min}(\Sigma_{i+1})$ and $\sigma_{\min}(S_{i+1})$. From the Algorithm~\ref{alg:svd}, we know
\begin{align*}
	\sigma_{\min}(S_{i+1}) &= \sigma_{\min}((\hat W^{\leq i-1}\hat W^{\leq i-1T}\otimes I)\wt\calW\lr{i})\no\\
	& \geq \sigma_{\min}((\calT^{*})\lr{i}) - \fro{(\calT^*)\lr{i} - (\hat W^{\leq i-1}\hat W^{\leq i-1T}\otimes I)\wt\calW\lr{i}}.
\end{align*}
So we need to derive an upper bound for $\fro{(\calT^*)\lr{i} - (\hat W^{\leq i-1}\hat W^{\leq i-1T}\otimes I)\wt\calW\lr{i}}$.
\begin{align}\label{eq:ttsvd+trim:0001}
	(\hat W^{\leq i-1}\hat W^{\leq i-1T}\otimes I)\wt\calW\lr{i}-(\calT^*)\lr{i} &= (\hat W^{\leq i-1}\hat W^{\leq i-1T} - T^{*\leq i-1}T^{*\leq i-1T})\otimes I \cdot (\wt\calW\lr{i} - (\calT^*)\lr{i})\no\\
	&~~~~+ T^{*\leq i-1}T^{*\leq i-1T}\otimes I \cdot (\wt\calW\lr{i} - (\calT^*)\lr{i})\no\\
	&~~~~+(\hat W^{\leq i-1}\hat W^{\leq i-1T} - T^{*\leq i-1}T^{*\leq i-1T})\otimes I \cdot (\calT^*)\lr{i}.
\end{align}
Since $\|\calT^*\|_{\ell_{\infty}} = \spiki(\calT^*)\frac{\fro{\calT^*}}{\sqrt{d^*}}\leq \nu\frac{\fro{\calT^*}}{\sqrt{d^*}}\leq\frac{10\nu}{9}\frac{\fro{\calW}}{\sqrt{d^*}} = \zeta$, we have $\fro{\wt\calW -\calT^*} \leq \fro{\calW-\calT^*}$. And the bound for $\op{\hat W^{\leq i-1}\hat W^{\leq i-1T} - T^{*\leq i-1}T^{*\leq i-1T}}$ goes as follows. First we notice that from Lemma~\ref{lemma:tt-perturbation},
\begin{align}\label{eq:hatW-T^*}
	\fro{\hat\calW - \calT^*}^2 \leq \fro{\wt\calW - \calT^*}^2 + \frac{600m\fro{\wt\calW - \calT^*}^3}{\sigmamin} \leq 2\fro{\calW - \calT^*}^2,
\end{align}
where the last inequality is from the assumption $\fro{\wt\calW - \calT^*} \leq \fro{\calW - \calT^*} \leq \frac{1}{600m}\sigmamin$. So we have from Wedin's theorem,
\begin{align*}
	\op{\hat W^{\leq i-1}\hat W^{\leq i-1T} - T^{*\leq i-1}T^{*\leq i-1T}} \leq \frac{\sqrt{2}\fro{\hat\calW - \calT^*}}{\sigma_{\min}(\Lambda_i^*) - \fro{\wt\calW - \calT^*}}\leq \frac{4\fro{\calW-\calT^*}}{\sigmamin}\leq \frac{1}{150m\sqrt{\rmax}\kappa_0},
\end{align*}
where in the penultimate inequality we use $\sigma_{\min}(\Lambda_i^*) - \fro{\wt\calW - \calT^*}\geq \frac{\sigmamin}{2}$ and \eqref{eq:hatW-T^*}. This together with \eqref{eq:ttsvd+trim:0001} gives 
\begin{align*}
	\fro{(\calT^*)\lr{i} - (\hat W^{\leq i-1}\hat W^{\leq i-1T}&\otimes I)\wt\calW\lr{i}} \no\\
	&\leq \frac{4\fro{\calW-\calT^*}}{\sigmamin}\cdot\fro{\calW-\calT^*} + \fro{\calW-\calT^*} + \frac{4\fro{\calW-\calT^*}}{\sigmamin}\cdot\fro{\calT^*}\no\\
	&\leq \frac{\sigmamin}{10}.
\end{align*}
So we conclude that $\sigma_{\min}(S_{i+1}) \geq \frac{9}{10}\sigmamin$.

Meanwhile, for $\sigma_{\min}(\Sigma_{i+1})$ and $\sigma_{\max}(\Sigma_{i+1})$, we have
\begin{align*}
	\sigma_{\max}(\Sigma_{i+1}) = \sigma_{\max}(\hat\calW\lr{i}) \leq \sigma_{\max}((\calT^*)\lr{i}) + \fro{\calT^*-\calW}\leq \frac{11}{10}\sigmamax,\\
	\sigma_{\min}(\Sigma_{i+1}) = \sigma_{\min}(\hat\calW\lr{i}) \geq \sigma_{\min}((\calT^*)\lr{i}) - \fro{\calT^*-\calW}\geq \frac{9}{10}\sigmamin.
\end{align*}

With these preparations, we are ready to bound the incoherence of $\hat\calW$. For all $j\in[d_1\ldots d_i]$,
\begin{align}\label{incoh:1}
	\|e_j^T\hat W^{\leq i}\|_{\ell_2} &= \|e_j^T(\hat W^{\leq i-1}\hat W^{\leq i-1T}\otimes I)\wt\calW\lr{i}R_{i+1}S_{i+1}^{-1}\|_{\ell_2}\no\\
	&\leq \frac{\|f_j^T\wt\calW\lr{i}\|_{\ell_2}}{\sigma_{\min}(S_{i+1})} 
	\overset{(a)}{\leq} \frac{10\|f_j\|_{\ell_2}\sqrt{d_{i+1}\ldots d_m}\|\wt\calW\|_{\ell_{\infty}}}{9\sigmamin}\no\\
	&\overset{(b)}{\leq} \frac{100}{81\sigmamin}\frac{\nu\fro{\calW}}{\sqrt{d_1\ldots d_i}} \leq \frac{100\nu}{81\sigmamin}\frac{\fro{\calT^*} + \fro{\calW-\calT^*}}{\sqrt{d_1\ldots d_i}}\no\\
	&\leq \frac{100\nu}{81}\kappa_0\frac{\sqrt{r_i}}{\sqrt{d_1\ldots d_i}} + \frac{10\nu}{81}\frac{1}{\sqrt{d_1\ldots d_i}} \leq \frac{110\nu\kappa_0}{81}\frac{\sqrt{r_i}}{\sqrt{d_1\ldots d_i}},
\end{align}
where $f_j = (\hat W^{\leq i-1}\hat W^{\leq i-1T}\otimes I)e_j$ and in $(a)$ we use the estimation for $\sigma_{\min}(S_{i+1})$ and $\|Wx\|_{\ell_2} \leq \sqrt{n}\|W\|_{\ell_{\infty}}\|x\|_{\ell_2}$ for some matrix $W\in\RR^{n\times m}$; in $(b)$ we use $\|\wt\calW\|_{\ell_{\infty}} \leq \zeta = \frac{10\fro{\calW}}{9\sqrt{d^*}}\nu$ and $\|f_j\|_{\ell_{\infty}}\leq 1$.

On the other hand, we have 
\begin{align}\label{eq:trim+svd:0001}
	\hat W^{\geq i+1} = R(\hat W_{i+1})(I\otimes \hat W^{\geq i+2}).
\end{align}
Here we use the convention that $\hat W^{m+1} = [1]$. And from $L(\hat W_{i+1}) =(\hat W^{\leq i}\otimes I)^T\wt\calW\lr{i+1}R_{i+2}S_{i+2}^{-1}$, we obtain
\begin{align}\label{eq:trim+svd:0002}
	R(\hat W_{i+1}) = (\hat W^{\leq i})^T\reshape(\wt\calW\lr{i+1}R_{i+2}S_{i+2}^{-1}) = (\hat W^{\leq i})^T\wt\calW\lr{i}(I\otimes R_{i+2}S_{i+2}^{-1}).
\end{align}
Combine \eqref{eq:trim+svd:0001} and \eqref{eq:trim+svd:0002}, we have
$$
\hat W^{\geq i+1} = (\hat W^{\leq i})^T\wt\calW\lr{i}(I\otimes R_{i+2}S_{i+2}^{-1}\hat W^{\geq i+2}), 
$$
which implies
$$
N_{i+1}^T = \Sigma_{i+1}^{-1}(\hat W^{\leq i})^T\wt\calW\lr{i}(I\otimes R_{i+2}S_{i+2}^{-1}\hat W^{\geq i+2}).
$$
Now for any $j\in[d_{i+1}\ldots d_m]$, we have
\begin{align}\label{incoh:2}
	\|N_{i+1}^Te_j\|_{\ell_2} &= \|\Sigma_{i+1}^{-1}(\hat W^{\leq i})^T\wt\calW\lr{i}(I\otimes R_{i+2}S_{i+2}^{-1}\hat W^{\geq i+2})e_j\|_{\ell_2}\no\\
	&\leq \frac{1}{\sigma_{\min}(\Sigma_{i+1})} \|(\hat W^{\leq i})^T\wt\calW\lr{i}(I\otimes R_{i+2}S_{i+2}^{-1}\hat W^{\geq i+2})e_j\|_{\ell_2}\no\\
	&\leq \frac{1}{\sigma_{\min}(\Sigma_{i+1})} \|\wt\calW\lr{i}(I\otimes R_{i+2}S_{i+2}^{-1}\hat W^{\geq i+2})e_j\|_{\ell_2}\no\\
	&\leq \frac{\sqrt{d_1\ldots d_i}}{\sigma_{\min}(\Sigma_{i+1})} \|\wt\calW\|_{\ell_{\infty}}\|(I\otimes R_{i+2}S_{i+2}^{-1}\hat W^{\geq i+2})e_j\|_{\ell_2}\no\\
	&\leq \frac{\sqrt{d_1\ldots d_i}}{\sigma_{\min}(\Sigma_{i+1})} \|\wt\calW\|_{\ell_{\infty}}\frac{\sigma_{\max}(\Sigma_{i+2})}{\sigma_{\min}(S_{i+2})}\no\\
	&\leq \frac{1100}{729\sigmamin}\kappa_0\frac{\nu\fro{\calW}}{\sqrt{d_{i+1}\ldots d_m}}\no\\
	&\leq\frac{1210}{729}\kappa_0^2\nu\frac{\sqrt{r_i}}{\sqrt{d_{i+1}\ldots d_m}}.
\end{align}
From \eqref{incoh:1} and \eqref{incoh:2}, we have
$$\incoh(\hat\calW) \leq 2\kappa_0^2\nu.$$
Now we consider the second part, from Lemma \ref{lemma:tt-perturbation}, we see that 
\begin{align*}
	\fro{\ttsvd(\wt\calW) - \calT^*}^2 \leq \fro{\wt\calW - \calT^*}^2 + \frac{600m\fro{\wt\calW - \calT^*}^3}{\sigmamin} \leq 2\fro{\wt\calT - \calT^*}^2\leq 2\fro{\calW - \calT^*}^2.
\end{align*}
And this finishes the proof of the lemma.

\subsection{Proof of Lemma~\ref{lemma:init:1}}\label{sec:lemma:init:1}
For simplicity, we denote 
$$
S_{1k} = \sum_{i=1}^kU^TX_i, \quad S_{2k} = \sum_{i=1}^kU^TY_i
\quad\text{and }\quad
\Delta_{1k} = S_{1k} - kU^TM, \quad \Delta_{2k} =  S_{2k} - 
kU^TM.
$$
First we estimate $\op{\frac{1}{n}S_{1n} - U^TM}$. Notice that 
$$
\frac{1}{n}S_{1n} - U^TM = \frac{1}{n}\sum_{i = 1}^n\left(p_1p_2U^T\calP_{\omega_i}(M) - U^TM\right) =:\frac{1}{n}\sum_{i=1}^nZ_i.
$$
And the uniform bound for $p_1p_2U^T\calP_{\omega_i}(M)$ is given by 
$$
\op{p_1p_2U^T\calP_{\omega_i}(M)}\leq p_1p_2\|M\|_{\infty}\sqrt{\frac{\mu r}{p_1}} = \sqrt{\mu p_2r}\sqrt{p_1p_2}\|M\|_{\infty}.
$$
and $\op{U^TM}\leq \op{M}\leq \sqrt{p_1p_2}\|M\|_{\infty}$. So
$\op{Z_i}\leq 2\sqrt{\mu p_2r}\sqrt{p_1p_2}\|M\|_{\infty}$.
On the other hand we have
\begin{align*}
	\EE Z_iZ_i^T &\leq p_1p_2\sum_{(j,k)\in[p_1]\times[p_2]}M_{jk}^2U^Te_je_j^TU\no\\
	&\leq p_1p_2\|M\|_{\infty}^2\sum_{(j,k)\in[p_1]\times[p_2]}U^Te_je_j^TU\no\\
	&= p_1p_2^2\|M\|_{\infty}^2I_{r}
\end{align*}
and
\begin{align*}
	\EE Z_i^TZ_i &\leq p_1p_2\sum_{(j,k)\in[p_1]\times[p_2]}M_{jk}^2e_ke_j^TUU^Te_je_k^T\no\\
	&\leq p_1p_2\|M\|_{\infty}^2\sum_k\sum_j\|U^Te_j\|_{\ell_2}^2 e_ke_k^T\no\\
	&\leq p_1p_2\|M\|_{\infty}^2\mu r I_{p_2}.
\end{align*}
Therefore we have
$
\max\{\op{\sum_{i=1}^nZ_iZ_i^T}, \op{\sum_{i=1}^nZ_i^TZ_i}\} \leq np_1p_2(\mu r\vee p_2)\|M\|_{\infty}^2
$. And from operator Bernstein inequality, with probability exceeding $1-p^{-\alpha}$, 
$$
\op{\sum_{i=1}^nZ_i} \leq \sqrt{\frac{8(\alpha+1)\mu p_1p_2^2r\log(p)}{3n}}\|M\|_{\infty}.
$$
Now set the event $\bcalE_1$ as 
$$\left\{\op{\Delta_{1n}}\leq \sqrt{\frac{8n(\alpha+1)\mu p_1p_2^2r\log(p)}{3}}\|M\|_{\infty}\right\} \cap\left\{\op{\Delta_{2n}}\leq \sqrt{\frac{8n(\alpha+1)\mu p_1p_2^2r\log(p)}{3}}\|M\|_{\infty}\right\} $$
and we know $\PP(\bcalE_1)\geq 1-2 p^{-\alpha}$. 
Also, define the event 
\begin{align*}
	\bcalE_2 &= \left\{\max_{j\in[p_2]}\sum_{\omega\in\Omega_1}1(\omega_{2} = j) \leq (3\alpha +7)(\frac{n}{p_2} + \log(p))\right\}\\
	&~~~~\cap\left\{\max_{j\in[p_2]}\sum_{\omega\in\Omega_2}1(\omega_{2} = j) \leq (3\alpha +7)(\frac{n}{p_2} + \log(p))\right\}.
\end{align*}
From Chernoff bound, we know $\bcalE_2$ holds with probability exceeding $1-2p^{-\alpha}$.
From triangular inequality, we have
\begin{align*}
	&~~~~\op{\frac{1}{2n^2}\sum_{i,j}(U^TX_iY_j^TU + U^TY_jX_i^TU) - U^TMM^TU} \no\\
	&\leq \frac{1}{2n^2}\op{\Delta_{1n}\Delta_{2n}^T + \Delta_{2n}\Delta_{1n}^T} + \frac{1}{2n}\op{\Delta_{2n}M^TU + U^TM \Delta_{2n}^T} + \frac{1}{2n}\op{\Delta_{1n}M^TU + U^TM \Delta_{1n}^T}\no\\
	&=: A_1 + A_2 + A_3.
\end{align*}
Using Markov's inequality and Golden-Thompson inequality repeatedly, we get 
\begin{align*}
	&~~~~\PP\left(\{\op{A_1}> t\}\cap\bcalE\right) \\
	&\leq r\cdot e^{-\lambda t}\EE\left(\left\|\EE\left\{\exp\left[\frac{\lambda}{2n^2}[\Delta_{1n}(Y_n-M)^TU + U^T(Y_n-M)\Delta_{1n}^T]\right]1_{\bcalE}\big| S_{1n}\right\}\right\|^n\right).
\end{align*}
From triangular inequality, we have 
$$\op{\frac{\lambda}{2n^2}\Delta_{1n}(Y_n-M)^TU + U^T(Y_n-M)\Delta_{1n}^T} \leq \frac{\lambda}{n^2}\left(\op{\Delta_{1n}Y_n^TU} + \op{\Delta_{1n}M^TU}\right).$$
Under the event $\bcalE_2$, we have 
\begin{align*}
	&~~~~\op{\Delta_{1n}Y_n^TU} \leq \op{S_{1n}Y_n^TU} + n\op{U^TMY_n^TU}\no\\
	&\leq (3\alpha + 7)p_1p_2\|M\|_{\infty}^2(\frac{n}{p_2} + \log(p))\mu r p_2 + n\sqrt{\mu rp_2}p_1p_2\|M\|_{\infty}^2.
\end{align*}
On the other hand, under the event $\bcalE_1$, 
$$
\op{\Delta_{1n} M^T U}\leq \op{\Delta_{1n}}\op{M^TU}\leq \sqrt{\frac{8(\alpha + 1)\mu p_1p_2^2r\log(p)n}{3}}\|M\|_{\infty}\sqrt{p_1p_2}\|M\|_{\infty}.
$$
As long as $n\geq \frac{8}{3}(\alpha + 1)\log(p)$, we have $\op{\Delta_{1n} M^T U} \leq n\sqrt{\mu p_2r}p_1p_2\|M\|_{\infty}^2$ and thus
\begin{align*}
	&~~~~\frac{\lambda}{2n^2}\op{\Delta_{1n}(Y_n-M)^TU + U^T(Y_n-M)\Delta_{1n}^T}\no\\
	&\leq \frac{\lambda}{n^2}\left((3\alpha + 7)p_1p_2\|M\|_{\infty}^2(\frac{n}{p_2} + \log(p))\mu r p_2 + 2n\sqrt{\mu rp_2}p_1p_2\|M\|_{\infty}^2\right).
\end{align*}
Therefore for any
$$
\lambda\leq n^2\left((3\alpha + 7)p_1p_2\|M\|_{\infty}^2(\frac{n}{p_2} + \log(p))\mu r p_2 + 2n\sqrt{\mu rp_2}p_1p_2\|M\|_{\infty}^2\right)^{-1},
$$
we have
\begin{align*}
	&~~~~\EE\left\{\exp\left[\frac{\lambda}{2n^2}[\Delta_{1n}(Y_n-M)^TU + U^T(Y_n-M)\Delta_{1n}^T]\right]1_{\bcalE}\big| S_{1n}\right\}\no\\
	&\leq I_r + \EE\left\{\left[\frac{\lambda^2}{4n^4}[\Delta_{1n}(Y_n-M)^TU + U^T(Y_n-M)\Delta_{1n}^T]^2\right]1_{\bcalE}\big| S_{1n}\right\}\no\\
	&\leq I_r + \EE\left\{\left[\frac{\lambda^2}{4n^4}[\Delta_{1n}Y_n^TU + U^TY_n\Delta_{1n}^T]^2\right]1_{\bcalE}\big| S_{1n}\right\}\no\\
	&\leq I_r + \frac{\lambda^2p_1p_2\|M\|_{\infty}^2}{4n^4}\left[(\mu r+ 2)\Delta_{1n}\Delta_{1n}^T + \text{tr}(\Delta_{1n}\Delta_{1n}^T)I_r\right],
\end{align*}
where in the first inequality we use $\exp(A)\leq I + A + A^2$ for $\op{A}\leq 1$. And notice that $\text{tr}(\Delta_{1n}\Delta_{1n}^T) \leq r\op{\Delta_{1n}\Delta_{1n}^T}$, we obtain 
\begin{align*}
	&~~~~\left\|\EE\left\{\exp\left[\frac{\lambda}{2n^2}[\Delta_{1n}(Y_n-M)^TU + U^T(Y_n-M)\Delta_{1n}^T]\right]1_{\bcalE}\big| S_{1n}\right\}\right\|\\
	&\leq 1 + \frac{\lambda^2 p_1p_2\|M\|_{\infty}^2\mu r}{2n^4}\op{\Delta_{1n}\Delta_{1n}^T}\no\\
	&\leq 1 + \frac{\lambda^2 p_1p_2\|M\|_{\infty}^2\mu r}{2n^4}\frac{8n(\alpha+1)\mu p_1p_2^2r\log(p)}{3}\|M\|_{\infty}^2\\ 
	&= 1 + \frac{4(\alpha+1)\lambda^2\mu^2 r^2p_2\log(p)}{3n^3}(p_1p_2\|M\|_{\infty}^2)^2.
\end{align*}
Therefore we have
\begin{align*}
	\PP\left(\{\op{A_1}> t\}\cap\bcalE\right) \leq r\cdot e^{-\lambda t} \exp\left(\frac{4(\alpha+1)\lambda^2\mu^2 r^2p_2\log(p)}{3n^2}(p_1p_2\|M\|_{\infty}^2)^2\right).
\end{align*}
Taking 
\begin{align*}
	\lambda = \min\big\{\frac{3n^2t}{8(\alpha+1)\mu^2r^2p_1^2p_2^3\|M\|_{\infty}^4\log(p)}, &\frac{n^2}{(6\alpha + 14)p_1p_2^2\mu r\|M\|_{\infty}^2(\frac{n}{p_2}+\log(p))}, \\
	&\frac{n}{4\sqrt{\mu rp_2}p_1p_2\|M\|_{\infty}^2}\big\}
\end{align*}
yields
\begin{align*}
	\PP\left(\{\op{A_1}> t\}\cap\bcalE\right) \leq r\cdot \exp&\big(-\min\big\{\frac{3n^2t^2}{16(\alpha+1)\mu^2r^2p_1^2p_2^3\|M\|_{\infty}^4\log(p)},\no\\ &\frac{n^2t}{(12\alpha + 28)p_1p_2^2\mu r\|M\|_{\infty}^2(\frac{n}{p_2}+\log(p))},
	 \frac{nt}{8\sqrt{\mu rp_2}p_1p_2\|M\|_{\infty}^2}\big\}\big).
\end{align*}
Now we bound $A_2$ and $A_3$. Due to the symmetry, we shall consider $A_2$ only. In a similar fashion, we have
$$
\PP\left(\{\op{A_2}> t\}\cap\bcalE\right) \leq r\cdot \left\|\EE\left\{\exp\left[\frac{\lambda}{2n}(U^TM(Y_n-M)^TU + U^T(Y_n-MM^TU))\right]1_{\bcalE}\right\}\right\|^n.
$$
Simple calculations show that 
$$
\op{\frac{\lambda}{2n}(U^TM(Y_n-M)^TU + U^T(Y_n-M)M^TU)} \leq \frac{2\mu rp_1p_2\|M\|_{\infty}^2\lambda}{n}.
$$
So as long as $\lambda \leq \frac{n}{2\mu rp_1p_2\|M\|_{\infty}^2}$,
we have 
\begin{align*}
	&~~~~\left\|\EE\left\{\exp\left[\frac{\lambda}{2n}(U^TM(Y_n-M)^TU + U^T(Y_n-M)M^TU)\right]1_{\bcalE}\right\}\right\|\no\\
	&\leq 1 + \left\|\EE\left\{\left[\frac{\lambda}{2n}(U^TM(Y_n-M)^TU + U^T(Y_n-M)M^TU)\right]^2\cdot 1_{\bcalE}\right\}\right\|\no\\
	&\leq 1 + \left\|\EE\left\{\left[\frac{\lambda}{2n}(U^TMY_n^TU + U^TY_nM^TU)\right]^2\cdot 1_{\bcalE}\right\}\right\|\no\\
	&\leq 1 + \frac{\lambda^2\mu rp_1^2p_2^2\|M\|_{\infty}^4}{n^2}.
\end{align*}
Now take 
$\lambda = \min\left\{\frac{n}{2\mu rp_1p_2\|M\|_{\infty}^2}, \frac{nt}{2\mu rp_1^2p_2^2\|M\|_{\infty}^2}\right\}$, then 
\begin{align*}
	\PP\left(\{\op{A_2}> t\}\cap\bcalE\right) \leq r\cdot \exp\left\{-\min\{\frac{nt}{4\mu rp_1p_2\|M\|_{\infty}^2}, \frac{nt^2}{4\mu rp_1^2p_2^2\|M\|_{\infty}^4}\}\right\}. 
\end{align*}
So we have 
\begin{align*}
	&~~~~\PP\left\{\op{\frac{1}{2n^2}\sum_{i,j}(U^TX_iY_j^TU + U^TY_jX_i^TU) - U^TMM^TU} > t\right\}\\ 
	&\leq \sum_{k=1}^3\PP(\{\op{A_k}>t/3\}\cap\bcalE) + \PP(\bcalE^c).
\end{align*}
By taking 
$$
t = C\alpha^2\log^2(p)\frac{p_1p_2\|M\|_{\infty}^2}{n}\left(\mu rp_2^{1/2} + \frac{\mu r p_2}{n} + (\frac{\mu rn}{\log^3(p)})^{1/2}\right),
$$
we conclude that 
$$
\PP\left\{\op{\frac{1}{2n^2}\sum_{i,j}(U^TX_iY_j^TU + U^TY_jX_i^TU) - U^TMM^TU} > t\right\} \leq 7p^{-\alpha}.
$$
And the proof is finalized by replacing $\alpha$ with $\alpha + \log_p(7)$ and adjusting the constant $C$ accordingly.

\subsection{Proof of Lemma \ref{lemma:sample bound}}\label{sec:lemma:sample bound}
First denote $Z_i = U^T(p_1p_2\calP_{\omega_i}(M) - M)$. Then we have $U^T(\frac{p_1p_2}{n}\calP_{\Omega}(M) - M) = \frac{1}{n}\sum_{i=1}^nZ_i$. Notice that 
$
\op{p_1p_2U^T\calP_{\omega_i}(M)}\leq \sqrt{p_1}p_2\sqrt{\mu r}\|M\|_{\infty}
$. And this implies that $$
\op{Z_i} \leq 2\sqrt{p_1}p_2\sqrt{\mu r}\|M\|_{\infty}.
$$
On the other hand, we have
\begin{align*}
	\EE Z_iZ_i^T \leq p_1p_2U^T(\sum_{i,j}M_{ij}^2e_ie_i^T) U\leq p_1p_2^2\|M\|_{\infty}^2I_r,
\end{align*} 
and
\begin{align*}
	\EE Z_i^TZ_i \leq \mu r p_1p_2\|M\|_{\infty}^2 I_{p_2}.
\end{align*}
So $\op{\EE\sum_{i=1}^nZ_iZ_i^T}\vee \op{\EE\sum_{i=1}^nZ_i^TZ_i}\leq np_1p_2\|M\|_{\infty}^2(\mu r\vee p_2)$. Using matrix Bernstein inequality and we have with probability exceeding $1-p^{-\alpha}$, 
$$
\op{U^T(p_1p_2\calP_{\omega_i}(M) - M)} \leq C\alpha\left(\frac{\sqrt{p_1}p_2\sqrt{\mu r}\|M\|_{\infty}\log(p)}{n} + \sqrt{\frac{p_1p_2\|M\|_{\infty}^2(\mu r\vee p_2)\log(p)}{n}}\right).
$$

%\vskip 0.2in
%\bibliography{sample}

\end{document}